\documentclass{article}

\PassOptionsToPackage{numbers, compress}{natbib}

\usepackage[final]{neurips_2019}

\usepackage{microtype}
\usepackage{subfigure}
\usepackage{booktabs} 
\usepackage{amsmath}
\usepackage{hyperref}
\usepackage{bbm,cleveref}

\usepackage[font={small}]{caption}

\usepackage{algorithm,algcompatible,amsmath}
\algnewcommand\INPUT{\item[\textbf{Input:}]}
\algnewcommand\OUTPUT{\item[\textbf{Output:}]}

\usepackage{epsfig,amsmath,amssymb,amsfonts,amstext,amsthm,mathrsfs}
\usepackage{latexsym,graphics,epsf,epsfig,psfrag}
\usepackage{dsfont,color,epstopdf,fixmath}
\usepackage{enumitem}

\usepackage{algcompatible}
\usepackage[utf8]{inputenc} 
\usepackage[T1]{fontenc}    
\usepackage{hyperref}       
\usepackage{url}            
\usepackage{booktabs}       
\usepackage{amsfonts}       
\usepackage{nicefrac}       
\usepackage{microtype}      
\usepackage{mathtools}
\usepackage{amssymb}

\newcommand{\mcs}{\mathcal S}
\newcommand{\mca}{\mathcal A}
\newcommand{\mcf}{\mathcal F}
\newcommand{\mcv}{\mathcal V}
\newcommand{\mcphi}{{\rm\Phi}}
\newcommand{\mcxi}{{\rm\Xi}}
\newcommand{\mcpi}{{\rm \Pi}}

\newcommand{\mE}{\mathbb{E}}
\newcommand{\mN}{\mathbb{N}}
\newcommand{\mP}{\mathbb{P}}
\newcommand{\mR}{\mathbb{R}}
\newcommand{\ltwo}[1]{\left\|#1\right\|_2}

\DeclareMathOperator*{\argmin}{argmin}

\newtheorem{theorem}{Theorem}
\newtheorem{lemma}{Lemma}
\newtheorem{assumption}{Assumption}

\title{Two Time-scale Off-Policy TD Learning: Non-asymptotic Analysis over Markovian Samples}

\author{%
	Tengyu Xu\\
	Department of Electrical and Computer Engineering\\
	The Ohio State University\\
	\texttt{xu.3260@osu.edu} \\
	\AND
	Shaofeng Zou \\
	Department of Electrical Engineering\\
	University at Buffalo, The State University of New York\\
	\texttt{szou3@buffalo.edu} \\
	\AND
	Yingbin Liang \\
	Department of Electrical and Computer Engineering \\
	The Ohio State University \\
	\texttt{liang.889@osu.edu} \\
}

\begin{document}

\maketitle

\begin{abstract}
Gradient-based temporal difference (GTD) algorithms are widely used in off-policy learning scenarios. Among them, the two time-scale TD with gradient correction (TDC) algorithm has been shown to have superior performance. In contrast to previous studies that characterized the non-asymptotic convergence rate of TDC only under identical and independently distributed (i.i.d.) data samples, we provide the first non-asymptotic convergence analysis for two time-scale TDC under a non-i.i.d.\ Markovian sample path and linear function approximation. We show that the two time-scale TDC can converge as fast as $\mathcal{O}(\frac{\log t}{t^{2/3}})$ under diminishing stepsize, and can converge exponentially fast under constant stepsize, but at the cost of a non-vanishing error. We further propose a TDC algorithm with blockwisely diminishing stepsize, and show that it asymptotically converges with an arbitrarily small error at a blockwisely linear convergence rate. Our experiments demonstrate that such an algorithm converges as fast as TDC under constant stepsize, and still enjoys comparable accuracy as TDC under diminishing stepsize. 
\end{abstract}

\vspace{-0.4 cm}
\section{Introduction}\label{sec:intro}
In practice, it is very common that we wish to learn the value function of a {\em target} policy based on data sampled by a different {\em behavior} policy, in order to make maximum use of the data available. For such off-policy scenarios, it has been shown that conventional temporal difference (TD) algorithms \cite{sutton1988learning,sutton2018reinforcement} and Q-learning \cite{watkins1992q} may diverge to infinity when using linear function approximation \cite{baird1995residual}. 
To overcome the divergence issue in off-policy TD learning, \cite{sutton2008convergent,sutton2009fast,maei2011gradient} proposed a family of gradient-based TD (GTD) algorithms, which were shown to have guaranteed convergence in off-policy settings and are more flexible than on-policy learning in practice \cite{maei2018offpolicyac, silver2014deterministic}. 
Among those GTD algorithms, the TD with gradient correction (TDC) algorithm has been verified to have superior performance \cite{maei2011gradient} \cite{dann2014policy} and is widely used in practice. To elaborate, TDC uses the mean squared projected Bellman error as the objective function, and iteratively updates the function approximation parameter with the assistance of an auxiliary parameter that is also iteratively updated. These two parameters are typically updated with stepsizes diminishing at different rates, resulting the two time-scale implementation of TDC, i.e., the function approximation parameter is updated at a slower time-scale and the auxiliary parameter is updated at a faster time-scale.

The convergence of two time-scale TDC and general two time-scale stochastic approximation (SA) have been well studied. The asymptotic convergence has been shown in \cite{borkar2009stochastic,borkar2018concentration} for two time-scale SA, and in \cite{sutton2009fast} for two time-scale TDC, where both studies assume that the data are sampled in an identical and independently distributed (i.i.d.) manner. Under non-i.i.d.\ observed samples, the asymptotic convergence of the general two time-scale SA and TDC were established in \cite{karmakar2017two,yu2017convergence}.

All the above studies did not characterize how fast the two time-scale algorithms converge, i.e, they did not establish the non-asymptotic convergence rate, which is specially important for a two time-scale algorithm. In order for two time-scale TDC to perform well, it is important to properly choose the relative scaling rate of the stepsizes for the two time-scale iterations. In practice, this can be done by fixing one stepsize and treating the other stepsize as a tuning hyper-parameter \cite{dann2014policy}, which is very costly. The non-asymptotic convergence rate by nature captures how the scaling of the two stepsizes affect the performance and hence can serve as a guidance for choosing the two time-scale stepsizes in practice. Recently, \cite{dalal2017finite} established the non-asymptotic convergence rate for the projected two time-scale TDC with i.i.d.\ samples under diminishing stepsize. 
\begin{list}{$\bullet$}{\topsep=0.ex \leftmargin=0.15in \rightmargin=0.in \itemsep =0.01in}
\item {\em One important open problem that still needs to be addressed is to characterize the {\bf non-asymptotic} convergence rate for two time-scale TDC under {\bf non-i.i.d.} samples and diminishing stepsizes, and explore what such a result suggests for designing the stepsizes of the fast and slow time-scales accordingly.  Existing method developed in \cite{dalal2017finite} that handles the non-asymptotic analysis for i.i.d. sampled TDC does not accommodate a direct extension to the non-i.i.d.\ setting. Thus, new technical developments are necessary to solve this problem.}
\end{list}

Furthermore, although {\em diminishing} stepsize offers accurate convergence, {\em constant} stepsize is often preferred in practice due to its much faster error decay (i.e., convergence) rate. 
For example, empirical results have shown that for {\em one} time-scale conventional TD, constant stepsize not only yields fast convergence, but also results in comparable convergence accuracy as diminishing stepsize \cite{dann2014policy}. However, for {\em two} time-scale TDC, our experiments (see \Cref{exp: sub2}) demonstrate that constant stepsize, although yields faster convergence, has much bigger convergence error than diminishing stepsize. This motivates to address the following two open issues.
\begin{list}{$\bullet$}{\topsep=0.ex \leftmargin=0.15in \rightmargin=0.in \itemsep =0.01in}
\item {\em It is important to theoretically understand/explain why constant stepsize yields large convergence error for two-time scale TDC. Existing non-asymptotic analysis for two time-scale TDC \cite{dalal2017finite} focused only on the diminishing stepsize, and does not characterize the convergence rate of two time-scale TDC under constant stepsize. }
\item {\em For two-time scale TDC, given the fact that constant stepsize yields large convergence error but converges fast, whereas diminishing stepsize has small convergence error but converges slowly, it is desirable to design a new update scheme for TDC that converges faster than diminishing stepsize, but has as good convergence error as diminishing stepsize. }
\end{list}
In this paper, we comprehensively address the above issues.


\vspace{-0.2 cm}
\subsection{Our Contribution}
\vspace{-0.2 cm}
Our main contributions are summarized as follows.

We develop a novel non-asymptotic analysis for two time-scale TDC with a single sample path and under non-i.i.d.\ data. We show that under the diminishing stepsizes $\alpha_t=c_\alpha/(1+t)^\sigma$ and $\beta_t=c_\beta/(1+t)^\nu$ respectively for slow and fast time-scales (where $c_\alpha,c_\beta,\nu,\sigma$ are positive constants and $0<\nu<\sigma\leq1$), the convergence rate can be as large as $\mathcal{O}(\frac{\log t}{t^{2/3}})$, which is achieved by $\sigma=\frac{3}{2}\nu=1$. This recovers the convergence rate (up to $\log t$ factor due to non-i.i.d.\ data) in \cite{dalal2017finite} for i.i.d.\ data as a special case. 

We also develop the non-asymptotic analysis for TDC under non-i.i.d.\ data and constant stepsize. In contrast to conventional one time-scale analysis, our result shows that the training error (at slow time-scale) and the tracking error (at fast time scale) converge at different rates (due to different condition numbers), though both converge linearly to the neighborhood of the solution. Our result also characterizes the impact of the tracking error on the training error. Our result suggests that TDC under constant stepsize can converge faster than that under diminishing stepsize at the cost of a large training error, due to a large tracking error caused by the auxiliary parameter iteration in TDC. 

We take a further step and propose a TDC algorithm under a blockwise diminishing stepsize inspired by \cite{yang2018does} in conventional optimization, in which both stepsizes are constants over a block, and decay across blocks. We show that TDC asymptotically converges with an arbitrarily small training error at a blockwisely linear convergence rate as long as the block length and the decay of stepsizes across blocks are chosen properly.
Our experiments demonstrate that TDC under a blockwise diminishing stepsize converges as fast as vanilla TDC under constant stepsize, and still enjoys comparable accuracy as TDC under diminishing stepsize. 



From the technical standpoint, our proof develops new tool to handle the non-asymptotic analysis of bias due to non-i.i.d.\ data for two time-scale algorithms under diminishing stepsize that does not require square summability, to bound the impact of the fast-time-scale tracking error on the slow-time-scale training error, and the analysis to recursively refine the error bound in order to sharpening the convergence rate.
\vspace{-0.2 cm}
\subsection{Related Work}
\vspace{-0.1 cm}
%
%
Due to extensive studies on TD learning, we here include only the most relevant work to this paper.

\textbf{On policy TD and SA.} The convergence of TD learning with linear function approximation with i.i.d samples has been well established by using standard results in SA \cite{borkar2000ode}. The non-asymptotic convergence have been established in \cite{borkar2009stochastic,kamal2010onthe,thoppe2015alekseev} for the general SA algorithms with martingale difference noise, and in \cite{dalal2018finite} for TD with i.i.d. samples. For the Markovian settings, the asymptotic convergence has been established in \cite{tsitsiklis1996analysis,tadi2001td} for of TD($\lambda$), and the non-asymptotic convergence has been provided for projected TD($\lambda$) in \cite{bhandari2018finite} and for linear SA with Markovian noise in \cite{karmakar2016dynamics,ramaswamy2018stability,srikant2019finite}.


\textbf{Off policy one time-scale GTD.} The convergence of one time-scale GTD and GTD2 (which are off-policy TD algorithms) were derived by applying standard results in SA \cite{sutton2008convergent} \cite{sutton2009fast,maei2011gradient}. The non-asymptotic analysis for GTD and GTD2 have been conducted in \cite{liu2015finite} by converting the objective function into a convex-concave saddle problem, and was further generalized to the Markovian setting in \cite{wang2017finite}. However, such an approach cannot be generalized for analyzing two-time scale TDC that we study here because TDC does not have an explicit saddle-point representation.


\textbf{Off policy two time-scale TDC and SA.} The asymptotic convergence of two time-scale TDC under i.i.d.\ samples has been established in \cite{sutton2009fast,maei2011gradient}, and the non-asymptotic analysis has been provided in \cite{dalal2017finite} as a special case of two time-scale linear SA. Under Markovian setting, the convergence of various two time-scale GTD algorithms has been studied in \cite{yu2017convergence}. The non-asymptotic analysis of two time-scale TDC  under non-i.i.d. data has not been studied before, which is the focus of this paper.

General two time-scale SA has also been studied. The convergence of two time-scale SA with martingale difference noise was established in \cite{borkar2009stochastic}, and its non-asymptotic convergence was provided in \cite{konda2004convergence,mokkadem2006convergence,dalal2017finite,borkar2018concentration}. Some of these results can be applied to two time-scale TDC under i.i.d.\ samples (which can fit into a special case of SA with martingale difference noise), but not to the non-i.i.d.\ setting. For two time-scale linear SA with more general Markovian noise, only asymptotic convergence was established in \cite{tadic2004almost,yaji2016stochastic,karmakar2017two}. In fact, our non-asymptotic analysis for two time-scale TDC can be of independent interest here to be further generalized for studying linear SA with more general Markovian noise.

\vspace{-0.3 cm}
\section{Problem Formulation}
\vspace{-0.2 cm}
\subsection{Off-policy Value Function Evaluation}
We consider the problem of policy evaluation for a Markov decision process (MDP) $(\mcs, \mca, \mathsf{P},r,\gamma)$, where $\mcs\subset \mR^d$ is a compact state space, $\mca$ is a finite action set, $\mathsf{P}=\mathsf{P}(s^\prime|s,a)$ is the transition kernel,  $r(s, a, s^\prime)$ is the reward function bounded by $r_{\max}$, and $\gamma\in(0,1)$ is the discount factor. A stationary policy $\pi$ maps a state $s\in \mcs$ to a probability distribution $\pi(\cdot|s)$ over $\mca$. At time-step $t$, suppose the process is in some state $s_t\in \mcs$. Then an action $a_t\in \mca$ is taken based on the distribution $\pi(\cdot|s_t)$, the system transitions to a next state $s_{t+1}\in \mcs$ governed by the transition kernel $\mathsf{P}(\cdot|s_t,a_t)$, and a reward $r_t=r(s_t, a_t, s_{t+1})$ is received. Assuming the associated Markov chain $p(s^\prime|s)=\sum_{a\in\mca}p(s^\prime|s,a)\pi(a|s)$ is ergodic, let $\mu_\pi$ be the induced stationary distribution of this MDP, i.e., $\sum_{s}p(s^\prime|s)\mu_{\pi}(s)=\mu_{\pi}(s^\prime)$. The value function for policy $\pi$ is defined as: $v^\pi\left(s\right)=\mE[\sum_{t=0}^{\infty}\gamma^t r(s_t,a_t, s_{t+1})|s_0=s,\pi]$, and it is known that $v^\pi(s)$ is the unique fixed point of the Bellman operator $T^\pi$, i.e., $v^\pi(s) = T^\pi v^\pi(s)\coloneqq r^\pi(s)+\gamma\mE_{s^\prime |s} v^\pi(s^\prime)$, where $r^\pi(s)=\mE_{a, s^\prime|s}r(s,a,s^\prime)$ is the expected reward of the Markov chain induced by policy $\pi$. 

We consider policy evaluation problem in the off-policy setting. Namely, a sample path $\{ (s_t, a_t, s_{t+1}) \}_{t\geq 0}$ is generated by the Markov chain according to the behavior policy $\pi_b$, but our goal is to obtain the value function of a target policy $\pi$, which is different from $\pi_b$.
\vspace{-0.2 cm}
\subsection{Two Time-Scale TDC}
When $\mcs$ is large or infinite, a linear function $\hat{v}(s,\theta)=\phi(s)^\top\theta$ is often used to approximate the value function, where $\phi(s)\in\mR^d$ is a fixed feature vector for state $s$ and $\theta\in\mR^d$ is a parameter vector. We can also write the linear approximation in the vector form as $\hat{v}(\theta)={\rm \Phi} \theta$, where ${\rm \Phi}$ is the $|\mcs|\times d$ feature matrix. To find a parameter $\theta^*\in\mR^d$ with $ \mE_{\mu_{\pi_b}}\hat{v}(s,\theta^*)= \mE_{\mu_{\pi_b}}T^\pi\hat{v}(s,\theta^*)$. The gradient-based TD algorithm TDC \cite{sutton2009fast} updates the parameter by minimizing the mean-square projected Bellman error (MSPBE) objective, defined as
\begin{flalign*}
	J(\theta)=\mE_{\mu_{\pi_b}}[\hat{v}(s,\theta)-{\rm \Pi} T^\pi\hat{v}(s,\theta)]^2,
\end{flalign*}
where ${\rm \Pi}=\mcphi(\mcphi^\top \mcxi\mcphi)^{-1}\mcphi^\top \mcxi$ is the orthogonal projection operation into the function space $\hat{\mcv}=\{\hat{v}(\theta)\  |\  \theta\in\mR^d\  \text{and}\ \hat{v}(\cdot, \theta)=\phi(\cdot)^\top \theta   \}$ and $\mcxi$ denotes the $|\mcs|\times|\mcs|$ diagonal matrix with the components of $\mu_{\pi_b}$ as its diagonal entries. Then, we define the matrices $A$, $B$, $C$ and the vector $b$ as 
\begin{flalign*}
	A\coloneqq\mE_{\mu_{\pi_b}}[\rho(s,a)\phi(s)(\gamma\phi(s^\prime)-\phi(s))^\top],\quad B\coloneqq-\gamma\mE_{\mu_{\pi_b}}[\rho(s,a)\phi(s^\prime)\phi(s)^\top],\\
	C\coloneqq-\mE_{\mu_{\pi_b}}[\phi(s)\phi(s)^\top],\quad b\coloneqq\mE_{\mu_{\pi_b}}[\rho(s,a)r(s,a,s^\prime)\phi(s)],
\end{flalign*}
where $\rho(s,a)=\pi(a|s)/\pi_b(a|s)$ is the importance weighting factor with $\rho_{\max}$ being its maximum value. If $A$ and $C$ are both non-singular, $J(\theta)$ is strongly convex and has $\theta^*=-A^{-1}b$ as its global minimum, i.e., $J(\theta^*)=0$. Motivated by minimizing the MSPBE objective function using the stochastic gradient methods, TDC was proposed with the following update rules:
\begin{flalign}
	&\theta_{t+1}=\mcpi_{R_\theta}\left(\theta_t + \alpha_t(A_t\theta_t+b_t+B_t w_t)\right), \label{algorithm1_1}\\
	&w_{t+1}=\mcpi_{R_w}\left( w_t + \beta_t(A_t\theta_t+b_t+C_t w_t) \right),\label{algorithm1_2}
\end{flalign}
where $A_t=\rho(s_t,a_t)\phi(s_t)(\gamma\phi(s_{t+1})-\phi(s_t))^\top$, $B_t=-\gamma\rho(s_t,a_t)\phi(s_{t+1})\phi(s_t)^\top$, $C_t = -\phi(s_t)\phi(s_t)^\top$, $b_t = \rho(s_t, a_t)r(s_t,a_t,s_{t+1})\phi(s_t)$, and $\mcpi_R(x)=\argmin_{x^\prime:||x^\prime||_2\leq R}||x-x^\prime||_2$ is the projection operator onto a norm ball of radius $R<\infty$. The projection step is widely used in the stochastic approximation literature. As we will show later, iterations \eqref{algorithm1_1}-\eqref{algorithm1_2} are guaranteed to converge to the optimal parameter $\theta^*$ if we choose the value of $R_\theta$ and $R_w$ appropriately. 
TDC with the update rules \eqref{algorithm1_1}-\eqref{algorithm1_2} is a two time-scale algorithm. The parameter $\theta$ iterates at a slow time-scale determined by the stepsize $\{\alpha_t\}$, whereas $w$ iterates at a fast time-scale determined by the stepsize $\{\beta_t\}$. 
Throughout the paper, we make the following standard assumptions \cite{bhandari2018finite,wang2017finite,maei2011gradient}.
\begin{assumption}[Problem solvability]\label{ass1}
	The matrix $A$ and $C$ are non-singular.
\end{assumption}
\begin{assumption}[Bounded feature]\label{ass2}
	$\ltwo{\phi(s)}\leq 1$ for all $s\in\mcs$ and $\rho_{\max}<\infty$.
\end{assumption}
\begin{assumption}[Geometric ergodicity]\label{ass3}
	There exist constants $m>0$ and $\rho\in(0,1)$ such that
	\begin{flalign*}
	\sup_{s\in\mcs}d_{TV}(\mP(s_t\in\cdot|s_0=s),\mu_{\pi_b})\leq m\rho^t, \forall t\geq 0,
	\end{flalign*}
	where $d_{TV}(P,Q)$ denotes the total-variation distance between the probability measures $P$ and $Q$.
\end{assumption}
In Assumption \ref{ass1}, the matrix $A$ is required to be non-singular so that the optimal parameter $\theta^*=-A^{-1}b$ is well defined. The matrix $C$ is non-singular when the feature matrix $\mcphi$ has linearly independent columns. Assumption \ref{ass2} can be ensured by normalizing the basis functions $\{\phi_i\}_{i=1}^d$ and when $\pi_b(\cdot|s)$ is non-degenerate for all $s$. Assumption \ref{ass3} holds for any time-homogeneous Markov chain with finite state-space and any uniformly ergodic Markov chains with general state space.
Throughout the paper, we require $R_\theta\geq \ltwo{A}\ltwo{b}$ and $R_w\geq 2\ltwo{C^{-1}}\ltwo{A}R_\theta$. In practice, we can estimate $A$, $C$ and $b$ as mentioned in \cite{bhandari2018finite} or simply let $R_\theta$ and $R_w$ to be large enough.
\vspace{-0.2cm}

\section{Main Theorems}
\subsection{Non-asymptotic Analysis under Diminishing Stepsize}
Our first main result is the convergence rate of two time-scale TDC with diminishing stepsize. We define the tracking error: $z_t = w_t -\psi(\theta_t)$, where $\psi(\theta_t)=-C^{-1}(b+A\theta_t)$ is the stationary point of the ODE given by $\dot{w}(t)= Cw(t) + A\theta_t + b$, with $\theta_t$ being fixed. Let $\lambda_\theta$ and $\lambda_w$ be any constants that satisfy $\lambda_{\max}(2A^\top C^{-1}A)\leq\lambda_\theta<0$ and $\lambda_{\max}(2C)\leq\lambda_w<0$.
\begin{theorem}\label{thm1}
	Consider the projected two time-scale TDC algorithm in \eqref{algorithm1_1}-\eqref{algorithm1_2}. Suppose Assumptions \ref{ass1}-\ref{ass3} hold. Suppose we apply diminishing stepsize $\alpha_t=\frac{c_\alpha}{(1+t)^\sigma}$, $\beta_t=\frac{c_\beta}{(1+t)^\nu}$ which satisfy $0<\nu<\sigma< 1$, $0<c_\alpha<\frac{1}{|\lambda_\theta|}$ and $0<c_\beta<\frac{1}{|\lambda_w|}$. Suppose $\epsilon$ and $\epsilon^\prime$ can be any constants in $(0,\sigma-\nu]$ and $(0,0.5]$, respectively. Then we have for $t\geq0$:
	\begin{flalign}
		\mE\ltwo{\theta_t-\theta^*}^2&\leq \mathcal{O}(e^{\frac{-|\lambda_\theta| c_\alpha}{1-\sigma}(t^{1-\sigma}-1)}) + \mathcal{O}\Big(\frac{\log t}{t^\sigma}\Big) +\mathcal{O}\Big(\frac{\log t}{t^\nu}+h(\sigma, \nu)\Big)^{1-\epsilon^\prime}, \label{thm1eq1_1}\\
		\mE\ltwo{z_t}^2&\leq \mathcal{O}\Big( \frac{\log t}{t^\nu} \Big) + \mathcal{O}(h(\sigma, \nu))\label{thm1eq1_2},
	\end{flalign}
	where
	\begin{equation}\label{thm1eq3}
	h(\sigma, \nu)=\left\{
	\begin{array}{lr}
	\frac{1}{t^\nu}, & \sigma>1.5\nu, \\
	\frac{1}{t^{2(\sigma-\nu)-\epsilon}}, &  \nu<\sigma\leq1.5\nu.
	\end{array}
	\right.
	\end{equation}
	If $0<\nu<\sigma=1$, with $c_\alpha=\frac{1}{|\lambda_\theta|}$ and $0<c_\beta<\frac{1}{|\lambda_w|}$, we have for $t\geq0$
	\begin{flalign}\label{thm1eq2}
		\mE\ltwo{\theta_t-\theta^*}^2&\leq \mathcal{O}\Big( \frac{(\log t)^2}{t} \Big) + \mathcal{O}\Big(\frac{\log t}{t^\nu}+h(1, \nu)\Big)^{1-\epsilon^\prime}.
	\end{flalign}
	For explicit expressions of \eqref{thm1eq1_1}, \eqref{thm1eq1_2} and \eqref{thm1eq2}, please refer to \eqref{thm1_1: explicit}, \eqref{eq: firstboundz} and \eqref{thm1_2: explicit} in the Appendix.
\end{theorem}
We further explain Theorem \ref{thm1} as follows: (a) In \eqref{thm1eq1_1} and \eqref{thm1eq3}, since both $\epsilon$ and $\epsilon^\prime$ can be arbitrarily small, the convergence of $\mE\ltwo{\theta_t-\theta^*}^2$ can be almost as fast as $ \frac{1}{t^{2(\sigma-\nu)}}$ when $\nu<\sigma<1.5\nu$, and $ \frac{\log t}{t^\nu}$ when $1.5\nu\leq\sigma$. Then best convergence rate is almost as fast as $\mathcal{O}(\frac{\log t}{t^{2/3}})$ with $\sigma=\frac{3}{2}\nu=1$. (b) If data are i.i.d. generated, then our bound reduces to $\mE\ltwo{\theta_t-\theta^*}^2\leq \mathcal{O}(\exp(\lambda_\theta c_\alpha(t^{1-\sigma}-1)/(1-\sigma))) + \mathcal{O}(1/t^\sigma) +\mathcal{O}(h(\sigma, \nu))^{1-\epsilon^\prime}$ with $h(\sigma, \nu)=\frac{1}{t^{\nu}}$ when $\sigma > 1.5\nu$, and $h(\sigma, \nu)=\frac{1}{t^{2(\sigma-\nu)-\epsilon}}$ when $\nu<\sigma\leq1.5\nu$.
The best convergence rate is almost as fast as $\frac{1}{t^{2/3}}$ with $\sigma=\frac{3}{2}\nu=1$ as given in \cite{dalal2017finite}.

Theorem \ref{thm1} characterizes the relationship between the convergence rate of $\theta_t$ and stepsizes $\alpha_t$ and $\beta_t$. The first term of the bound in \eqref{thm1eq1_1} corresponds to the convergence rate of $\theta_t$ with full gradient $\nabla J(\theta_t)$, which exponentially decays with $t$. The second term is introduced by the bias and variance of the gradient estimator which decays sublinearly with $t$. The last term arises due to the accumulated tracking error $z_t$, which specifically arises in two time-scale algorithms, and captures how accurately $w_t$ tracks $\psi(\theta_t)$. Thus, if $w_t$ tracks the stationary point $\psi(\theta_t)$ in each step perfectly, then we have only the first two terms in \eqref{thm1eq1_1}, which matches the results of one time-scale TD learning \cite{bhandari2018finite,dalal2018finite}. Theorem \ref{thm1} indicates that asymptotically, \eqref{thm1eq1_1} is dominated by the tracking error term $\mathcal{O}(h(\sigma,\nu)^{1-\epsilon^\prime})$, which depends on the diminishing rate of $\alpha_t$ and $\beta_t$. Since both $\epsilon$ and $\epsilon^\prime$ can be arbitrarily small, if the diminishing rate of $\alpha_t$ is close to that of $\beta_t$, then the tracking error is dominated by the slow drift, which has an approximate order $\mathcal{O}(1/t^{2(\sigma-\nu)})$; if the diminishing rate of $\alpha_t$ is much faster than that of $\beta_t$, then the tracking error is dominated by the accumulated bias, which has an approximate order $\mathcal{O}(\log t/t^\nu)$. Moreover, \eqref{thm1eq3} and \eqref{thm1eq2} suggest that for any fixed $\sigma\in(0,1]$, the optimal diminishing rate of $\beta_t$ is achieved by $\sigma=\frac{3}{2}\nu$. 

From the technical standpoint, we develop novel techniques to handle the interaction between the training error and the tracking error and sharpen the error bounds recursively. The proof sketch and the detailed steps are provided in Appendix \ref{proof_thm1}.

\vspace{-0.2 cm}
\subsection{Non-asymptotic Analysis under Constant Stepsize}\label{sec: constant}

As we remark in \Cref{sec:intro}, it has been demonstrated by empirical results \cite{dann2014policy} that the standard TD under constant stepsize not only converges fast, but also has comparable training error as that under diminishing stepsize. However, this does not hold for TDC. When the two variables in TDC are updated both under constant stepsize, our experiments demonstrate that constant stepsize yields fast convergence, but has large training error. In this subsection, we aim to explain why this happens by analyzing the convergence rate of the two variables in TDC, and the impact of one on the other. 

The following theorem provides the convergence result for TDC with the two variables iteratively updated respectively by two different constant stepsizes.
\begin{theorem}\label{thm2}
Consider the projected TDC algorithm in \cref{algorithm1_1,algorithm1_2}. Suppose Assumption \ref{ass1}-\ref{ass3} hold. Suppose we apply constant stepsize $\alpha_t=\alpha$, $\beta_t=\beta$ and $\alpha=\eta\beta$ which satisfy $\eta>0$, $0<\alpha<\frac{1}{|\lambda_\theta|}$ and $0<\beta<\frac{1}{|\lambda_w|}$. We then have for $t\geq0$:
	\begin{flalign}
	\mE\ltwo{\theta_t-\theta^*}^2 &\leq (1-|\lambda_\theta|\alpha)^t (\ltwo{\theta_0-\theta^*}^2 + C_1)\nonumber \\
	&\quad + C_2\max\{ \alpha, \alpha\ln\frac{1}{\alpha}  \} + (C_3\max\{\beta, \beta\ln\frac{1}{\beta}  \} + C_4\eta)^{0.5}\label{thm2: eq1}\\
	\mE\ltwo{z_t}^2 &\leq (1-|\lambda_w|\beta)^t\ltwo{z_0}^2 + C_5\max\{\beta, \beta\ln\frac{1}{\beta}  \} + C_6\eta,\label{thm2: eq2}
	\end{flalign}
	where  $C_1 = 4\gamma\rho_{\max}R_\theta R_w \frac{1-(1- |\lambda_\theta|\alpha)^{T+1}}{|\lambda_\theta|(1-|\lambda_\theta|\alpha)^{T+1}}$ with $T = \lceil\frac{\ln[C_5\max\{\beta, \ln(\frac{1}{\beta})\beta  \}/\ltwo{z_0}^2]}{-\ln(1-|\lambda_w|\beta)}\rceil$, and $C_2$, $C_3$, $C_4$, $C_5$ and $C_6$ are positive constants independent of $\alpha$ and $\beta$. For explicit expressions of $C_2$, $C_3$, $C_4$, $C_5$ and $C_6$, please refer to \eqref{thm2: c2}, \eqref{thm2: c3}, \eqref{thm2: c4},  \eqref{thm2: c5}, and  \eqref{thm2: c6} in the Supplementary Materials.
\end{theorem}
Theorem \ref{thm2} shows that TDC with constant stepsize converges to a neighborhood of $\theta^*$ exponentially fast. The size of the neighborhood depends on the second and the third terms of the bound in \eqref{thm2: eq1}, which arise from the bias and variance of the update of $\theta_t$ and the tracking error $z_t$ in \eqref{thm2: eq2}, respectively. Clearly, the convergence $z_t$, although is also exponentially fast to a neighborhood, is under a different rate due to the different condition number. We further note that as the stepsize parameters $\alpha$, $\beta$ approach $0$ in a way such that $\alpha/\beta\rightarrow0$, $\theta_t$ approaches to $\theta^*$ as $t\rightarrow\infty$, which matches the asymptotic convergence result for two time-scale TDC under constant stepsize in \cite{yu2017convergence}. 

\textbf{Diminishing vs Constant Stepsize:} We next discuss the comparison between TDC under diminishing stepsize and constant stepsize. Generally, Theorem \ref{thm1} suggests that diminishing stepsize yields better converge guarantee (i.e., converges exactly to $\theta^*$) than constant stepsize shown in \Cref{thm2} (i.e., converges to the neighborhood of $\theta^*$). In practice, constant stepsize is recommended because diminishing stepsize may take much longer time to converge. 
However, as Figure \ref{fig: 2} in Section \ref{exp: sub2} shows, although TDC with large constant stepsize converges fast, the training error due to the convergence to the neighborhood is significantly worse than the diminishing stepsize. More specifically, when $\eta=\alpha/\beta$ is fixed, as $\alpha$ grows, the convergence becomes faster, but as a consequence, the term $(C_3\max\{\beta, \beta\ln\frac{1}{\beta}  \} + C_4\eta)^{0.5}$ due to the tracking error increases and results in a large training error. Alternatively, if $\alpha$ gets small so that the training error is comparable to that under diminishing stepsize, then the convergence becomes very slow. This suggests that simply setting the stepsize to be constant for TDC does not yield desired performance. This motivates us to design an appropriate update scheme for TDC such that it can enjoy as fast error convergence rate as constant stepsize offers, but still have comparable accuracy as diminishing stepsize enjoys.

\vspace{-0.2cm}
\subsection{TDC under Blockwise Diminishing Stepsize}
In this subsection, we propose a blockwise diminishing stepsize scheme for TDC (see \Cref{algorithm_stagewise}), and study its theoretical convergence guarantee. In \Cref{algorithm_stagewise}, we define $t_s=\sum_{i=0}^{s}T_s$.
\begin{algorithm}
	\caption{Blockwise Diminishing Stepsize TDC}
	\label{algorithm_stagewise}
	\begin{algorithmic}[1]
		\INPUT $\theta_{0,0}=\theta_0$, $w_{0,0}=w_0=0$, $T_0=0$, block index $S$
		\FOR{$s=1, 2, ..., S$}
		\STATE $\theta_{s,0}=\theta_{s-1}$, $w_{s,0} =w_{s-1}$
		\FOR{$i = 1, 2, ...,T_s$}
		\STATE Sample $(s_{t_{s-1}+i}, a_{t_{s-1}+i}, s_{t_{s-1}+i+1}, r_{t_{s-1}+i})$ from trajetory
		\STATE $\theta_{s,i}=\mcpi_{R_\theta}\left(\theta_{s,i-1} + \alpha_s(A_{t_{s-1}+i}\theta_{s,i-1}+b_{t_{s-1}+i}+B_{t_{s-1}+i} w_{s,i-1})\right)$
		\STATE $w_{s,i}=\mcpi_{R_w}\left( w_{s,i-1} + \beta_s(A_{t_{s-1}+i}\theta_{s,i-1}+b_{t_{s-1}+i}+C_{t_{s-1}+i} w_{s,i-1}) \right)$
		\ENDFOR
		\STATE $\theta_s=\theta_{s,T_s}$, $w_s=w_{s,T_s}$
		\ENDFOR
		\OUTPUT $\theta_S$, $w_S$
	\end{algorithmic}
\end{algorithm}


The idea of \Cref{algorithm_stagewise} is to divide the iteration process into blocks, and diminish the stepsize blockwisely, but keep the stepsize to be constant within each block. In this way, within each block, TDC can decay fast due to constant stepsize and still achieve an accurate solution due to blockwisely decay of the stepsize, as we will demonstrate in \Cref{experiment}. More specifically, the constant stepsizes $\alpha_s$ and $\beta_s$ for block $s$ are chosen to decay geometrically, such that the tracking error and accumulated variance and bias are asymptotically small; and the block length $T_s$ increases geometrically across blocks, such that the training error $\mE\ltwo{\theta_s-\theta^*}^2$ decreases geometrically blockwisely. We note that the design of the algorithm is inspired by the method proposed in \cite{yang2018does} for conventional optimization problems.

The following theorem characterizes the convergence of \Cref{algorithm_stagewise}.
\begin{theorem}\label{thm3}
	Consider the projected TDC algorithm with blockwise diminishing stepsize as in \Cref{algorithm_stagewise}. Suppose Assumptions \ref{ass1}-\ref{ass3} hold. Suppose $\max\{ \log(1/\alpha_s)\alpha_s, \alpha_s  \}\leq  \min\{\epsilon_{s-1}/(4C_7), 1/|\lambda_x| \}$, $\beta_s=\eta\alpha_s$ and $T_s= \lceil \log_{1/(1-|\lambda_x|\alpha_s)}4\rceil$, where $\lambda_x<0$ and $C_7>0$ are constant independent of $s$ (see \eqref{eq: update_x_II} and \eqref{thm3: C_7} in the Supplementary Materials for explicit expression of $\lambda_x$ and $C_7$), $\epsilon_s=\ltwo{\theta_0-\theta^*}/2^s$ and $\eta\geq 1/2\max\{ 0, \lambda_{\min}(C^{-1}(A^\top + A))  \}$. Then, after $S=\lceil \log_2(\epsilon_0/\epsilon) \rceil$ blocks, we have
	\begin{flalign*}
		\mE\ltwo{\theta_{S}-\theta^*}^2&\leq \epsilon.
	\end{flalign*}
	The total sample complexity is $\mathcal{O}(\frac{1}{\epsilon^{1+\xi}})$, where $\xi>0$ can be any arbitrarily small constant.
\end{theorem}

\Cref{thm3} indicates that the sample complexity of TDC under blockwise diminishing stepsize is slightly better than that under diminishing stepsize.  Our empirical results (see Section \ref{exp: sub3}) also demonstrate that blockwise diminishing stepsize yields as fast convergence as constant stepsize and has comparable training error as diminishing stepsize. However, we want to point out that the advantage of blockwise diminishing stepsize does not come for free, rather at the cost of some extra parameter tuning in practice to estimate $\epsilon_0$, $|\lambda_x|$, $C_7$ and $\eta$; whereas diminishing stepsize scheme as guided by our \Cref{thm1} requires to tune at most three parameters to obtain desirable performance.

\vspace{-0.3cm}
\section{Experimental Results}\label{experiment}
In this section, we provide numerical experiments to verify our theoretical results and the efficiency of Algorithm \ref{algorithm1_1}. 
More precisely, we consider Garnet problems \cite{archibald1995generation} denoted as $\mathcal{G}(n_S,n_A,p,q)$, where $n_s$ denotes the number of states, $n_A$ denotes the number of actions, $p$ denotes the number of possible next states for each state-action pair, and $q$ denotes the number of features. 
The reward is state-dependent 
and both the reward and the feature vectors are generated randomly. The discount factor $\gamma$ is set to $0.95$ in all experiments. We consider the $\mathcal{G}(500,20,50,20)$ problem. For all experiments, we choose $\theta_0=w_0=0$. All plots report the evolution of the mean square error over $500$ independent runs.
\vspace{-0.2 cm}
\subsection{Optimal Diminishing Stepsize}\label{exp: sub1}
In this subsection, we provide numerical results to verify Theorem \ref{thm1}. We compare the performance of TDC updates with the same $\alpha_t$ but different $\beta_t$. We consider four different diminishing stepsize settings: (1) $c_\alpha=c_\beta=0.03$, $\sigma=0.15$; (2) $c_\alpha=c_\beta=0.18$, $\sigma=0.30$; (3) $c_\alpha=c_\beta=1$, $\sigma=0.45$; (4) $c_\alpha=c_\beta=4$, $\sigma=0.60$. For each case with fixed slow time-scale parameter $\sigma$, the fast time-scale stepsize $\beta_t$ has decay rate $\nu$ to be $\frac{1}{2}\sigma$, $\frac{1}{3}\sigma$, $\frac{5}{9}\sigma$, $\frac{2}{3}\sigma$, $\frac{5}{6}\sigma$, and $\sigma$. Our results are reported in Figure \ref{fig: 1}, in which for each case the left figure reports the overall iteration process and the right figure reports the corresponding zoomed tail process of the last 100000 iterations. It can be seen that in all cases, TDC iterations with the same slow time-scale stepsize $\sigma$ share similar error decay rates (see the left plot), and the difference among the fast time-scale parameter $\nu$ is reflected by the behavior of the error convergence tails (see the right plot). We observe that $\nu = \frac{2}{3}\sigma$ yields the best error decay rate. This corroborates \Cref{thm1}, which illustrates that the fast time-scale stepsize $\beta_t$ with parameter $\nu$ affects only the tracking error term in \eqref{thm1eq1_1}, that dominates the error decay rate asymptotically.
\begin{figure}[h]
	\vspace{-0.5cm}
	\centering 
	\subfigure[$\sigma=0.15$ (left: full; right: tail)]{\includegraphics[width=1.3in]{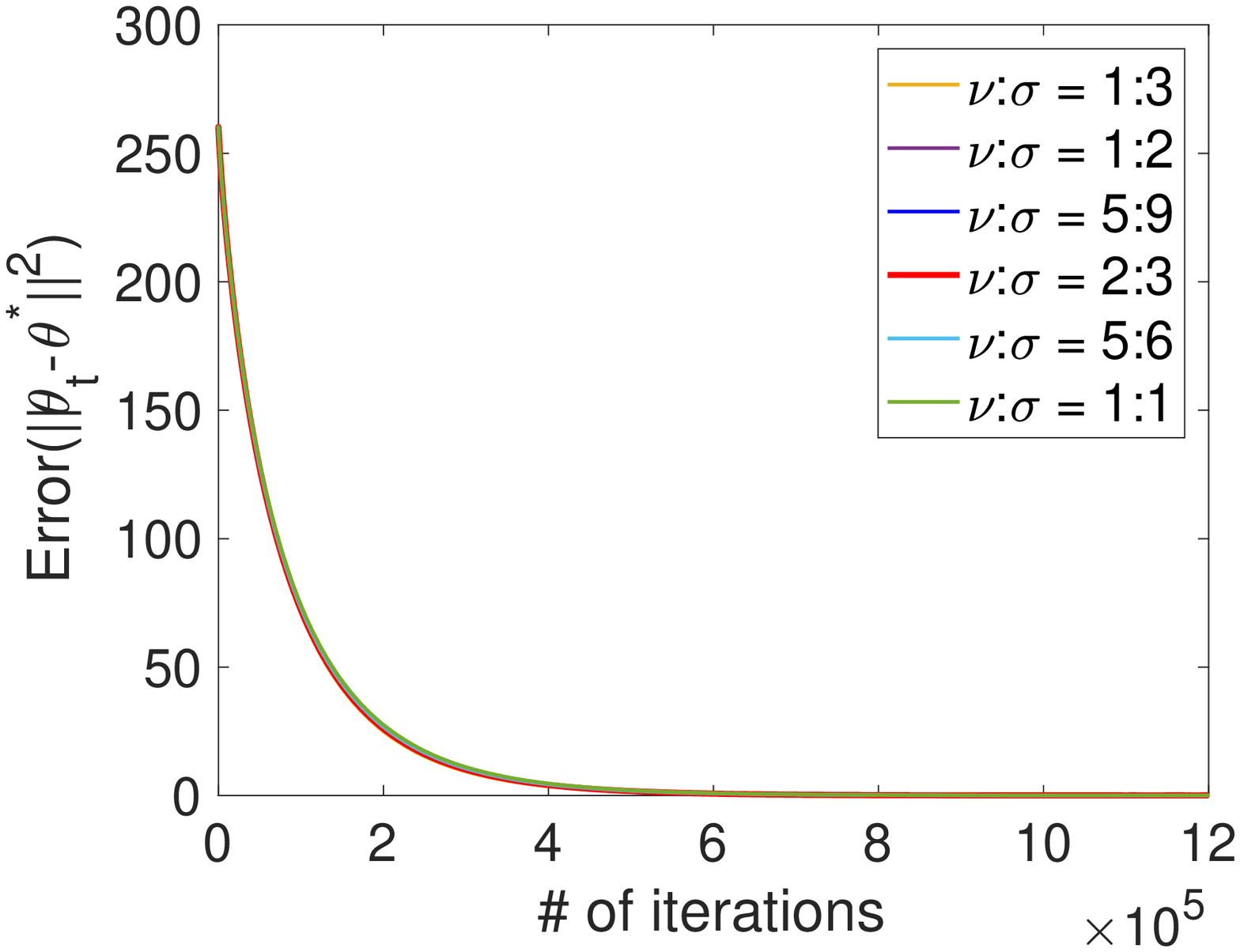}\includegraphics[width=1.3in]{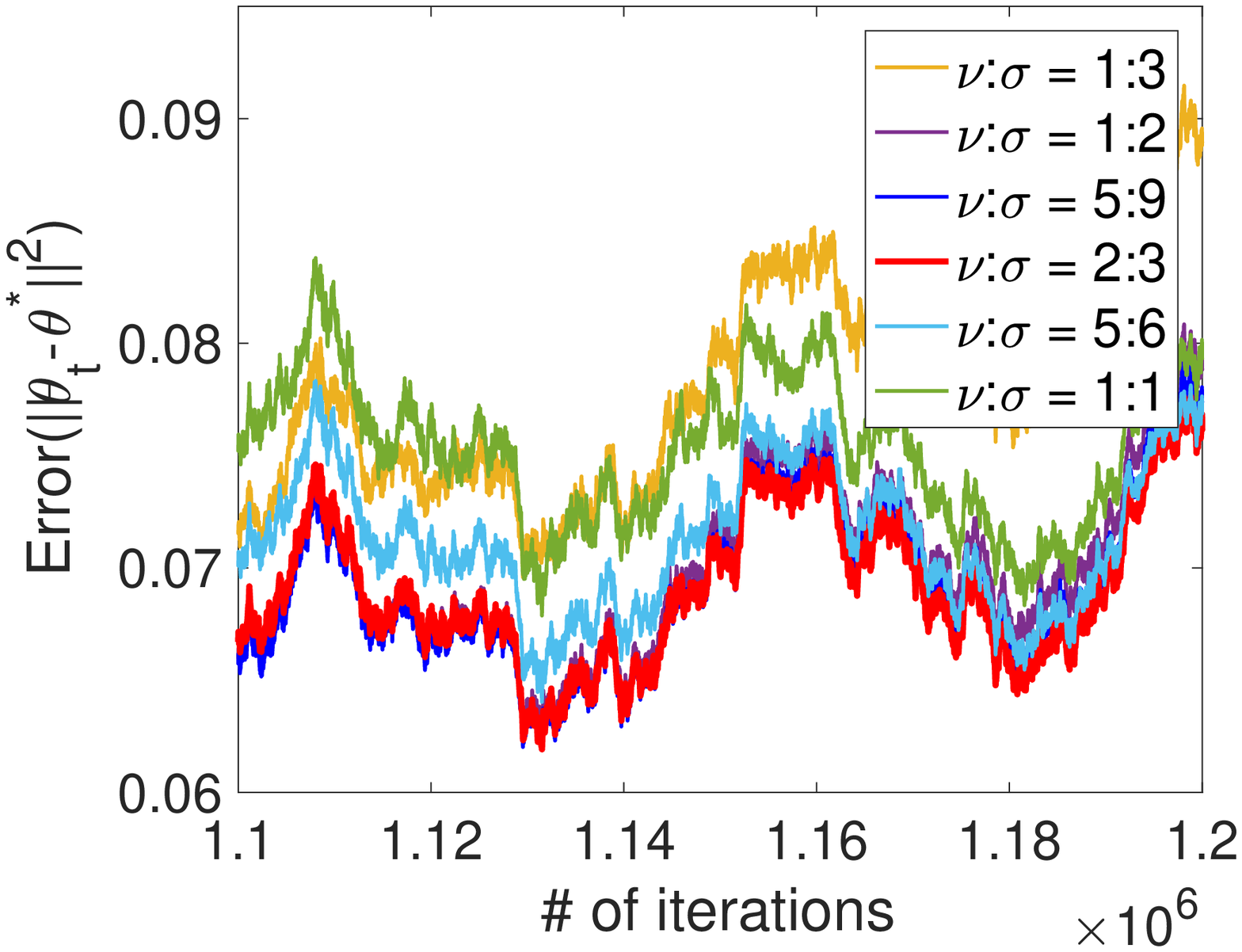}}
	\subfigure[$\sigma=0.3$ (left: full; right: tail)]{\includegraphics[width=1.3in]{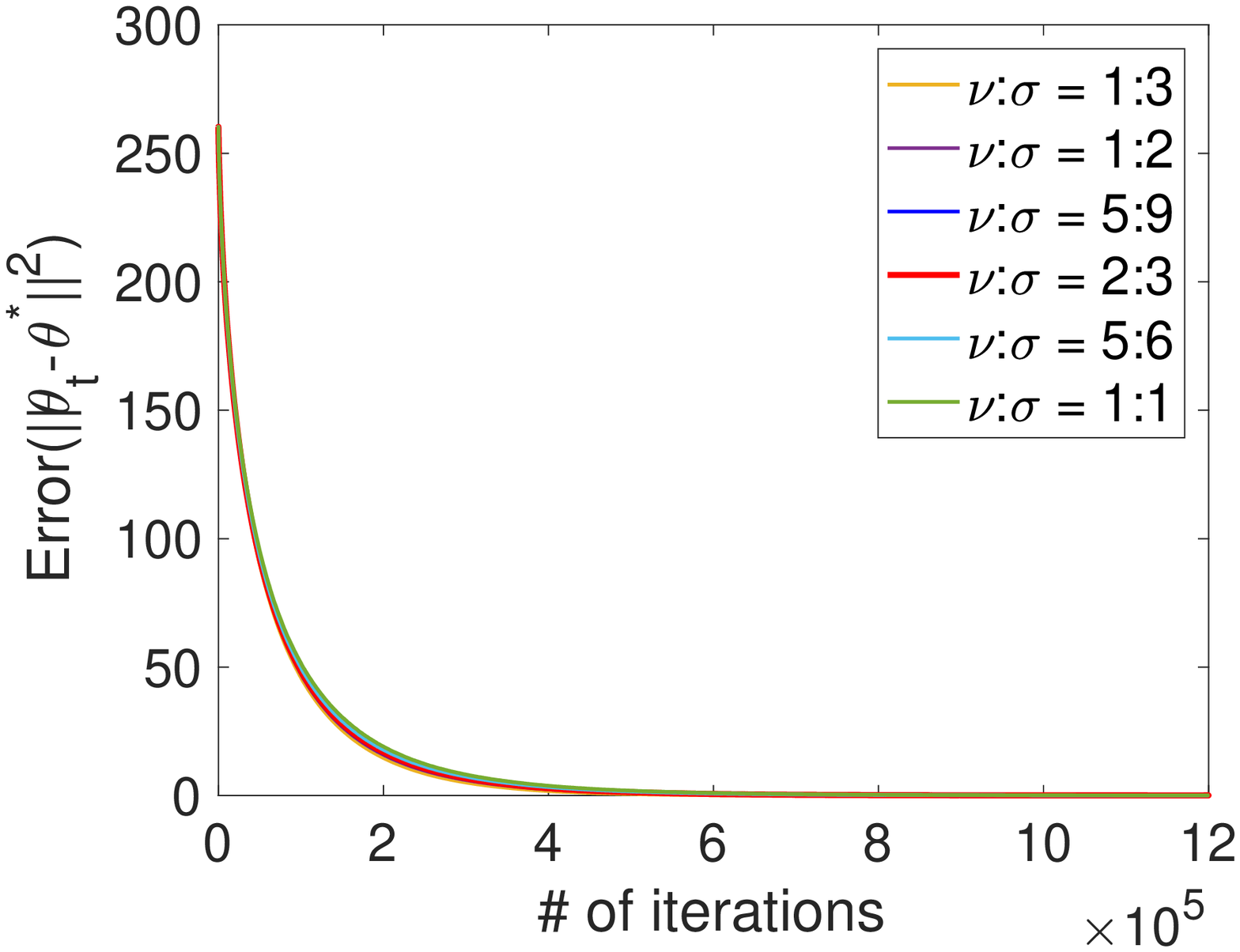}\includegraphics[width=1.3in]{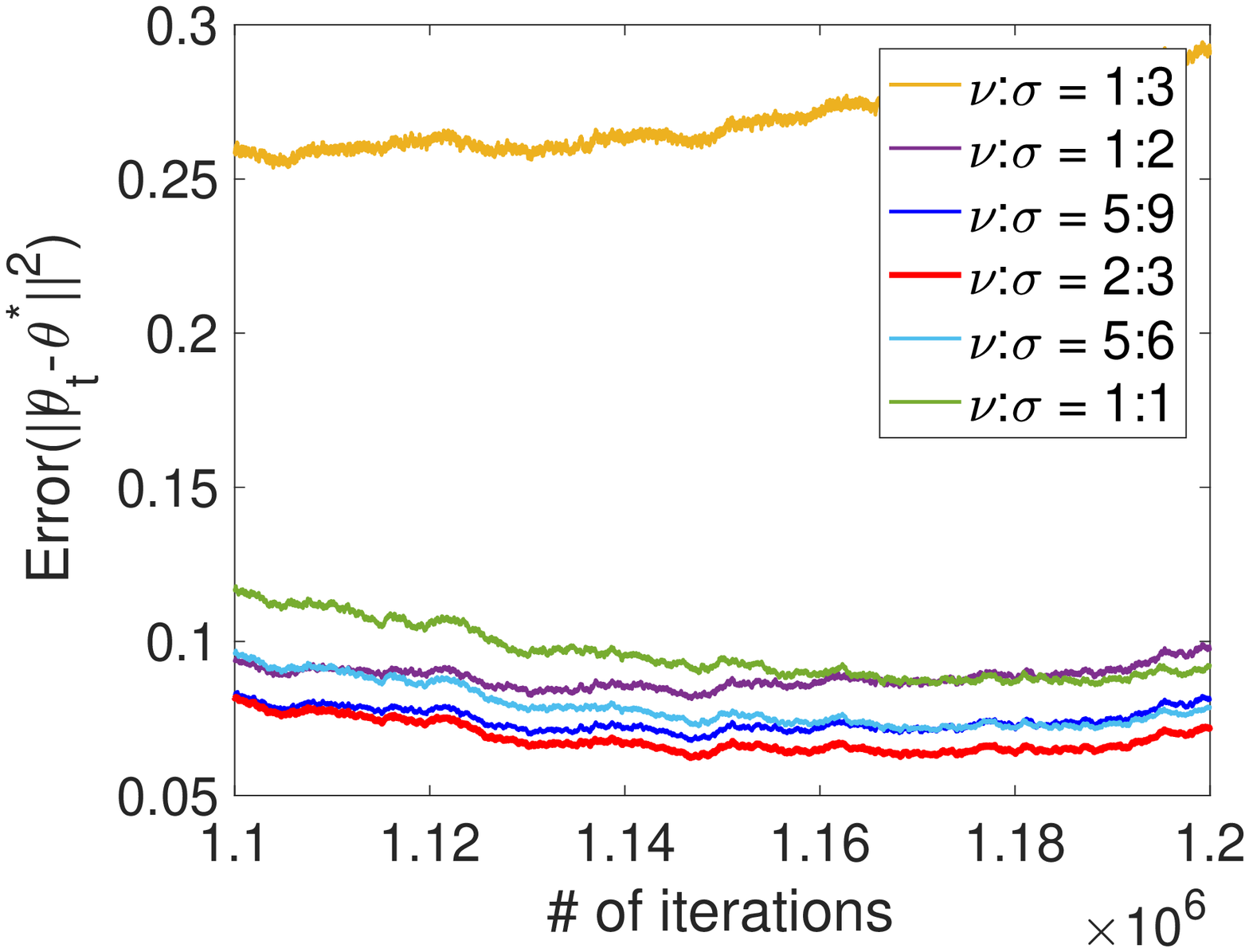}}
	\subfigure[$\sigma=0.45$ (left: full; right: tail)]{\includegraphics[width=1.3in]{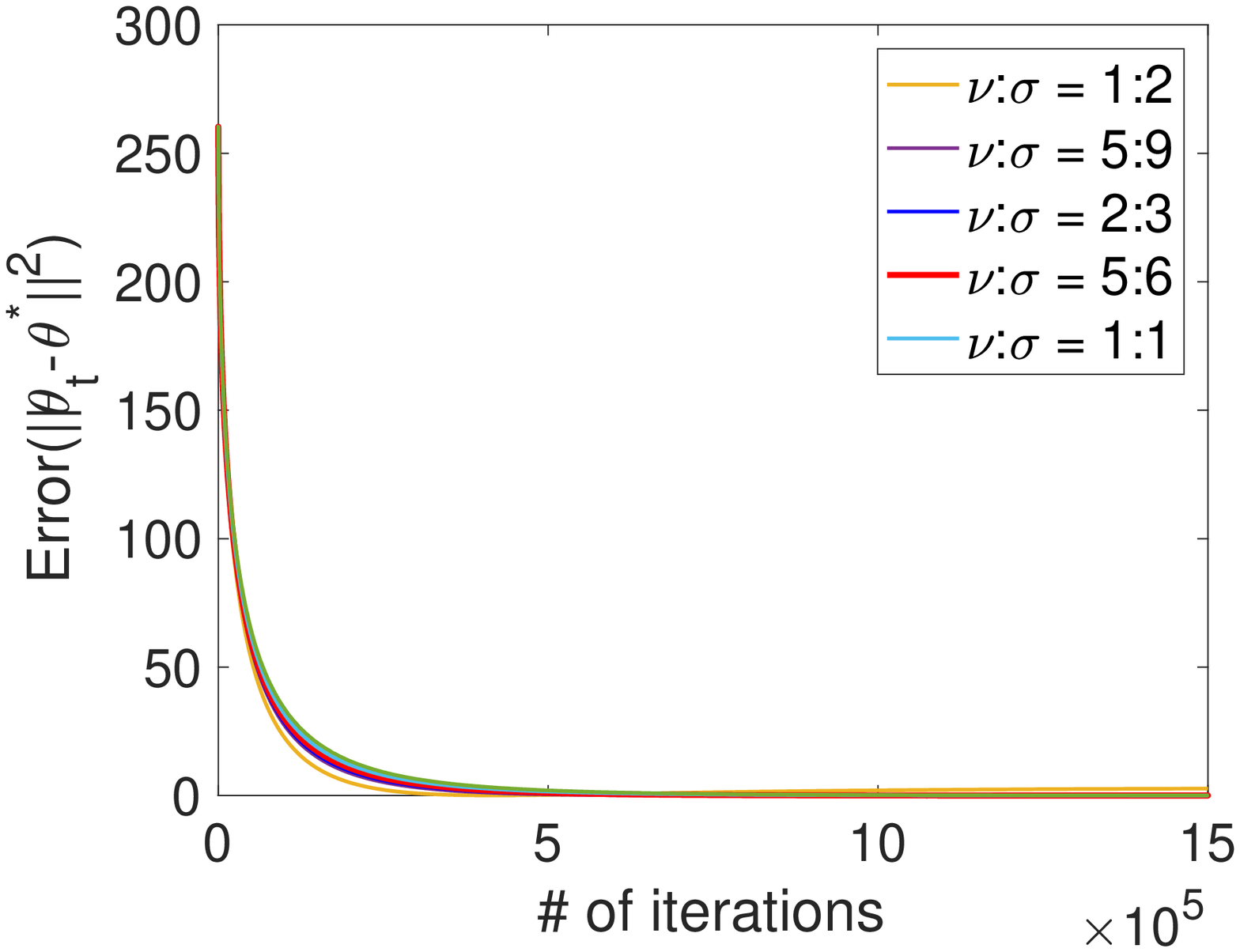}\includegraphics[width=1.3in]{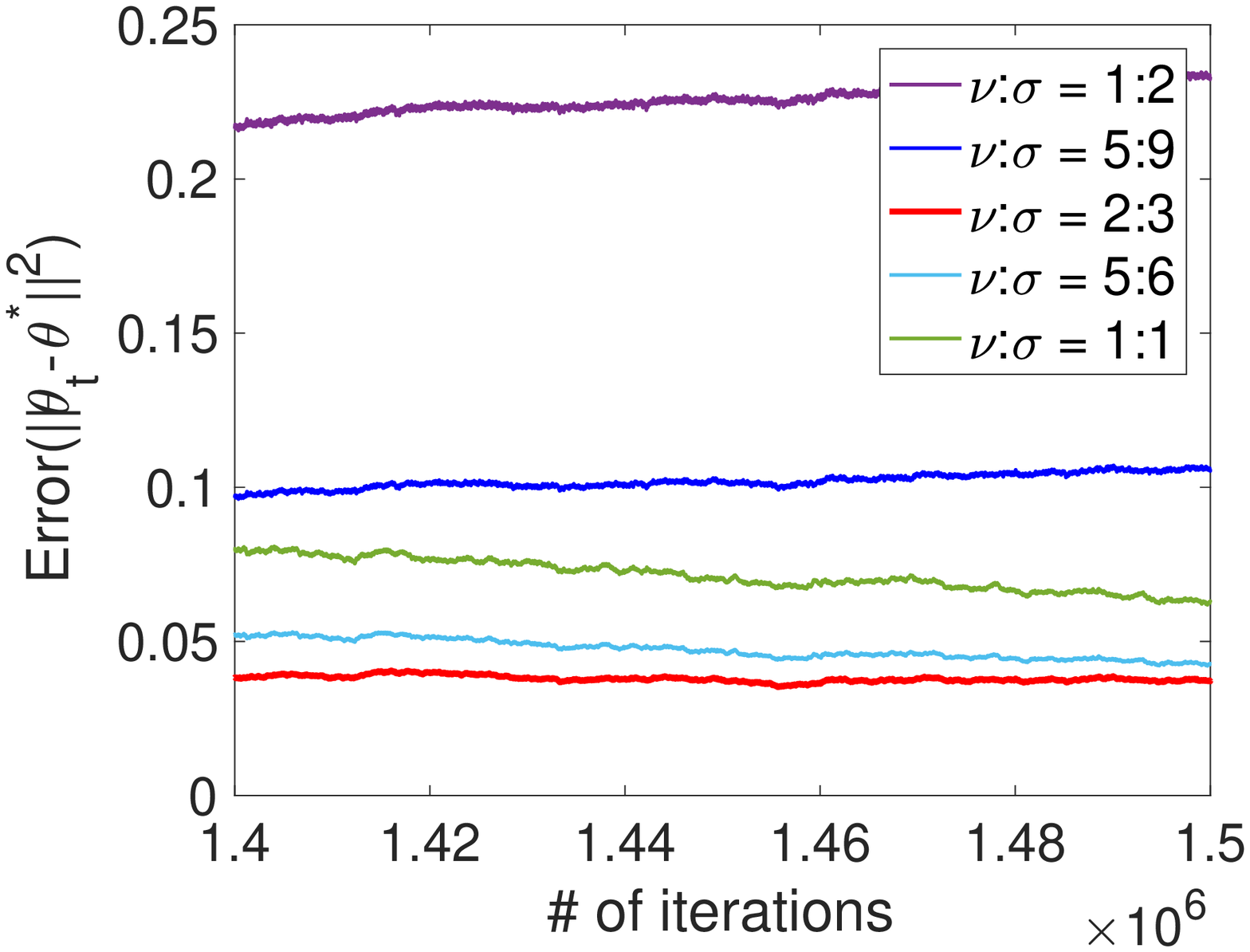}}
	\subfigure[$\sigma=0.6$ (left: full; right: tail)]{\includegraphics[width=1.3in]{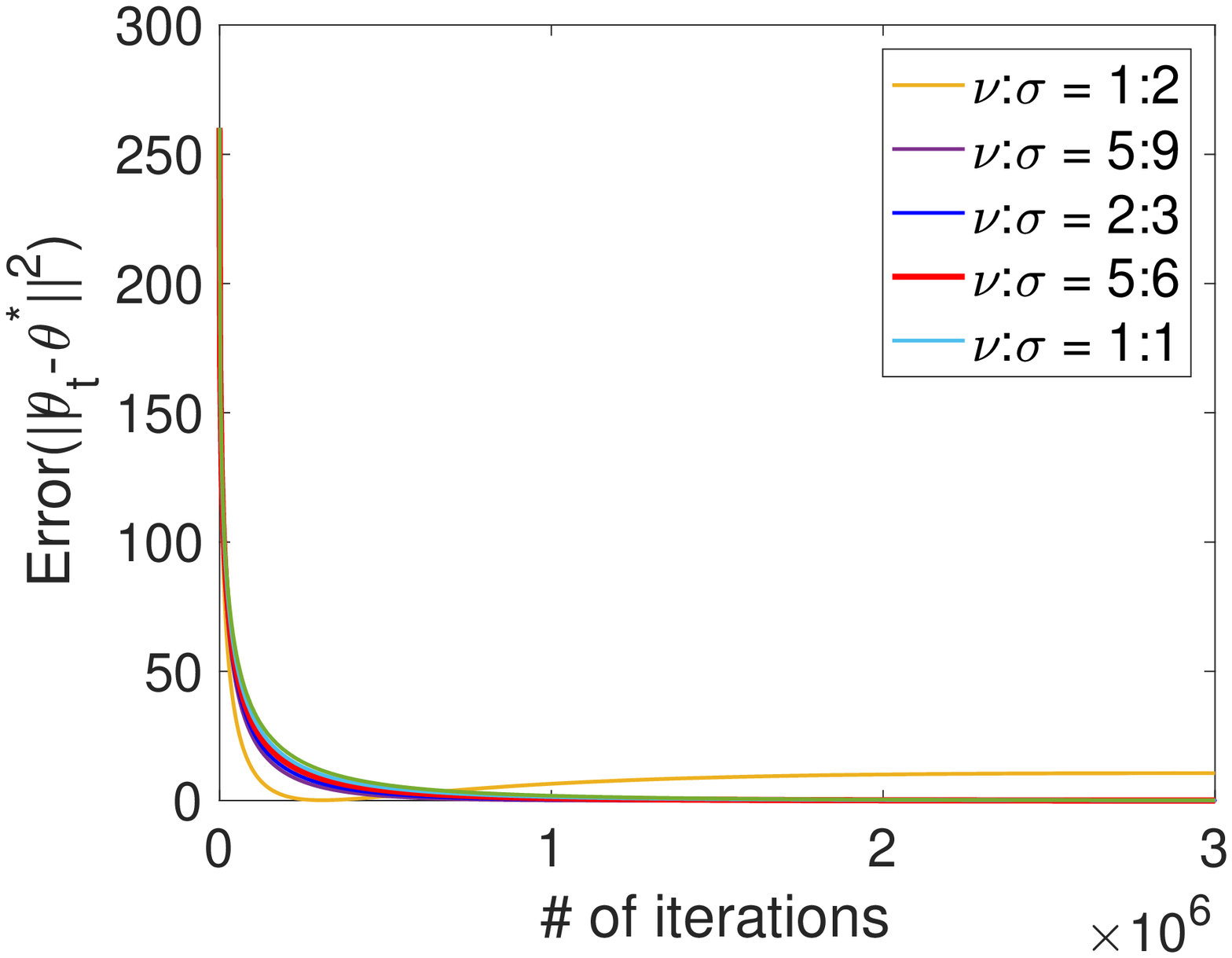}\includegraphics[width=1.3in]{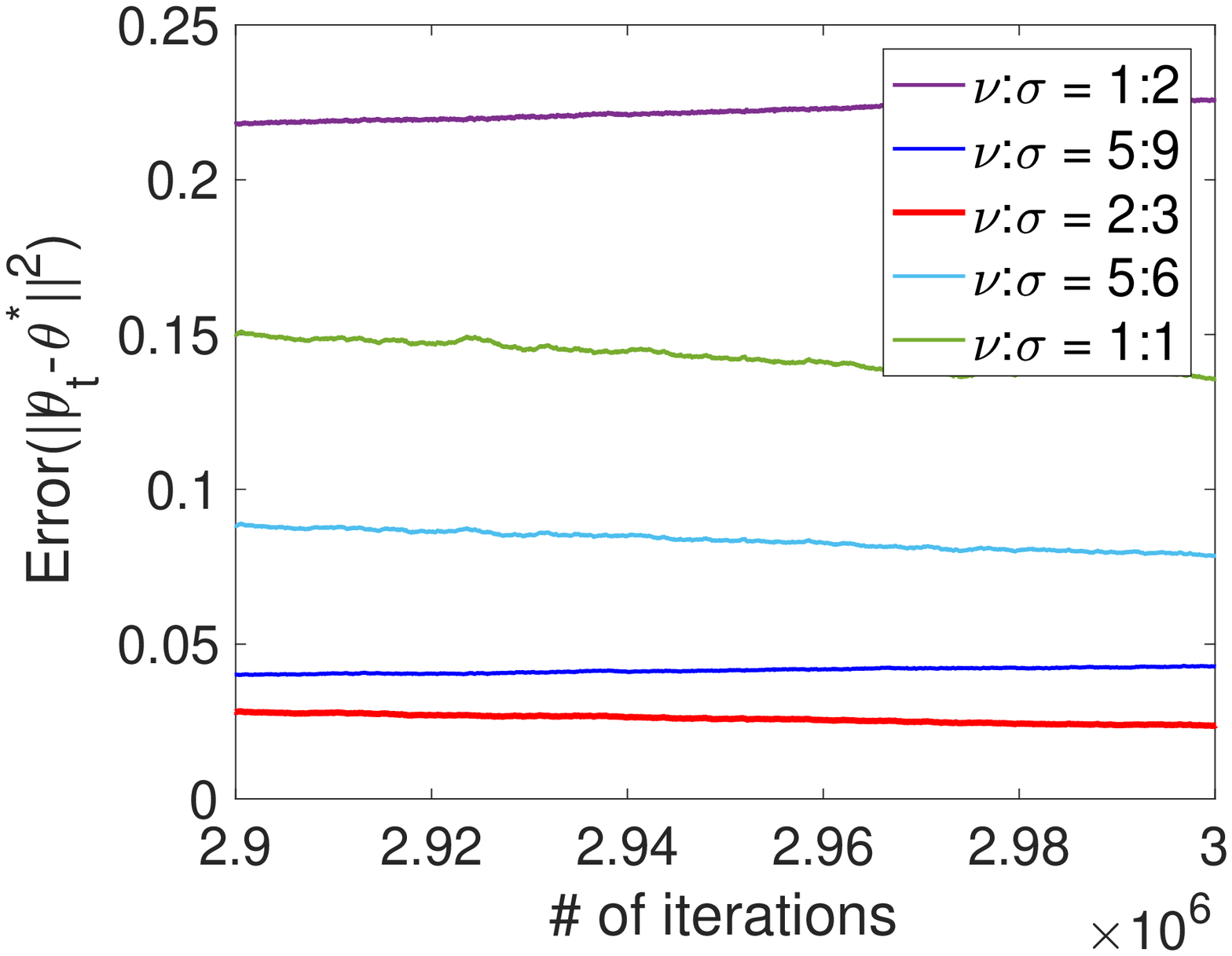}}
	\caption{Comparison among diminishing stepsize settings. For settings $\sigma=0.45$ and $\sigma=0.6$, the case $\nu:\sigma=1:3$ has much larger training error than others and is not included in the tail figures.} 
	\label{fig: 1}
	\vspace{-0.6cm}
\end{figure}

\subsection{Constant Stepsize vs Diminishing Stepsize}\label{exp: sub2}
In this subsection, we compare the error decay of TDC under diminishing stepsize with that of TDC under four different constant stepsizes. For diminishing stepsize, we set $c_\alpha=c_\beta$ and $\sigma=\frac{3}{2}\nu$, and tune their values to the best, which are given by $c_\alpha=c_\beta=1.8$, $\sigma=\frac{3}{2}\nu=0.45$. For the four constant-stepsize cases, we fix $\alpha$ for each case, and tune $\beta$ to the best. The resulting parameter settings are respectively as follows: $\alpha_t=0.01,\,\beta_t=0.006$; $\alpha_t=0.02,\,\beta_t=0.008$; $\alpha_t=0.05,\,\beta_t=0.02$; and $\alpha_t=0.1,\,\beta_t=0.02$. 
The results are reported in \Cref{fig: 2}, in which for both the training and tracking errors, the left plot illustrates the overall iteration process and the right plot illustrates the corresponding zoomed error tails. The results suggest that although some large constant stepsizes ($\alpha_t=0.05,\,\beta_t=0.02$ and $\alpha_t=0.1,\,\beta_t=0.02$) yield initially faster convergence than diminishing stepsize, they eventually oscillate around a large neighborhood of $\theta^*$ due to the large tracking error. Small constant stepsize ($\alpha_t=0.02,\,\beta_t=0.008$ and $\alpha_t=0.01,\,\beta_t=0.006$) can have almost the same asymptotic accuracy as that under diminishing stepsize, but has very slow convergence rate. We can also observe strong correlation between the training and tracking errors under constant stepsize, i.e., larger training error corresponds to larger tracking error, which corroborates \Cref{thm2} and suggests that the accuracy of TDC heavily depends on the decay of the tracking error $\ltwo{z_t}$.
\begin{figure}[h]
	\vspace{-0.5cm}
	\centering 
	\subfigure[Training error (left: full; right: tail)]{\includegraphics[width=1.3in]{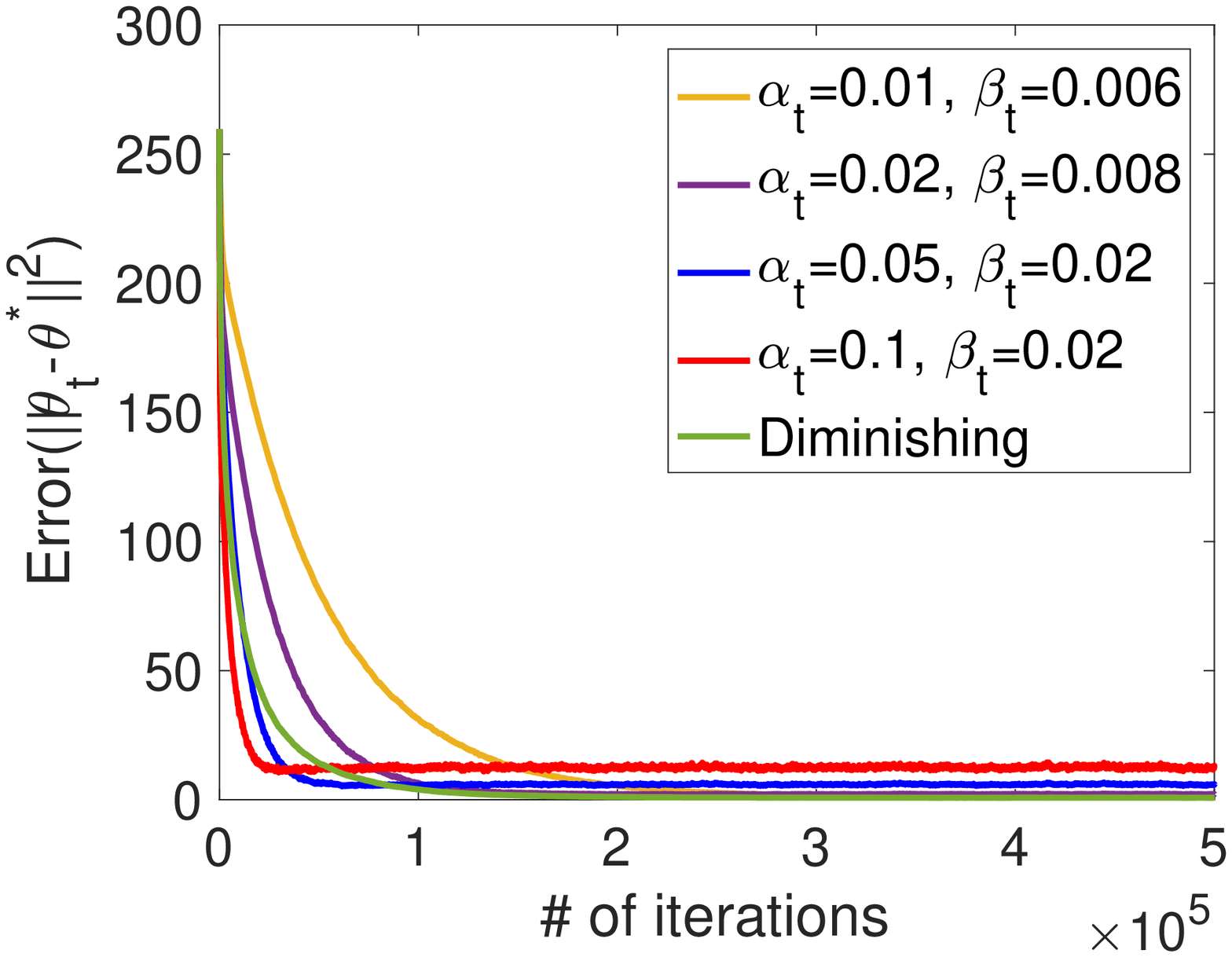}\includegraphics[width=1.3in]{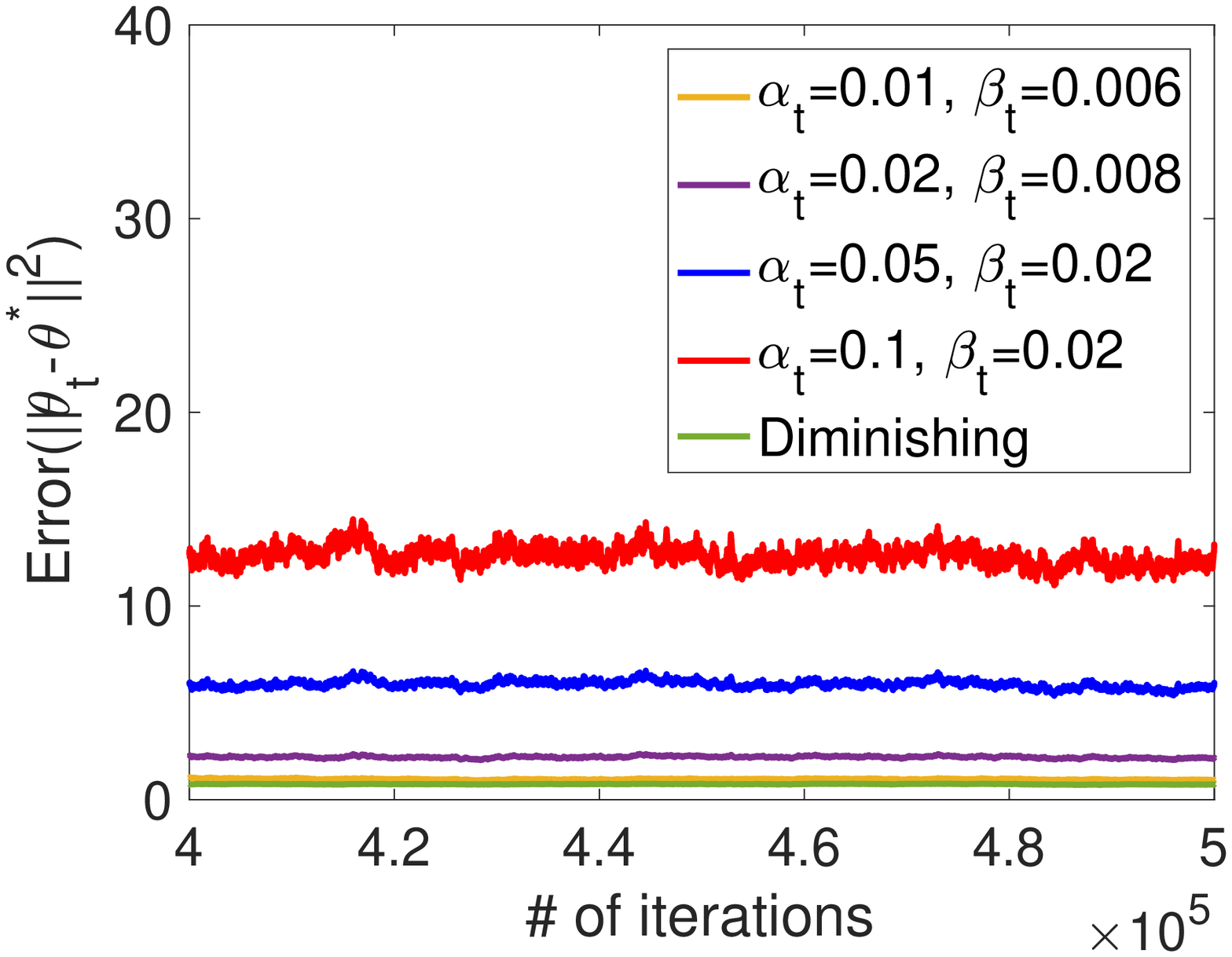}}
	\subfigure[Tracking error (left: full; right: tail)]{\includegraphics[width=1.3in]{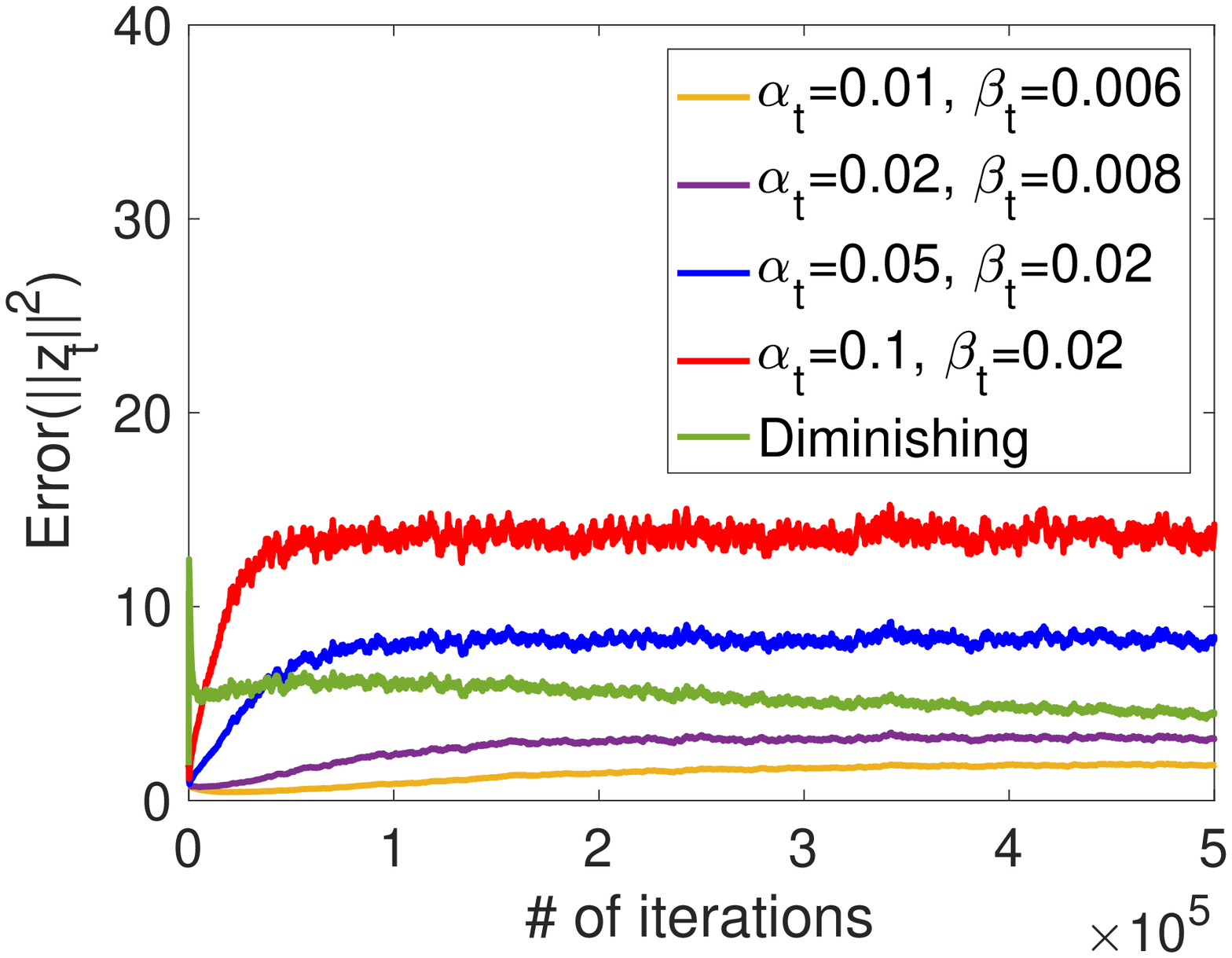}\includegraphics[width=1.3in]{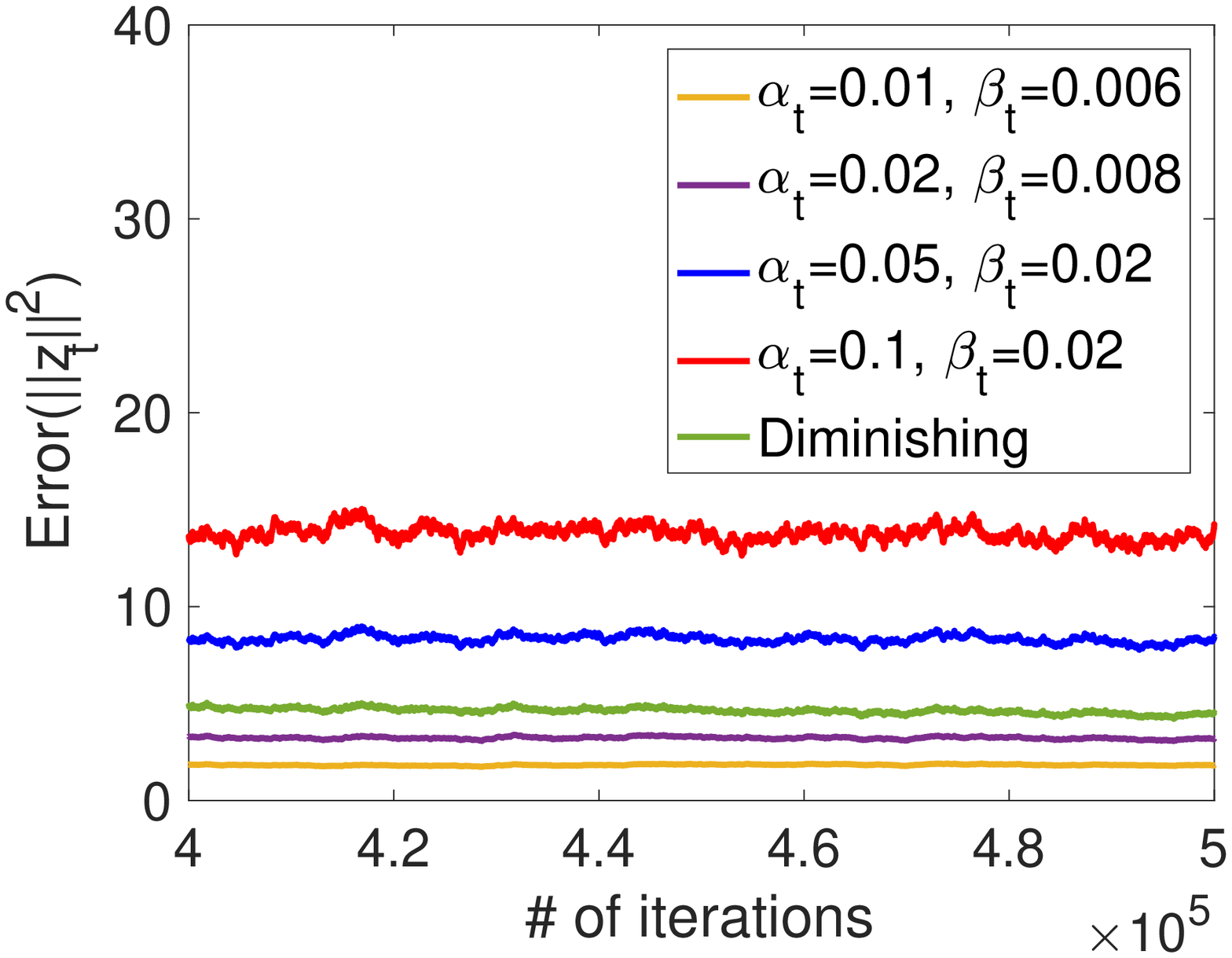}}
	\caption{Comparison between TDC updates under constant stepsizes and diminishing stepsize. } 
	\label{fig: 2}
	\vspace{-0.6cm}
\end{figure}
\subsection{Blockwise Diminishing Stepsize}\label{exp: sub3}
In this subsection, we compare the error decay of TDC under blockwise diminishing stepsize with that of TDC under diminishing stepsize and constant stepsize. We use the best tuned parameter settings as listed in Section \ref{exp: sub2} for the latter two algorithms, i.e., $c_\alpha=c_\beta=1.8$ and $\sigma=\frac{3}{2}\nu=0.45$ for diminishing stepsize, and $\alpha_t=0.1,\, \beta_t=0.02$ for constant stepsize. We report our results in Figure \ref{fig: 3}. It can be seen that TDC under blockwise diminishing stepsize 
converges faster than that under diminishing stepsize and almost as fast as that under constant stepsize. Furthermore, TDC under blockwise diminishing stepsize also has comparable training error as that under diminishing stepsize. Since the stepsize decreases geometrically blockwisely, the algorithm approaches to a very small neighborhood of $\theta^*$ in the later blocks. We can also observe that the tracking error under blockwise diminishing stepsize decreases rapidly blockwisely.
\begin{figure}[h]
	\vspace{-0.6cm}
	\centering 
	\subfigure[Training error (left: full; right: tail)]{\includegraphics[width=1.3in]{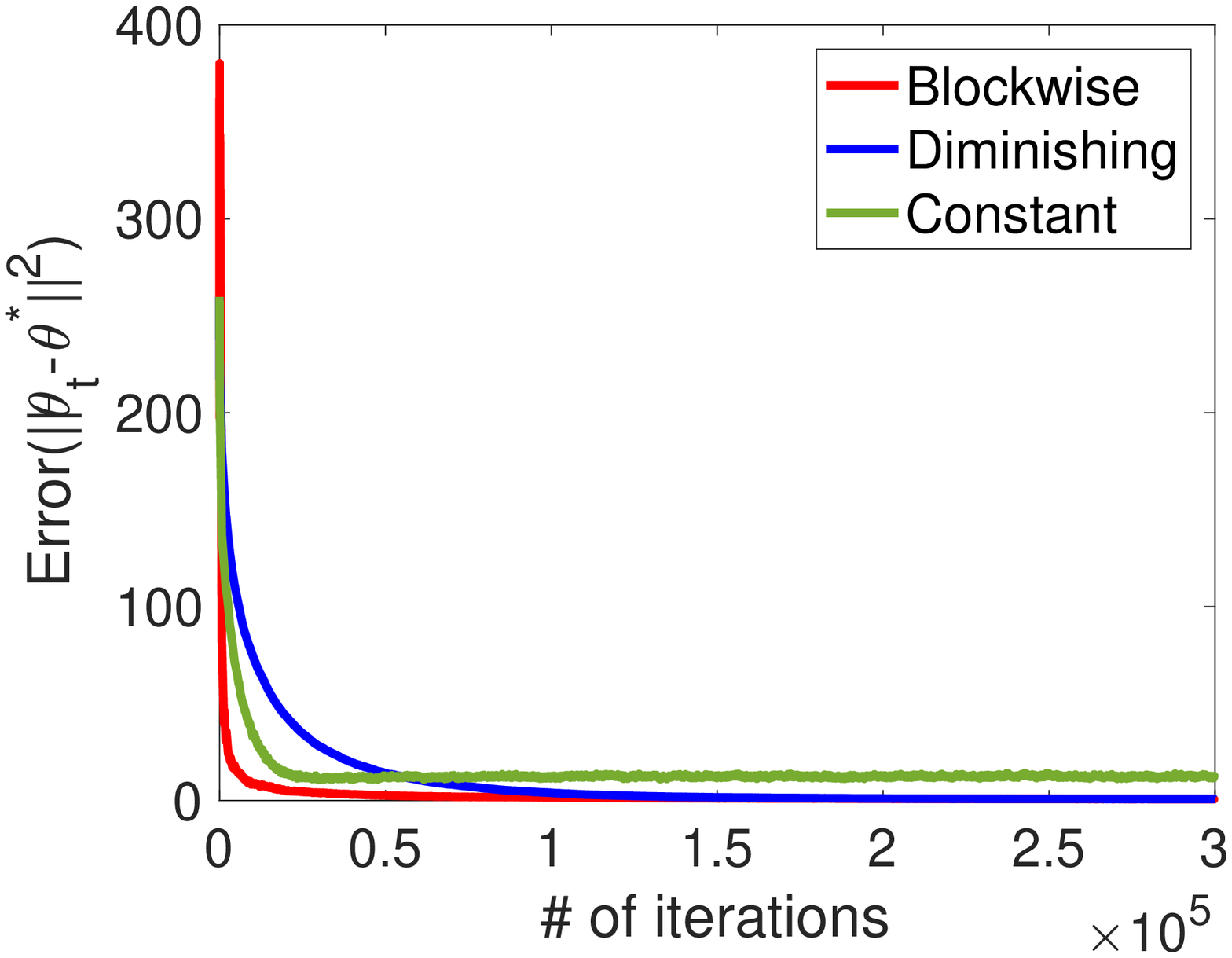}\includegraphics[width=1.3in]{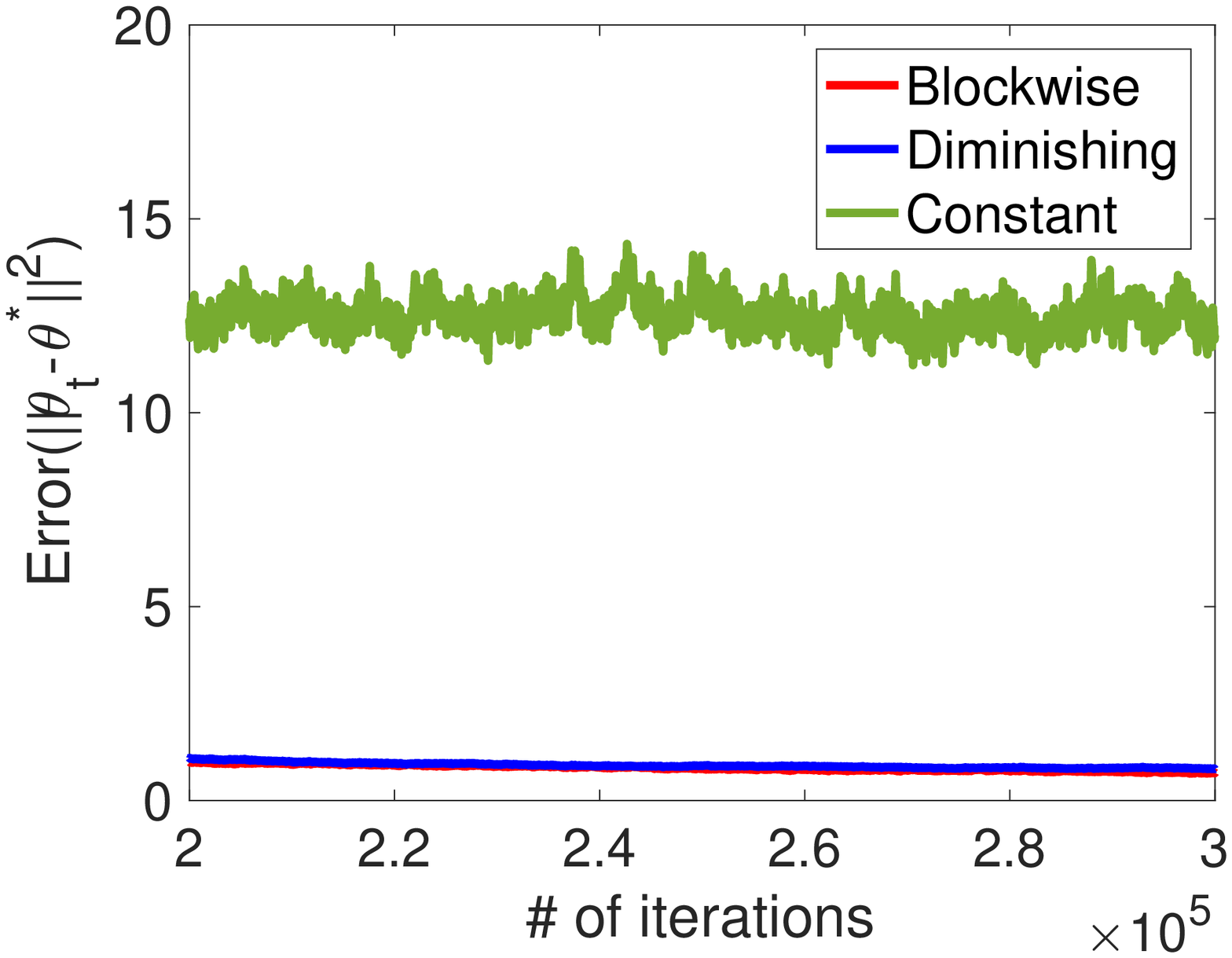}}
	\subfigure[Tracking error (left: full; right: tail)]{\includegraphics[width=1.3in]{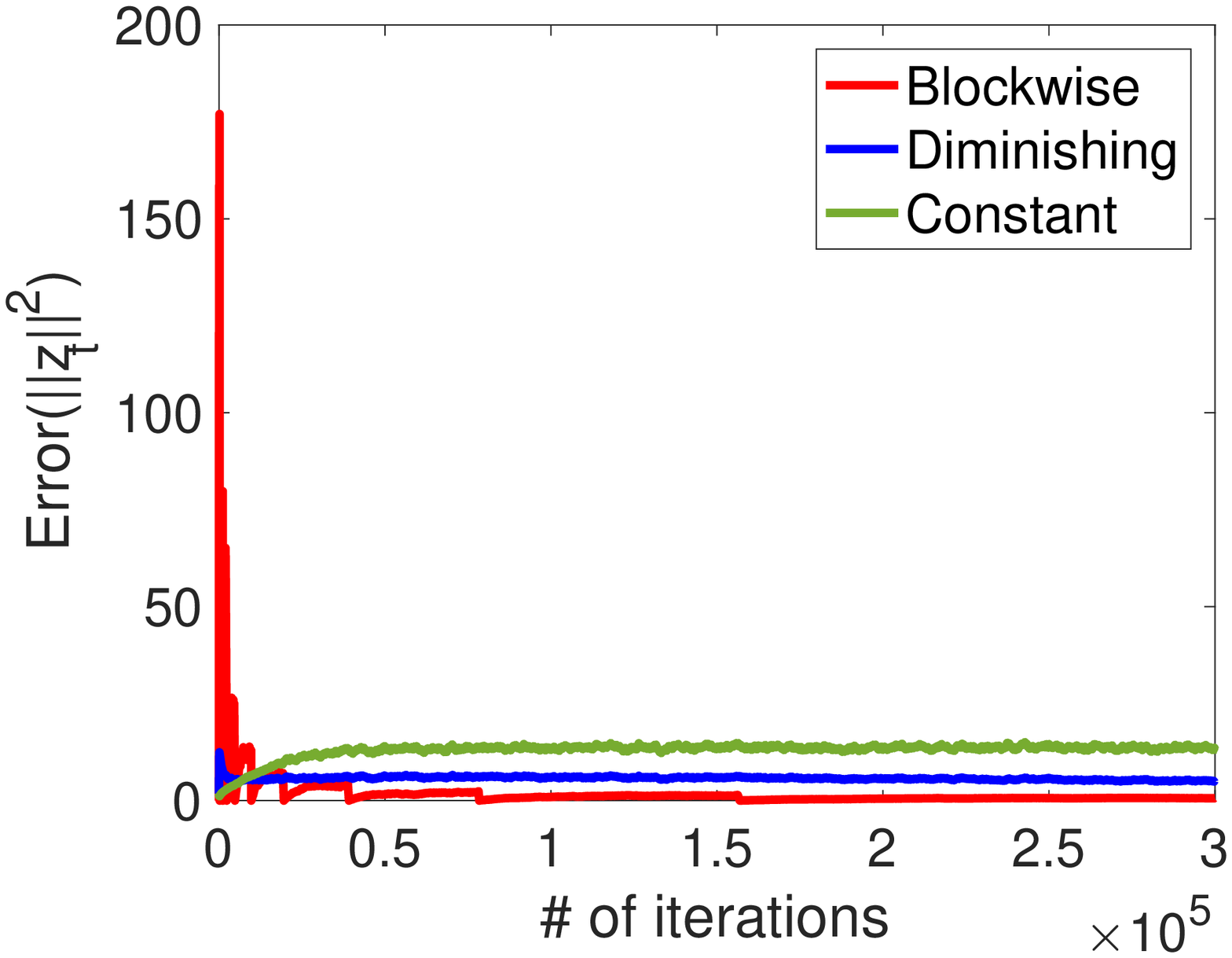}\includegraphics[width=1.3in]{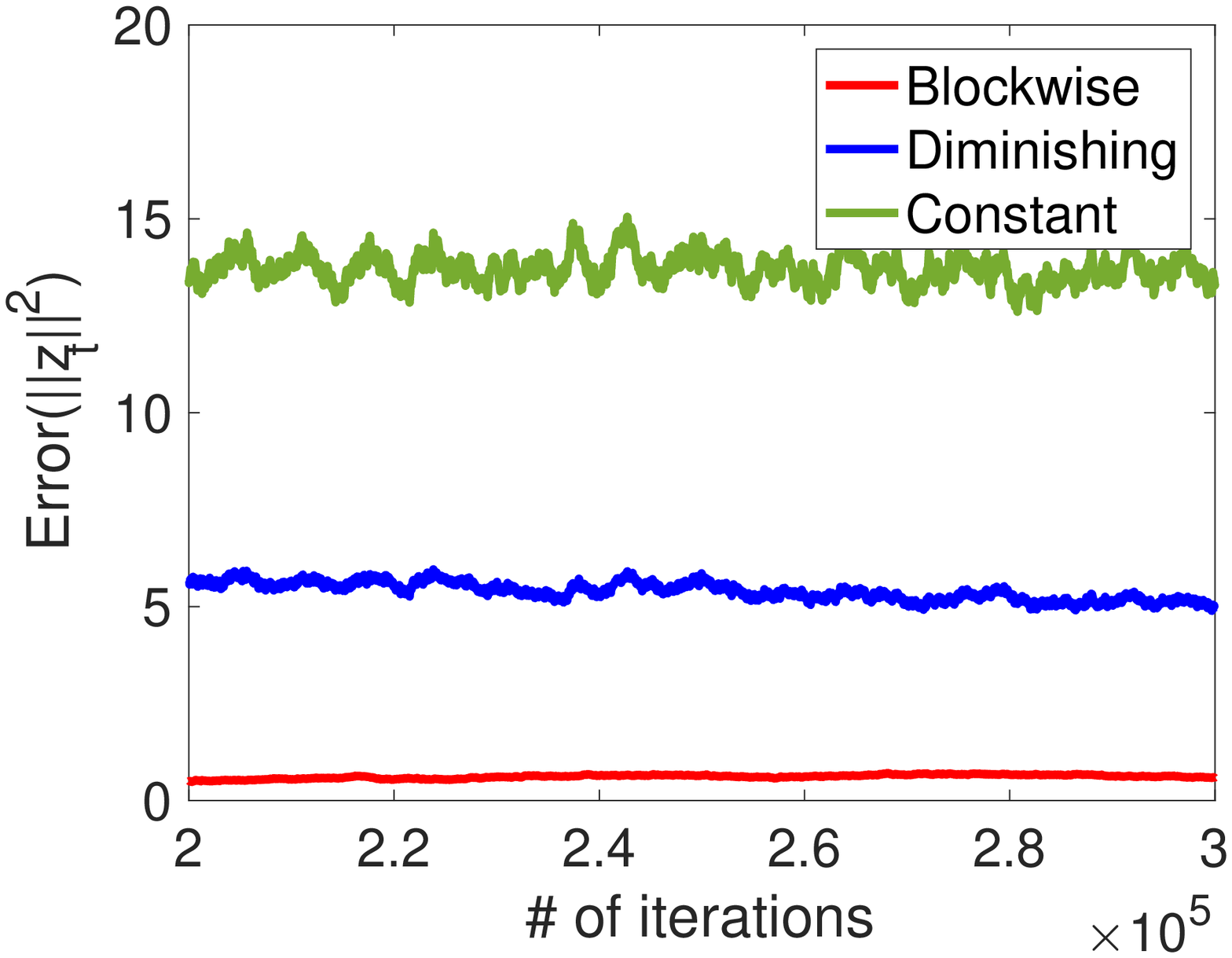}}
	\caption{Comparison between TDC updates under blockwise diminishing stepsizes, diminishing stepsize and constant stepsize} 
	\label{fig: 3}
	\vspace{-0.6cm}
\end{figure}
\vspace{-0.2cm}
\section{Conclusion}
\vspace{-0.1cm}
In this work, we provided the first non-asymptotic analysis for the two time-scale TDC algorithm over Markovian sample path. We developed a novel technique to handle the accumulative tracking error caused by the two time-scale update, using which we characterized the non-asymptotic convergence rate with general diminishing stepsize and constant stepsize. We also proposed a blockwise diminishing stepsize scheme for TDC and proved its convergence. Our experiments demonstrated the performance advantage of such an algorithm over both the diminishing and constant stepsize TDC algorithms. Our technique for non-asymptotic analysis of two time-scale algorithms can be applied to studying other off-policy algorithms such as actor-critic \cite{maei2018offpolicyac} and gradient Q-learning algorithms \cite{maei2010gq}.

\bibliographystyle{abbrv}
\bibliography{ref}
\newpage
\appendix
\noindent {\Large \textbf{Supplementary Materials}}
\section{Technical Proofs for TDC under Decreasing Stepsize}\label{proof_thm1}
We present the proof of Theorem \ref{thm1} in four subsections.  Section \ref{proof_thm1_sketch} provides the proof sketch. Section \ref{proof_thm1_main} contains the main part of the proof. Section \ref{proof_thm1_fast} includes all technical lemmas for the convergence proof of fast time-scale iteration, and Section \ref{proof_thm1_slow} includes all the technical lemmas for the convergence proof of the slow time-scale iteration.
\subsection{Proof Sketch of Theorem \ref{thm1}}\label{proof_thm1_sketch}
\begin{proof}[Proof Sketch of Theorem \ref{thm1}]
The proof consists of four steps as we briefly describe here. The details are provided in Appendix \ref{proof_thm1_main}.

\textbf{Step 1.} \emph{Formulate training and tracking error updates.} In stead of investigating the convergence of $\{ \theta_t \}$ and $\{ w_t \}$ directly, we substitute $z_t$ into the TDC update \eqref{algorithm1_1}-\eqref{algorithm1_2} and analyze the update of TDC in terms of $\{ \theta_t \}$ and tracking error $\{ z_t \}$.

\textbf{Step 2.} \emph{Derive preliminary bound on $\mE\ltwo{z_t}^2$.} We decompose the mean square tracking error $\mE\ltwo{z_t}^2$ into an exponentially decaying term, a variance term, a bias term, and a slow drift term, and bound each term individually. We obtain a preliminary upper bound on $\mE\ltwo{z_t}^2$ with order $\mathcal{O}(1/t^{\sigma-\nu})$.

\textbf{Step 3.} \emph{Recursively refine bound on $\mE\ltwo{z_t}^2$.} By recursively substituting the preliminary bound of $\mE\ltwo{z_t}^2$ into the slow drift term, we obtain the refined decay rate $\mE\ltwo{z_t}^2=\mathcal{O}(h(\sigma,\nu))$.

\textbf{Step 4.} \emph{Derive bound on $\mE\ltwo{\theta_t-\theta^*}^2$.} We decompose the training error $\mE\ltwo{\theta_t-\theta^*}^2$ into an exponentially decaying term, a variance term, a bias term, and a tracking error term, and bound each term individually. We then recursively substitute the decay rate of $\mE\ltwo{z_t}^2$ and $\mE\ltwo{\theta_t-\theta^*}^2$ into the tracking error term to obtain an upper bound on the training error with order $\mathcal{O}(h(\sigma,\nu)^{1-\epsilon^\prime})$. Combining each term yields the final bound of $\mE\ltwo{\theta_t-\theta^*}^2$ in \eqref{thm1eq1_1}.
\end{proof}
\subsection{Proof of Theorem \ref{thm1}}\label{proof_thm1_main}
We provide the proof of Theorem \ref{thm1} following four steps.
	
	\textbf{Step 1.} \emph{Formulation of training error and tracking error update.} We define the tracking error vector $z_t = w_t + C^{-1}(b+A\theta_t)$. By substituting $z_t$ into \eqref{algorithm1_1}-\eqref{algorithm1_2}, we can rewrite the update rule of TDC in terms of $\theta_t$ and $z_t$ as follows:
	\begin{flalign}
	&\theta_{t+1}=\mcpi_{R_\theta}\left(\theta_t + \alpha_t(f_1(\theta_t, O_t)+g_1(z_t, O_t))\right),\label{algorithm2_1}\\
	&z_{t+1}=\mcpi_{R_w}\left( z_t + \beta_t(f_2(\theta_t, O_t)+g_2(z_t,O_t))-C^{-1}(b+A\theta_t) \right) + C^{-1}(b+A\theta_{t+1}),\label{algorithm2_2}
	\end{flalign}
	where 
	\begin{flalign*}
	&f_1(\theta_t, O_t) = (A_t-B_tC^{-1}A)\theta_t+(b_t-B_tC^{-1}b), \qquad g_1(z_t, O_t)=B_t z_t,\\
	&f_2(\theta_t, O_t)=(A_t-C_tC^{-1}A)\theta_t+(b_t-C_tC^{-1}b),\qquad  g_2(z_t,O_t) = C_t z_t,
	\end{flalign*}
	with $O_t=(s_t, a_t, r_t, s_{t+1})$ denoting the observation at time step $t$. We further define
	\begin{flalign*}
	\bar{f}_1(\theta_t) = (A-BC^{-1}A)\theta_t+(b-BC^{-1}b), \qquad \bar{g}_1(z_t)=B z_t,\qquad \bar{g}_2(z_t) = C z_t.
	\end{flalign*}
	
	\textbf{Step 2.} \emph{Derive preliminary bound on $\mE\ltwo{z_t}^2$.} We bound the recursion of the tracking error vector $z_t$ in \eqref{algorithm2_2} as follows. For any $t\geq 0$, we derive
	\begin{flalign*}
	\ltwo{z_{t+1}}^2&= \ltwo{\mcpi_{R_w}\left( z_t + \beta_t(f_2(\theta_t, O_t)+g_2(z_t,O_t))-C^{-1}(b+A\theta_t) \right) + C^{-1}(b+A\theta_{t+1})}^2\\
	&= \ltwo{\mcpi_{R_w}\left( z_t + \beta_t(f_2(\theta_t, O_t)+g_2(z_t,O_t))-C^{-1}(b+A\theta_t) \right) + \mcpi_{R_w}\left(C^{-1}(b+A\theta_{t+1})\right)}^2\\
	&\leq  \ltwo{z_t + \beta_t(f_2(\theta_t, O_t)+g_2(z_t,O_t))+C^{-1}A(\theta_{t+1}-\theta_t)} ^2\\
	&= \ltwo{z_t}^2 + 2\beta_t \langle f_2(\theta_t, O_t), z_t\rangle+2\beta_t \langle g_2(z_t, O_t), z_t \rangle + 2\langle C^{-1}A(\theta_{t+1}-\theta_t),z_t\rangle \\
	&\quad + \ltwo{ \beta_t f_2(\theta_t, O_t) + \beta_t g_2(z_t,O_t) + C^{-1}A(\theta_{t+1}-\theta_t)}^2\\
	&\leq \ltwo{z_t}^2 +2\beta_t \langle \bar{g}_2(z_t), z_t \rangle + 2\beta_t \langle f_2(\theta_t, O_t), z_t\rangle+2\beta_t \langle g_2(z_t, O_t) - \bar{g}_2(z_t), z_t \rangle\\
	&\quad + 2\langle C^{-1}A(\theta_{t+1}-\theta_t),z_t\rangle \\
	&\quad + 3\beta_t^2 \ltwo{f_2(\theta_t, O_t)}^2 + 3\beta_t^2\ltwo{g_2(z_t,O_t)}^2 + 3\ltwo{C^{-1}A(\theta_{t+1}-\theta_t)}^2\\
	&\leq \ltwo{z_t}^2 +2\beta_t \langle Cz_t, z_t \rangle + 2\beta_t \langle f_2(\theta_t, O_t), z_t\rangle+2\beta_t \langle g_2(z_t, O_t) - \bar{g}_2(z_t), z_t \rangle \\
	&\quad + 2\langle C^{-1}A(\theta_{t+1}-\theta_t),z_t\rangle + 3\beta_t^2\ltwo{f_2(\theta_t, O_t)}^2 + 3\beta_t^2\ltwo{g_2(z_t,O_t)}^2 \\
	&\quad + 3\alpha_t^2\ltwo{C^{-1}}^2\ltwo{A}^2\ltwo{ f_1(\theta_t,O_t) + g_1(z_t,O_t)}^2\\
	&\leq (1-\beta_t|\lambda_w|)\ltwo{z_t}^2 + 2\beta_t\zeta_{f_2}(\theta_t,z_t,O_t) + 2\beta_t\zeta_{g_2}(z_t,O_t)+ 2\langle C^{-1}A(\theta_{t+1}-\theta_t),z_t\rangle\\
	&\quad + 3\beta_t^2 K_{f_2}^2 + 3\beta_t^2 K_{g_2}^2 + 6\alpha_t^2\ltwo{C^{-1}}^2\ltwo{A}^2(K_{f_1}^2 + K_{g_1}^2),
	\end{flalign*}
	where $\lambda_{\max}(2C)\leq\lambda_w<0$, $\zeta_{f_2}(\theta_t,z_t,O_t) = \langle f_2(\theta_t, O_t), z_t\rangle$, $\zeta_{g_2}(z_t,O_t)=\langle g_2(z_t, O_t) - \bar{g}_2(z_t), z_t \rangle$. $K_{f_1}$ and $K_{g_1}$, $K_{f_2}$ and $K_{g_1}$ are positive constants, please refer to Lemma \ref{lemma: boundedf1}, \ref{lemma: boundedg1}, \ref{lemma: boundedf2} and \ref{lemma: boundedg2} for their definitions. Then, defining $K_{r_1}=\ltwo{C^{-1}}\ltwo{A}(K_{f_1} + K_{g_1})$ and taking the expectation over $\mathcal{F}_{t+1}$ (the filtration up to state $s_{t+1}$) on both sides, we have
	\begin{flalign}\label{eq: expzincremental}
	\mE\ltwo{z_{t+1}}^2&\leq (1-\beta_t|\lambda_w|)\mE\ltwo{z_t}^2 + 2\beta_t\mE[\zeta_{f_2}(\theta_t,z_t,O_t)] + 2\beta_t\mE[\zeta_{g_2}(z_t,O_t)] \nonumber\\
	&\quad + 2\mE\langle C^{-1}A(\theta_{t+1}-\theta_t),z_t\rangle + 3\beta_t^2 K_{f_2}^2 + 3\beta_t^2 K_{g_2}^2 + 6\alpha_t^2K^2_{r_1}.
	\end{flalign}
	From the definition of $\beta_t$, $|\beta_t|\leq c_\beta$ for all $t\geq 0$. If $c_\beta|\lambda_w|<1$, then we have $0<1-\beta_t|\lambda_w|<1$. Telescoping the above inequality yields that
	\begin{flalign}\label{eq: expzinter}
	\mE\ltwo{z_{t+1}}^2&\leq \left[ \prod_{i=0}^{t}(1-\beta_i|\lambda_w|) \right]\ltwo{z_0}^2 \nonumber\\
	&\quad + 2\sum_{i=0}^{t} \left[ \prod_{k=i+1}^{t}(1-\beta_k|\lambda_w|) \right]\beta_i[\zeta_{f_2}(\theta_i,z_i,O_i)] \nonumber\\
	&\quad + 2\sum_{i=0}^{t} \left[ \prod_{k=i+1}^{t}(1-\beta_k|\lambda_w|) \right]\beta_i[\zeta_{g_2}(z_i,O_i)]\nonumber\\
	&\quad + 2\sum_{i=0}^{t} \left[ \prod_{k=i+1}^{t}(1-\beta_k|\lambda_w|) \right] \mE\langle C^{-1}A(\theta_{i+1}-\theta_i),z_i\rangle\nonumber\\
	&\quad + 3(K^2_{f_2} + K^2_{g_2})\sum_{i=0}^{t} \left[ \prod_{k=i+1}^{t}(1-\beta_k|\lambda_w|) \right]\beta_i^2 + 6K^2_{r_1}\sum_{i=0}^{t} \left[ \prod_{k=i+1}^{t}(1-\beta_k|\lambda_w|) \right]\alpha_i^2.
	\end{flalign}
	Since $1-\beta_i|\lambda_w|\leq e^{-\beta_i|\lambda_w|}$ and using the fact that $(1+i)^{-\nu}\geq(1+i)^{-\sigma}$ for all $i\geq 0$, we have
	\begin{flalign}
	\mE\ltwo{z_{t+1}}^2&\leq e^{-|\lambda_w|\sum_{i=0}^{t}\beta_i}\ltwo{z_0}^2 \label{fast1z}\\ 
	&\quad + 2\sum_{i=0}^{t} e^{-|\lambda_w|\sum_{k=i+1}^{t}\beta_k} \beta_i\mE[\zeta_{f_2}(\theta_i,z_i,O_i)] \label{bias1z} \\
	&\quad + 2\sum_{i=0}^{t} e^{-|\lambda_w|\sum_{k=i+1}^{t}\beta_k}\beta_i\mE[\zeta_{g_2}(z_i,O_i)] \label{bias2z}\\
	&\quad + 2\sum_{i=0}^{t} e^{-|\lambda_w|\sum_{k=i+1}^{t}\beta_k}\mE\langle C^{-1}A(\theta_{i+1}-\theta_i),z_i\rangle \label{interac1}\\
	&\quad + 3\max\{1,\frac{c^2_\alpha}{c^2_\beta}\}(K^2_{f_2} + K^2_{g_2} + 3K^2_{r_1})\sum_{i=0}^{t} e^{-|\lambda_w|\sum_{k=i+1}^{t} \beta_k }\beta_i^2. \label{variance1z}
	\end{flalign}
	The first term \eqref{fast1z} captures how fast the tracking error vector $z_t$ converges to the neighborhood of zero, the second term \eqref{bias1z} and third term \eqref{bias2z} are the accumulative bias term induced by the biased gradient estimator $f_2(\theta_t,O_t)$ and $g_2(z_t,O_t)$ respectively, the forth term \eqref{interac1} is the accumulative error caused by the slow drift, and the last term \eqref{variance1z} is the accumulative variance. Combining Lemma \ref{lemma: boundaccumulatebiasf2}, Lemma \ref{lemma: boundaccumulatebiasg2}, Lemma \ref{lemma: 1thbdinter} and applying Lemma \ref{lemma:variance} to the above upper bound, we obtain
	\begin{flalign}\label{eq: firstboundz}
	&\mE\ltwo{z_{t+1}}^2\nonumber\\
	&\leq e^{\frac{-|\lambda_w| c_\beta}{1-\nu}[(1+t)^{1-\nu}-1]}\ltwo{z_0}^2 + 8(R_wK_{f_2} + 2R_wK_{g_2})\frac{e^{|\lambda_w| c_\beta}}{|\lambda_w|} \frac{c_\beta}{(1+t)^\nu}\nonumber\\
	&\quad + 2c_\beta[ K_{r_3} + L_{g_2,z}K_{r_2}] \tau_\beta \frac{e^{|\lambda_w| c_\beta}}{|\lambda_w|} e^{\frac{-|\lambda_w| c_\beta}{1-\nu}[ (1+t)^\nu - (1+\tau_\beta)^\nu ]}\nonumber \\
	&\quad + 4(K_{r_3} + L_{g_2,z}K_{r_2} )\tau_\beta \frac{e^{|\lambda_w| c_\beta/2}}{|\lambda_w|} (e^{\frac{-|\lambda_w| c_\beta}{2(1-\nu)}[(t+1)^{1-\nu}-1]}D_1 \mathbbm{1}_{\{ \tau_\beta+1< i_{d_1} \}}+\beta_{t-\tau_\beta})\nonumber\\
	&\quad + \frac{4c_\alpha(1+\gamma)\rho_{\max}}{c_\beta\lambda_{cm}}R_\theta R_w \frac{2e^{|\lambda_w| c_\beta/2}}{|\lambda_w|} \Big(e^{\frac{-|\lambda_w| c_\beta}{2(1-\nu)}[(1+t)^{1-\nu}-1]}D_2 + \frac{1}{(1+t)^{\sigma-\nu}} \Big)\nonumber\\
	&\quad + 3\max\left\{1,\frac{c^2_\alpha}{c^2_\beta}\right\}(K^2_{f_2} + K^2_{g_2} + K^2_{r_1}) \frac{2c_\beta e^{|\lambda_w| c_\beta/2}}{|\lambda_w|}(e^{\frac{-|\lambda_w| c_\beta}{2(1-\nu)}[(t+1)^{1-\nu}-1]}D_3 +\beta_t)
	\end{flalign}
	Where, $D_3= e^{(|\lambda_w| c_\beta/2)\sum_{k=0}^{i_{d_3}}\beta_k} $, with $i_{d_3} = (\frac{|\lambda_w| c_\beta}{2\nu})^{\frac{1}{1-\nu}}$, and $\tau_\beta = \min\{ i\in\mN|m\rho^i\leq \beta_t \}$.
	
	\textbf{Step 3.} \emph{Recursively refine bound on $\mE\ltwo{z_t}^2$.} By applying Lemma \ref{lemma: mixorderbound2}, we have 
	\begin{flalign*}
		\mE\ltwo{z_t}^2\leq \mathcal{O}\Big( \frac{\log t}{t^\nu} \Big) + \mathcal{O}(h(\sigma, \nu)),
	\end{flalign*}
	where
\begin{equation*}
h(\sigma, \nu)=\left\{
\begin{array}{lr}
\frac{1}{t^\nu}, & \sigma>1.5\nu, \\
\frac{1}{t^{2(\sigma-\nu)-\epsilon}}. &  \nu<\sigma\leq1.5\nu,
\end{array}
\right.
\end{equation*}
where $\epsilon\in (0,\sigma-\nu]$ can be any small constant.

    \textbf{Step 4.} \emph{Derive bound on $\mE\ltwo{\theta_t-\theta^*}^2$}. For the recursion of $\theta_t$ in \eqref{algorithm2_1}, for any $t\geq 0$,
	\begin{flalign}
	\ltwo{\theta_{t+1}-\theta^*}^2 &=\ltwo{\mcpi_{R_\theta}\left( \theta_t + \alpha_t(f_1(\theta_t, O_t)+g_1(z_t,O_t)) \right) - \theta^* }^2\nonumber\\
	&=\ltwo{\mcpi_{R_\theta}\left( \theta_t + \alpha_t(f_1(\theta_t, O_t)+g_1(z_t,O_t)) \right) - \mcpi_{R_\theta} \theta^* }^2\nonumber\\
	&\leq \ltwo{ \theta_t -\theta^*  + \alpha_t(f_1(\theta_t, O_t)+g_1(z_t,O_t)) }^2\nonumber\\
	&= \ltwo{\theta_t -\theta^*}^2 +  2\alpha_t \langle f_1(\theta_t, O_t), \theta_t-\theta^*\rangle+2\alpha_t \langle g_1(z_t, O_t), \theta_t - \theta^* \rangle \nonumber \\
	&\quad + \alpha_t^2 \ltwo{  f_1(\theta_t, O_t) +  g_1(z_t,O_t)}^2 \nonumber\\
	&\leq \ltwo{\theta_t -\theta^*}^2 + 2\alpha_t \langle \bar{f}_1(\theta_t), \theta_t-\theta^*\rangle +  2\alpha_t \langle f_1(\theta_t, O_t)-\bar{f}_1(\theta_t), \theta_t-\theta^*\rangle \nonumber\\
	&\quad +2\alpha_t \langle g_1(z_t, O_t), \theta_t-\theta^* \rangle + 2\alpha_t^2 \ltwo{  f_1(\theta_t, O_t) }^2 + 2\alpha_t^2 \ltwo{g_1(z_t,O_t)}^2 \nonumber\\
	&\leq \ltwo{\theta_t -\theta^*}^2 + 2\alpha_t \langle (A^\top C^{-1}A)(\theta_t-\theta^*), \theta_t-\theta^*\rangle +  2\alpha_t \langle f_1(\theta_t, O_t)-\bar{f}_1(\theta_t), \theta_t-\theta^*\rangle \nonumber \\
	&\quad +2\alpha_t \langle g_1(z_t, O_t), \theta_t-\theta^* \rangle + 2\alpha_t^2 \ltwo{  f_1(\theta_t, O_t) }^2 + 2\alpha_t^2 \ltwo{g_1(z_t,O_t)}^2 \nonumber\\
	&\leq (1-\alpha_t|\lambda_\theta|)\ltwo{\theta_t -\theta^*}^2 + 2\alpha_t\zeta_{f_1}(\theta_t,O_t) + 2\alpha_t\langle B_t z_t, \theta_t-\theta^* \rangle  \label{eq: thetasigmais1} \\
	&\quad + 2\alpha_t^2 \ltwo{  f_1(\theta_t, O_t) }^2 + 2\alpha_t^2 \ltwo{g_1(z_t,O_t)}^2, \nonumber
	\end{flalign}
	where $2\lambda_{\max}(A^\top C^{-1}A)\leq\lambda_\theta<0$ and $\zeta_{f_1}(\theta_t,O_t)=\langle f_1(\theta_t, O_t) - \bar{f}_1(\theta_t), \theta_t-\theta^* \rangle$.
	
	First consider the case when $0<\nu<\sigma<1$. Telescoping the above inequality and taking the expectation over $\mathcal{F}_{t+1}$ on both sides yield that
	\begin{flalign}\label{eq: expthetainter}
	\mE\ltwo{\theta_{t+1}-\theta^*}^2&\leq \left[ \prod_{i=0}^{t}(1-\alpha_i|\lambda_\theta|) \right]\ltwo{\theta_0-\theta^*}^2 \nonumber \\
	&\quad + 2\sum_{i=0}^{t} \left[ \prod_{k=i+1}^{t}(1-\alpha_i|\lambda_\theta|) \right]\alpha_i\mE\zeta_{f_1}(\theta_i,O_i)\nonumber\\
	&\quad + 2\sum_{i=0}^{t} \left[ \prod_{k=i+1}^{t}(1-\alpha_i|\lambda_\theta|) \right]\alpha_i\mE\langle B_iz_i,\theta_i-\theta^*\rangle\nonumber\\
	&\quad + 2(K^2_{f_1} + K^2_{g_1})\sum_{i=0}^{t} \left[ \prod_{k=i+1}^{t}(1-\alpha_i|\lambda_\theta|) \right]\alpha_i^2.
	\end{flalign}
	Then following steps that are similar to \eqref{bias1z}-\eqref{variance1z}, we obtain
	\begin{flalign}
	\mE\ltwo{\theta_{t+1}-\theta^*}^2&\leq e^{-|\lambda_\theta|\sum_{i=0}^{t}\alpha_i}\ltwo{\theta_0-\theta^*}^2 \label{theta_neig}\\ 
	&\quad + 2\sum_{i=0}^{t} e^{-|\lambda_\theta|\sum_{k=i+1}^{t}\alpha_k} \alpha_i\mE[\zeta_{f_1}(\theta_i,O_i)] \label{bias1theta} \\
	&\quad + 2\sum_{i=0}^{t} e^{-|\lambda_\theta|\sum_{k=i+1}^{t}\alpha_k}\alpha_i\mE[\langle B_i z_i, \theta_i-\theta^* \rangle] \label{errorz}\\
	&\quad + 2(K_{f_1}^2+K_{g_1}^2)\sum_{i=0}^{t} e^{-|\lambda_\theta|\sum_{k=i+1}^{t} \alpha_k }\alpha_i^2. \label{variance1theta}
	\end{flalign}
	Similarly, the first term \eqref{theta_neig} captures how fast $\theta_t$ converges to the neighborhood of $\theta^*$, the second term \eqref{bias1theta} represents the accumulative bias induced by the biased gradient estimator $f_1(\theta_t,O_t)$, the third term \eqref{errorz} represents the accumulative error caused by imperfect tracking of $w_t$, and the last term \eqref{variance1theta} captures the accumulative variance. Combining Lemma \ref{lemma: boundaccumulatebiasf1}, Lemma \ref{lemma: mixorderbound3} and applying Lemma \ref{lemma:variance} to the above upper bound, we obtain
	\begin{flalign}\label{thm1_1: explicit}
	\mE\ltwo{\theta_{t+1}-\theta^*}^2&\leq e^{\frac{-|\lambda_\theta| c_\alpha}{1-\sigma}[(1+t)^{1-\sigma}-1]}\ltwo{\theta_0-\theta^*}^2 \nonumber\\ 
	&\quad + 2c_\alpha L_{f_1,\theta}(K_{f_1}+K_{g_1})\tau_\alpha \frac{e^{|\lambda_\theta| c_\alpha}}{|\lambda_\theta|}e^{\frac{-|\lambda_\theta| c_\alpha}{1-\sigma}[ (1+t)^\sigma - (1+\tau_\sigma)^\sigma ]}  
	+ 16R_\theta K_{f_1}\frac{e^{|\lambda_\theta| c_\alpha}}{|\lambda_\theta|} \frac{c_\alpha}{(1+t)^\sigma}  \nonumber\\
	&\quad + 2L_{f_1,\theta}(K_{f_1}+K_{g_1})\tau_\alpha \frac{2e^{|\lambda_\theta| c_\alpha/2}}{|\lambda_\theta|} (e^{\frac{-|\lambda_\theta| c_\alpha}{2(1-\sigma)}[(t+1)^{1-\sigma}-1]}D_4 \mathbbm{1}_{\{ \tau_\alpha+1< i_\alpha \}}+\alpha_{t-\tau_\alpha}) \nonumber\\
	&\quad + 2(K_{f_1}^2+K_{g_1}^2)\frac{2c_\alpha e^{|\lambda_\theta| c_\alpha/2}}{|\lambda_\theta|}(e^{\frac{-|\lambda_\theta| c_\alpha}{2(1-\sigma)}[(t+1)^{1-\sigma}-1]}D_4 +\alpha_t) \nonumber\\
	&\quad+\mathcal{O}(\frac{\log t}{t^\nu}+h(\sigma, \nu))^{1-\epsilon^\prime},
	\end{flalign}
where $\epsilon^\prime\in(0,0.5]$ can be any small constant, $D_4= e^{(|\lambda_\theta| c_\alpha/2)\sum_{k=0}^{i_{d_4}}\alpha_k} $, $i_{d_4} = (\frac{|\lambda_\theta| c_\alpha}{2\sigma})^{\frac{1}{1-\sigma}}$ and $\tau_\alpha = \min\{ i\in\mN|m\rho^i\leq \alpha_t \}$.
	
	If $\sigma=1$, choosing the stepsize $\alpha_t=\frac{1}{|\lambda_\theta|(1+t)}$, starting from \eqref{eq: thetasigmais1} and applying Lemma \ref{lemma: boundedf1} and Lemma \ref{lemma: boundedg1}, we have
	\begin{flalign*}
	\ltwo{\theta_{t+1}-\theta^*}^2 &\leq (1-\frac{1}{1+t})\ltwo{\theta_t -\theta^*}^2 + \frac{2}{|\lambda_\theta|(1+t)}\zeta_{f_1}(\theta_t,O_t) + \frac{2}{|\lambda_\theta|(1+t)}\langle B_t z_t, \theta_t-\theta^* \rangle  \\
	&\quad + \frac{2}{\lambda_\theta^2(1+t)^2} (K_{f_1}^2+K_{g_1}^2),
	\end{flalign*}
	which further implies that
	\begin{flalign}\label{eq: thetasigmais2}
	&(1+t)\ltwo{\theta_{t+1}-\theta^*}^2-t\ltwo{\theta_t -\theta^*}^2 \nonumber\\
	&\quad \leq \frac{2}{|\lambda_\theta|}\zeta_{f_1}(\theta_t,O_t) + \frac{2}{|\lambda_\theta|}\langle B_t z_t, \theta_t-\theta^* \rangle + \frac{2(K_{f_1}^2+K_{g_1}^2)}{\lambda_\theta^2} \frac{1}{1+t}.
	\end{flalign}
	Applying \eqref{eq: thetasigmais2} recursively and taking the expectation over $\mathcal{F}_{t+1}$ on both sides yields that
	\begin{flalign}\label{eq: thetasigmais3}
	&\mE\ltwo{\theta_{t+1}-\theta^*}^2 \nonumber\\
	&\qquad \leq \frac{2}{|\lambda_\theta|(1+t)} \sum_{i=0}^{t}\mE\zeta_{f_1}(\theta_i,O_i) + \frac{2}{|\lambda_\theta|(1+t)} \sum_{i=0}^{t} \mE\langle B_i z_i, \theta_i-\theta^* \rangle +  \frac{2(K_{f_1}^2+K_{g_1}^2)}{\lambda_\theta^2(1+t)} \sum_{i=0}^{t}\frac{1}{1+i}.
	\end{flalign}
	Then applying Lemma \ref{lemma: boundaccumulatebiasf1_II} and Lemma \ref{lemma: mixorderbound5}, we obtain
	\begin{flalign}\label{thm1_2: explicit}
	\mE\ltwo{\theta_{t+1}-\theta^*}^2 &\leq \frac{4L_{f_1,\theta}(K_{f_1}+K_{g_1})}{\lambda_\theta^2}\frac{\tau_\alpha^2}{1+t} + \frac{16R_\theta K_{f_1}}{\lambda_\theta^2(1+t)} + \frac{2L_{f_1,\theta}(K_{f_1}+K_{g_1})}{|\lambda_\theta|}\frac{\tau_\alpha \log(1+t)}{1+t}\nonumber\\
	&\quad + \frac{2(K_{f_1}^2+K_{g_1}^2)}{\lambda_\theta^2}\frac{1+\log(1+t)}{1+t} +\mathcal{O}(\frac{\log t}{t^\nu}+h(1, \nu))^{1-\epsilon^\prime}.
	\end{flalign}

\subsection{Technical Lemmas for Convergence Proof of Fast Time-scale Iteration}\label{proof_thm1_fast}
\begin{lemma}\label{lemma:variance}
	Let $p<0$, $0<q<1$, then for every integer $t\geq 0$,
	\begin{flalign*}
	\sum_{i=0}^{t} e^{p\sum_{k=i+1}^{t}(1+k)^{-q}}\frac{1}{(1+i)^{2q}}\leq \frac{2e^{|p|/2}}{|p|}\left[  D_p e^{p/2\sum_{k=0}^{t}(1+k)^{-q}} + \frac{1}{(1+t)^q}  \right],
	\end{flalign*}
	where $D_p=e^{|p|/2\sum_{k=0}^{i_p}(1+k)^{-q}}$, with $i_p$ denoting a constant larger than $(|p|/2q)^{1/(1-q)}$.
\end{lemma}
\begin{proof}
	For detailed proof of Lemma \ref{lemma:variance} please refer to Theorem 4.3 in \cite{dalal2018finite}.
\end{proof}
In order to bound the accumulated bias terms \eqref{bias1z} and \eqref{bias2z}, we prove the following lemmas.
\begin{lemma}\label{lemma: boundedf2}
	For any $\theta\in \mR^d$ such that $\ltwo{\theta}\leq R_\theta$, $\ltwo{f_2(\theta, O_i)}\leq K_{f_2}$ for any $i\geq 0$, where $K_{f_2}<\infty$ is a bounded positive constant indepedent of $\theta$ and $w$.
\end{lemma}
\begin{proof}
	By the definition of $f_2(\theta, O_i)$, and denoting $\lambda_{cm}=\min|\lambda(C)|$, we obtain
	\begin{flalign*}
		||f_2(\theta, O_i)|| &=  \ltwo{(A_i-C_iC^{-1}A)\theta + (b_i-C_iC^{-1}b) }\\
		&\leq  \ltwo{(A_i-C_iC^{-1}A)\theta} + \ltwo{(b_i-C_iC^{-1}b) }\\
		&\leq (\ltwo{A_i}+\ltwo{C_i}\ltwo{C^{-1}}\ltwo{A})\ltwo{\theta}+\ltwo{b_i}+\ltwo{C_i}\ltwo{C^{-1}}\ltwo{b}\\
		&\leq \left[(1+\gamma)\rho_{\max}+\frac{1}{\lambda_{cm}}(1+\gamma)\rho_{\max}\right]R_\theta+\rho_{\max}r_{\max}+\frac{1}{\lambda_{cm}}\rho_{\max}r_{\max}\\
		&\triangleq K_{f_2}.
	\end{flalign*}
\end{proof}
\begin{lemma}\label{lemma: boundedbiasf2}
	For all $\theta\in \mR^d$ such that $\ltwo{\theta}\leq R_\theta$ and all $z\in \mR^d$ such that $\ltwo{z}\leq R_w$, for all $i\geq 0$,  
	(a) $\ltwo{\zeta_{f_2}(\theta,z,O_i)}\leq 4R_wK_{f_2}$; 
	(b) $|\zeta_{f_2}(\theta_1,z_1,O_i)-\zeta_{f_2}(\theta_2,z_2,O_i)|\leq L_{f_2,\theta}\ltwo{\theta_1-\theta_2} + L_{f_2,z}\ltwo{z_1-z_2}$.
\end{lemma}
\begin{proof}
	For (a), by the defination we have $\ltwo{\zeta_{f_2}(\theta,z,O_i)} = \ltwo { \langle f_2(\theta_t, O_t), z_t\rangle} \leq \ltwo{f_2(\theta,O_i)} \ltwo{z}\leq 2R_wK_{f_2}$.
	
	For (b), we derive the bound as follows
	\begin{flalign*}
	|\zeta_{f_2}(\theta_1,z_1,O_i)-\zeta_{f_2}(\theta_2,z_2,O_i)|&=| \langle f_2(\theta_1, O_i), z_1\rangle - \langle f_2(\theta_2, O_i), z_2\rangle |\\
	&\leq \ltwo{z_1}\ltwo{f_2(\theta_1,O_i)-f_2(\theta_2,O_i)} + \ltwo{f_2(\theta_2,O_i)}\ltwo{z_1-z_2}\\
	&\leq 2R_w\ltwo{(A_t-C_tC^{-1}A)(\theta_1-\theta_2)}+2K_{f_2}\ltwo{z_1-z_2}\\
	&\leq 2R_w \left[(1+\gamma)\rho_{\max}+\frac{1}{\lambda_{cm}}(1+\gamma)\rho_{\max}\right] \ltwo{\theta_1-\theta_2} + 2K_{f_2}\ltwo{z_1-z_2}\\
	&\leq L_{f_2,\theta}\ltwo{\theta_1-\theta_2} + L_{f_2,z}\ltwo{z_1-z_2}.
	\end{flalign*}
\end{proof}
\begin{lemma}\label{lemma: biasf2final}
	Let $K_{r_3}=[\max\{1, c_\alpha/c_\beta\}L_{f_2,\theta}(K_{f_1}+K_{g_1}) + L_{f_2,z}K_{r_2} ]$. Then for $i\leq \tau_\beta$, $\mE[\zeta_{f_2}(\theta_i,z_i,O_i)]\leq c_\beta K_{r_3}\tau_\beta$; and for $i>\tau_\beta$, $\mE[\zeta_{f_2}(\theta_i,z_i,O_i)]\leq 8R_wK_{f_2}\beta_i + K_{r_3}\tau_\beta\beta_{i-\tau_\beta}$.
\end{lemma}
\begin{proof}
	Note that for any $i\geq0$, 
	\begin{flalign*}
		\ltwo{\theta_{i+1}-\theta_{i}}&=\ltwo{\mcpi_{R_\theta}\left(\theta_i + \alpha_i(f_1(\theta_i, O_i)+g_1(z_i, O_i))\right)-\mcpi_{R_\theta}\theta_i}\\
		&\leq \ltwo{\theta_i + \alpha_i(f_1(\theta_i, O_i)+g_1(z_i, O_i))-\theta_i}\\
		&\leq \alpha_i \ltwo{f_1(\theta_i, O_i)+g_1(z_i, O_i)}\\
		&\leq \alpha_i(K_{f_1}+K_{g_1}).
	\end{flalign*}
	Furthermore,
	\begin{flalign}\label{eq: boundztdifference}
		&\ltwo{z_{i+1}-z_{i}}\nonumber\\
		&=\ltwo{\mcpi_{R_w}\left( z_i + \beta_i(f_2(\theta_i, O_i)+g_2(z_i,O_i))-C^{-1}(b+A\theta_i) \right) + C^{-1}(b+A\theta_{i+1})-z_i}\nonumber\\
		&=\ltwo{\mcpi_{R_w}\left( z_i + \beta_i(f_2(\theta_i, O_i)+g_2(z_i,O_i))-C^{-1}(b+A\theta_i) \right) + C^{-1}(b+A\theta_{i})-z_i + C^{-1}A(\theta_{i+1}-\theta_i)}\nonumber\\
		&=\ltwo{\mcpi_{R_w}\left( z_i + \beta_i(f_2(\theta_i, O_i)+g_2(z_i,O_i))-C^{-1}(b+A\theta_i) \right) - \mcpi_{R_w} \left[ z_i - C^{-1}(b+A\theta_{i}) \right] + C^{-1}A(\theta_{i+1}-\theta_i)}\nonumber\\
		&\leq \ltwo{\mcpi_{R_w}\left( z_i + \beta_i(f_2(\theta_i, O_i)+g_2(z_i,O_i))-C^{-1}(b+A\theta_i) \right) - \mcpi_{R_w} \left[ z_i - C^{-1}(b+A\theta_{i}) \right]} \nonumber\\
		&\quad + \ltwo{C^{-1}A(\theta_{i+1}-\theta_i)}\nonumber\\
		&\leq \ltwo{\left( z_i + \beta_i(f_2(\theta_i, O_i)+g_2(z_i,O_i))-C^{-1}(b+A\theta_i) \right) -  \left[ z_i - C^{-1}(b+A\theta_{i}) \right]} + \ltwo{C^{-1}A(\theta_{i+1}-\theta_i)}\nonumber\\
		&= \beta_i\ltwo{f_2(\theta_i, O_i)+g_2(z_i,O_i)} + \ltwo{C^{-1}A(\theta_{i+1}-\theta_i)}\nonumber\\
		&\leq \beta_i(K_{f_2}+K_{g_2})+\alpha_i\ltwo{C^{-1}}\ltwo{A}(K_{f_1}+K_{g_1})\nonumber\\
		&\leq \beta_i(K_{f_2}+K_{g_2}+\max\{1, \frac{c_\alpha}{c_\beta}\}\frac{(1+\gamma)\rho_{\max}}{\lambda_{cm}}(K_{f_1}+K_{g_1}))\nonumber\\
		&=\beta_i K_{r_2}
	\end{flalign}
	where $K_{r_2}\triangleq K_{f_2}+K_{g_2}+\max\{1, \frac{c_\alpha}{c_\beta}\}\frac{(1+\gamma)\rho_{\max}}{\lambda_{cm}}(K_{f_1}+K_{g_1})$. Applying the Lipschitz continuous property in Lemma \ref{lemma: boundedbiasf2}, it follows that
	\begin{flalign*}
		| \zeta_{f_2}(\theta_i,z_i,O_i) - \zeta_{f_2}(\theta_{i-\tau},z_{i-\tau},O_i) |&\leq L_{f_2,\theta}\ltwo{\theta_i-\theta_{i-\tau}} + L_{f_2,z}\ltwo{z_i-z_{i-\tau}}\\
		&\leq L_{f_2,\theta}(K_{f_1}+K_{g_1})\sum_{k=i-\tau}^{i-1}\alpha_k + L_{f_2,z}K_{r_2}\sum_{k=i-\tau}^{i-1}\beta_k.
	\end{flalign*}
	The next step is to provide an upper bound for $\mE[\zeta_{f_2}(\theta_{i-\tau},z_{i-\tau},O_i)]$. 
	We further define an independent $(\theta_{i-\tau}^\prime, z_{i-\tau}^\prime)$ and $O_i^\prime=(s_i^\prime, a_i^\prime, r_i^\prime, s_{i+1}^\prime)$ that has the same marginal distribution as $(\theta_{i-\tau},z_{i-\tau})$ and $O_i$. It is clear that $\mE[\zeta_{f_2}(\theta_{i-\tau}^\prime,z_{i-\tau}^\prime,O_i^\prime)]=0$.
	Note that the following Markov chain holds
	\begin{flalign*}
	(\theta_{i-\tau}, z_{i-\tau})\rightarrow s_{i-\tau}\rightarrow s_i\rightarrow O_i.
	\end{flalign*}
	Since $\ltwo{\zeta_{f_2}(\theta,z,O_i)}\leq 4R_wK_{f_2}$ for all $\theta, z\in \mR^d$, by Lemma \ref{lemma: boundedbiasf2}, applying Lemma 10 in \cite{bhandari2018finite} yields
	\begin{flalign*}
		\mE[\zeta_{f_2}(\theta_{i-\tau},z_{i-\tau},O_i)]&\leq | \mE[\zeta_{f_2}(\theta_{i-\tau},z_{i-\tau},O_i)] - \mE[\zeta_{f_2}(\theta_{i-\tau}^\prime,z_{i-\tau}^\prime,O_i^\prime)] |\leq 8R_wK_{f_2}m\rho^\tau.
	\end{flalign*}
	Recall that $\tau_\beta=\min\{i\geq 0: m\rho^i\leq \beta_t \}$. For $i\leq \tau_\beta$, it follows that
	\begin{flalign*}
		\mE[\zeta_{f_2}(\theta_i,z_i,O_i)]&\leq \mE[\zeta_{f_2}(\theta_0,z_0,O_i)] + L_{f_2,\theta}(K_{f_1}+K_{g_1})\sum_{k=0}^{i-1}\alpha_k + L_{f_2,z}K_{r_2}\sum_{k=0}^{i-1}\beta_k\\
		&\leq L_{f_2,\theta}(K_{f_1}+K_{g_1})i\alpha_0 + L_{f_2,z}K_{r_2}i\beta_0\\
		&\leq c_\beta[\max\{1, \frac{c_\alpha}{c_\beta}\}L_{f_2,\theta}(K_{f_1}+K_{g_1}) + L_{f_2,z}K_{r_2} ]\tau_\beta\\
		&\leq c_\beta K_{r_3}\tau_\beta.
	\end{flalign*}
	For $i> \tau_\beta$, it follows that
	\begin{flalign*}
		\mE[\zeta_{f_2}(\theta_i,z_i,O_i)]&\leq \mE[\zeta_{f_2}(\theta_{i-\tau_\beta},z_{i-\tau_\beta},O_i)] + L_{f_2,\theta}(K_{f_1}+K_{g_1})\sum_{k=i-\tau_\beta}^{i-1}\alpha_k + L_{f_2,z}K_{r_2}\sum_{k=i-\tau_\beta}^{i-1}\beta_k\\
		&\leq 8R_wK_{f_2}m\rho^{\tau_\beta}+L_{f_2,\theta}(K_{f_1}+K_{g_1})\tau_\beta\alpha_{i-\tau_\beta} + L_{f_2,z}K_{r_2}\tau_\beta\beta_{i-\tau_\beta}\\
		&\leq 8R_wK_{f_2}\beta_t + [\max\{1, \frac{c_\alpha}{c_\beta}\}L_{f_2,\theta}(K_{f_1}+K_{g_1}) + L_{f_2,z}K_{r_2} ]\tau_\beta\beta_{i-\tau_\beta}\\
		&= 8R_wK_{f_2}\beta_t + K_{r_3}\tau_\beta\beta_{i-\tau_\beta}.
	\end{flalign*}
\end{proof}
\begin{lemma}\label{lemma: boundaccumulatebiasf2}
	Fix $0<\nu<1$, and let $\beta_t=c_\beta/(1+t)^\nu$. Then
	\begin{flalign*}
		&\sum_{i=0}^{t} e^{\lambda_w\sum_{k=i+1}^{t}\beta_k} \beta_i\mE[\zeta_{f_2}(\theta_i,z_i,O_i)]&\\
		&\leq  c_\beta K_{r_3} \tau_\beta \frac{e^{-\lambda_w c_\beta}}{-\lambda_w} e^{\frac{\lambda_w c_\beta}{1-\nu}[ (1+t)^\nu - (1+\tau_\beta)^\nu ]} + 8R_wK_{f_2}\frac{e^{-\lambda_w c_\beta}}{-\lambda_w} \frac{c_\beta}{(1+t)^\nu}& \\
		&\quad + 2K_{r_3}\tau_\beta \frac{e^{-\lambda_w c_\beta/2}}{-\lambda_w} (e^{\frac{-\lambda_w c_\beta}{2(1-\nu)}[(t+1)^{1-\nu}-1]}D_1 \mathbbm{1}_{\{ \tau_\beta+1< i_{d_1} \}}+\beta_{t-\tau_\beta}),&
	\end{flalign*}
	where $D_1=c_\beta \max_{i\in[0, i_{d_1}]}\{ e^{-(\lambda_w/2)\sum_{k=0}^{i}\beta_k}  \}$ and $i_{d_1} = ( \frac{-2\nu}{\lambda_w c_\beta} )^{\frac{1}{1-\nu}}$.
\end{lemma}
\begin{proof}
	Applying Lemma \ref{lemma: biasf2final}, it follows that
	\begin{flalign}\label{eq: bdacmbiasf2_1}
		&\sum_{i=0}^{t} e^{\lambda_w\sum_{k=i+1}^{t}\beta_k} \beta_i\mE[\zeta_{f_2}(\theta_i,z_i,O_i)] &\nonumber\\
		&\leq c_\beta K_{r_3}\tau_\beta\sum_{i=0}^{\tau_\beta} e^{\lambda_w\sum_{k=i+1}^{t}\beta_k} \beta_i + 8R_wK_{f_2}\beta_t \sum_{i=\tau_\beta+1}^{t} e^{\lambda_w\sum_{k=i+1}^{t}\beta_k} \beta_i &\nonumber\\
		& \quad + K_{r_3}\tau_\beta \sum_{i=\tau_\beta+1}^{t} e^{\lambda_w\sum_{k=i+1}^{t}\beta_k} \beta_{i-\tau_\beta}\beta_i.&
	\end{flalign}
	For the first term in \eqref{eq: bdacmbiasf2_1}, we have
	\begin{flalign}\label{eq: bdacmbiasf2_7}
		\sum_{i=0}^{\tau_\beta} e^{\lambda_w\sum_{k=i+1}^{t}\beta_k} \beta_i &\leq \max_{i\geq 0}\{ e^{-\lambda_w\beta_i} \} \sum_{i=0}^{\tau_\beta} e^{\lambda_w\sum_{k=i}^{t}\beta_k} \beta_i\nonumber\\
		&=  e^{-\lambda_w c_\beta} \sum_{i=0}^{\tau_\beta} e^{\lambda_w(T_{t+1}-T_i)} \beta_i\nonumber\\
		&\leq e^{-\lambda_w c_\beta} \int_{0}^{T_{\tau_\beta+1}}e^{\lambda_w (T_{t+1}-s)}ds\nonumber\\
		&\leq \frac{e^{-\lambda_w c_\beta}}{-\lambda_w}e^{\lambda_w(T_{t+1}-T_{\tau_\beta+1})}\nonumber\\
		&\leq \frac{e^{-\lambda_w c_\beta}}{-\lambda_w}e^{\lambda_wc_\beta\sum_{k=\tau_\beta}^{t}1/(1+k)^{-\nu}}\nonumber\\
		&= \frac{e^{-\lambda_w c_\beta}}{-\lambda_w}e^{\frac{\lambda_w c_\beta}{1-\nu}[ (1+t)^\nu - (1+\tau_\beta)^\nu ]},
	\end{flalign}
	where $T_n=\sum_{k=0}^{n-1}\beta_k$. For the second term in \eqref{eq: bdacmbiasf2_1}, we have
	\begin{flalign}\label{eq. bdacmbiasf2_6}
		\beta_t \sum_{i=\tau_\beta+1}^{t} e^{\lambda_w\sum_{k=i+1}^{t}\beta_k} \beta_i &\leq \max_{i\geq 0}\{ e^{-\lambda_w\beta_i} \} \beta_t \sum_{i=\tau_\beta+1}^{t} e^{\lambda_w\sum_{k=i+1}^{t}\beta_k} \beta_i\nonumber\\
		&\leq e^{-\lambda_w c_\beta} \beta_t \sum_{i=\tau_\beta+1}^{t} e^{\lambda_w(T_{t+1}-T_i)} \beta_i\nonumber\\
		&\leq e^{-\lambda_w c_\beta} \beta_t \int_{T_{\tau_\beta+1}}^{T_{t+1}}e^{\lambda_w (T_{t+1}-s)}ds\nonumber\\	
		&= \frac{e^{-\lambda_w c_\beta}}{-\lambda_w}\beta_t\left(1-e^{\lambda_w(T_{t+1}-T_{\tau_\beta+1})} \right)\nonumber\\
		&\leq \frac{e^{-\lambda_w c_\beta}}{-\lambda_w} \frac{c_\beta}{(1+t)^\nu}.
	\end{flalign}
	For the third term in \eqref{eq: bdacmbiasf2_1}, we have
	\begin{flalign}\label{eq: bdacmbiasf2_2}
		\sum_{i=\tau_\beta+1}^{t} e^{\lambda_w\sum_{k=i+1}^{t}\beta_k} \beta_{i-\tau_\beta}\beta_i&\leq \max_{i\in[\tau_\beta+1, t]}\{ e^{(\lambda_w/2)\sum_{k=i+1}^{t}\beta_k} \beta_{i-\tau_\beta} \} \sum_{i=\tau_\beta+1}^{t} e^{(\lambda_w/2)\sum_{k=i+1}^{t}\beta_k} \beta_i\nonumber\\
		&\leq \max_{i\in[\tau_\beta+1, t]}\{ e^{(\lambda_w/2)\sum_{k=i+1}^{t}\beta_k} \beta_{i-\tau_\beta} \}\frac{2e^{-\lambda_w c_\beta/2}}{-\lambda_w}.
	\end{flalign}
	To bound \eqref{eq: bdacmbiasf2_2}, we define $ y_i = e^{(\lambda_w/2)\sum_{k=i+1}^{t}\beta_k} \beta_{i-\tau_\beta} $, and then we have
	\begin{flalign*}
		\frac{y_{i+1}}{y_i}=e^{-(\lambda_w/2) \beta_{i+1}}\left(  1-\frac{1}{2+i-\tau_\beta}  \right)^\nu.
	\end{flalign*}
	If $i\geq i_{d_1}$ and $\tau_\beta+1>i_{d_1}$, then $\frac{y_{i+1}}{y_i}\geq 1$ for all $i\in[\tau_\beta+1, t] $. Thus
	\begin{flalign}\label{eq: bdacmbiasf2_3}
	\max_{i\in[\tau_\beta+1, t]}\{ e^{(\lambda_w/2)\sum_{k=i+1}^{t}\beta_k} \beta_{i-\tau_\beta} \}=\beta_{t-\tau_\beta}.
	\end{flalign}
	If $\tau_\beta+1< i_{d_1}$, then
	\begin{flalign}\label{eq: bdacmbiasf2_4}
		&\max_{i\in[\tau_\beta+1, t]}\{ e^{(\lambda_w/2)\sum_{k=i+1}^{t}\beta_k} \beta_{i-\tau_\beta} \}&\nonumber\\
		&\leq \max_{i\in[\tau_\beta+1, i_{d_1}]}\{ e^{(\lambda_w/2)\sum_{k=i+1}^{t}\beta_k} \beta_{i-\tau_\beta} \} + \max_{i\in[i_{d_1}+1, t]}\{ e^{(\lambda_w/2)\sum_{k=i+1}^{t}\beta_k} \beta_{i-\tau_\beta} \}&\nonumber\\
		&\leq e^{(\lambda_w/2)\sum_{k=0}^{t}\beta_k}  \max_{i\in[\tau_\beta+1, i_{d_1}]}\{ e^{-(\lambda_w/2)\sum_{k=0}^{i}\beta_k} \beta_{i-\tau_\beta} \} + \beta_{t-\tau_\beta}&\nonumber\\
		&\leq e^{(\lambda_w/2)\sum_{k=0}^{t}\beta_k}  \max_{i\in[0, i_{d_1}]}\{ e^{-(\lambda_w/2)\sum_{k=0}^{i}\beta_k} \beta_0 \} + \beta_{t-\tau_\beta}&\nonumber\\
		&\leq e^{\frac{\lambda_w c_\beta}{2(1-\nu)}[(t+1)^{1-\nu}-1]}D_1+\beta_{t-\tau_\beta}.&
	\end{flalign}
	Combining \eqref{eq: bdacmbiasf2_3} and \eqref{eq: bdacmbiasf2_4} and substituting into \eqref{eq: bdacmbiasf2_2}, we have
	\begin{flalign}\label{eq: bdacmbiasf2_5}
		\sum_{i=\tau_\beta+1}^{t} e^{\lambda_w\sum_{k=i+1}^{t}\beta_k} \beta_{i-\tau_\beta}\beta_i \leq \frac{2e^{-\lambda_w c_\beta/2}}{-\lambda_w} (e^{\frac{\lambda_w c_\beta}{2(1-\nu)}[(t+1)^{1-\nu}-1]}D_1 \mathbbm{1}_{\{ \tau_\beta+1< i_\beta \}}+\beta_{t-\tau_\beta}).
	\end{flalign}
	Finally, \eqref{eq: bdacmbiasf2_5}, \eqref{eq. bdacmbiasf2_6}, and \eqref{eq: bdacmbiasf2_4} imply that
	\begin{flalign}\label{eq. bdacmbiasf2_7}
		&\sum_{i=0}^{t} e^{\lambda_w\sum_{k=i+1}^{t}\beta_k} \beta_i\mE[\zeta_{f_2}(\theta_i,z_i,O_i)]\nonumber\\
		&\leq  [c_\alpha L_{f_2,\theta}(K_{f_1}+K_{g_1}) + c_\beta L_{f_2,z}K_{r_2}] \tau_\beta \frac{e^{-\lambda_w c_\beta}}{-\lambda_w} e^{\frac{\lambda_w c_\beta}{1-\nu}[ (1+t)^\nu - (1+\tau_\beta)^\nu ]}\\
		& + 4R_wK_{f_2}\frac{e^{-\lambda_w c_\beta}}{-\lambda_w} \frac{c_\beta}{(1+t)^\nu} + 2K_{r_3}\tau_\beta \frac{e^{-\lambda_w c_\beta/2}}{-\lambda_w} (e^{\frac{-\lambda_w c_\beta}{2(1-\nu)}[(t+1)^{1-\nu}-1]}D_1 \mathbbm{1}_{\{ \tau_\beta+1< i_{d_1} \}}+\beta_{t-\tau_\beta}).
	\end{flalign}
\end{proof}
\begin{lemma}\label{lemma: boundedg2}
	For any $z\in \mR^d$ such that $\ltwo{z}\leq R_w$, $\ltwo{g_2(z,O_i)}\leq K_{g_2}$ for any $i\geq 0$.
\end{lemma}
\begin{proof}
	By the definition of $g_2(z,O_t)$, we obtain
	\begin{flalign*}
	\ltwo{g_2(z,O_i)} &= \ltwo{C_iz_i}\leq \ltwo{C_i}\ltwo{z_i}\leq 2R_w\leq K_{g_2}.
	\end{flalign*}
\end{proof}
\begin{lemma}\label{lemma: boundedbiasg2}
	For all $z\in \mR^d$ such that $\ltwo{z}\leq R_w$, we have for all $i\geq 0$, (1) $\ltwo{\zeta_{g_2}(z,O_i)}\leq 4R_wK_{g_2}$; (2) $|\zeta_{g_2}(z_1,O_i)-\zeta_{g_2}(z_2,O_i)|\leq L_{g_2,z}\ltwo{z_1-z_2}$.
\end{lemma}
\begin{proof}
	For (1), by the defination of $\zeta_{g_2}(z,O_i)$, we have $\ltwo{\zeta_{g_2}(z_i,O_i)} = \ltwo{\langle g_2(z_t, O_t) - \bar{g}_2(z_t), z_t \rangle} \leq (\ltwo{g_2(\theta_i,O_i)} + \ltwo{\bar{g}_2(\theta_i)})\ltwo{z_i}\leq 4R_wK_{g_2}$.
	For (2), we derive the bound as follows.
	\begin{flalign*}
	&|\zeta_{g_2}(z_1,O_i)-\zeta_{g_2}(z_2,O_i)|\\
	&=| \langle g_2(z_1, O_i) - \bar{g}_2(z_1), z_1 \rangle + \langle g_2(z_2, O_i) - \bar{g}_2(z_2), z_2 \rangle|\\
	&\leq \ltwo{z_1}\ltwo{g_2(z_1, O_i) - \bar{g}_2(z_1) - g_2(z_2, O_i) + \bar{g}_2(z_2)} + \ltwo{g_2(z_2, O_i) - \bar{g}_2(z_2)}\ltwo{z_1-z_2}\\
	&= \ltwo{z_1} \ltwo{(C_i-C)(z_1-z_2)}+\ltwo{g_2(z_2, O_i) - \bar{g}_2(z_2)}\ltwo{z_1-z_2}\\
	&\leq 2R_w(\ltwo{C_i}+\ltwo{C})\ltwo{z_1-z_2}+2K_{g_2}\ltwo{z_1-z_2}\\
	&\leq 4R_w\ltwo{z_1-z_2}+2K_{g_2}\ltwo{z_1-z_2}\\
	&\leq L_{g_2,z}\ltwo{z_1-z_2}.
	\end{flalign*}
\end{proof}
\begin{lemma}\label{lemma: biasg2final}
	For $i\leq \tau_\beta$, $\mE[\zeta_{g_2}(z_i,O_i)]\leq c_\beta L_{g_2,z}K_{r_2}\tau_\beta$; and for $i>\tau_\beta$, $\mE[\zeta_{g_2}(z_i,O_i)]\leq 8R_wK_{g_2}\beta_t + L_{g_2,z}K_{r_2}\tau_\beta\beta_{i-\tau_\beta}$.
\end{lemma}
\begin{proof}
	Applying the Lipschitz continuous property of $\zeta_{g_2}(z,O_i)$ and the inequality \eqref{eq: boundztdifference} in Lemma \ref{lemma: biasf2final}, it follows that
	\begin{flalign*}
		|\zeta_{g_2}(z_i,O_i)-\zeta_{g_2}(z_{i-\tau},O_i)|\leq L_{g_2,z}\ltwo{z_i-z_{i-\tau}}\leq L_{g_2,z}K_{r_2}\sum_{k=i-\tau}^{i-1}\beta_k.
	\end{flalign*}
	Then we need to provide an upper bound for $\mE[\zeta_{g_2}(z_{i-\tau},O_i)]$. 
	We further define an independent $z_{i-\tau}^\prime$ and $O_i^\prime=(s_i^\prime, a_i^\prime, r_i^\prime, s_{i+1}^\prime)$ which have the same marginal distribution as $z_{i-\tau}$ and $O_i$. Using Lemma \ref{lemma: boundedbiasg2} and following the steps similar to those in Lemma \ref{lemma: biasf2final}, we obtain
	\begin{flalign*}
	\mE[\zeta_{g_2}(z_{i-\tau},O_i)]&\leq | \mE[\zeta_{g_2}(z_{i-\tau},O_i)] - \mE[\zeta_{f_2}(z_{i-\tau}^\prime,O_i^\prime)] |\leq 8R_wK_{g_2}m\rho^\tau.
	\end{flalign*}
	For $i\leq \tau_\beta$, it follows that
	\begin{flalign*}
	\mE[\zeta_{g_2}(z_i,O_i)]&\leq \mE[\zeta_{g_2}(z_0,O_i)] + L_{g_2,z}K_{r_2}\sum_{k=0}^{i-1}\beta_k \leq L_{g_2,z}K_{r_2}i\beta_0 \leq c_\beta L_{g_2,z}K_{r_2}\tau_\beta.
	\end{flalign*}
	For $i> \tau_\beta$, it follows that
	\begin{flalign*}
	\mE[\zeta_{g_2}(z_i,O_i)]&\leq \mE[\zeta_{g_2}(z_{i-\tau_\beta},O_i)] + L_{g_2,z}K_{r_2}\sum_{k=i-\tau_\beta}^{i-1}\beta_k\\
	&\leq 8R_wK_{g_2}m\rho^{\tau_\beta} + L_{g_2,z}K_{r_2}\tau_\beta\beta_{i-\tau_\beta} \\
	&\leq 8R_wK_{g_2}\beta_t + L_{g_2,z}K_{r_2}\tau_\beta\beta_{i-\tau_\beta}.
	\end{flalign*}
\end{proof}
\begin{lemma}\label{lemma: boundaccumulatebiasg2}
	Fix $0<\nu<1$, and let $\beta_t=c_\beta/(1+t)^\nu$. Then
	\begin{flalign*}
		&\sum_{i=0}^{t} e^{\lambda_w\sum_{k=i+1}^{t}\beta_k}\beta_i\mE[\zeta_{g_2}(z_i,O_i)]&\\
		&\leq  c_\beta L_{g_2,z}K_{r_2} \tau_\beta \frac{e^{-\lambda_w c_\beta}}{-\lambda_w} e^{\frac{\lambda_w c_\beta}{1-\nu}[ (1+t)^\nu - (1+\tau_\beta)^\nu ]} + 8R_wK_{g_2}\frac{e^{-\lambda_w c_\beta}}{-\lambda_w} \frac{c_\beta}{(1+t)^\nu}& \nonumber \\
		&\quad + 2L_{g_2,z}K_{r_2} \tau_\beta \frac{e^{-\lambda_w c_\beta/2}}{-\lambda_w} (e^{\frac{\lambda_w c_\beta}{2(1-\nu)}[(t+1)^{1-\nu}-1]}D_1 \mathbbm{1}_{\{ \tau_\beta+1< i_{d_1} \}}+\beta_{t-\tau_\beta}).&
	\end{flalign*}
	where $D_1=c_\beta \max_{i\in[0, i_{d_1}]}\{ e^{-(\lambda_w/2)\sum_{k=0}^{i}\beta_k}  \}$ and $i_{d_1} = ( \frac{-2\nu}{\lambda_w c_\beta} )^{\frac{1}{1-\nu}}$.
\end{lemma}
\begin{proof}
	Applying Lemma \ref{lemma: biasg2final}, it follows that
	\begin{flalign}
		&\sum_{i=0}^{t} e^{\lambda_w\sum_{k=i+1}^{t}\beta_k}\beta_i\mE[\zeta_{g_2}(z_i,O_i)] \nonumber\\
		&\leq c_\beta L_{g_2,z}K_{r_2}\tau_\beta \sum_{i=0}^{\tau_\beta} e^{\lambda_w\sum_{k=i+1}^{t}\beta_k}\beta_i + 8R_wK_{g_2}\beta_t \sum_{i=\tau_\beta+1}^{t} e^{\lambda_w\sum_{k=i+1}^{t}\beta_k}\beta_i \nonumber\\
		&\quad + L_{g_2,z}K_{r_2} \tau_\beta \sum_{i=\tau_\beta+1}^{t} e^{\lambda_w\sum_{k=i+1}^{t}\beta_k}\beta_{i-\tau_\beta}\beta_i\nonumber.
	\end{flalign}
	Following steps similar to those in \eqref{eq: bdacmbiasf2_1}-\eqref{eq. bdacmbiasf2_7}, we have the desired result.
\end{proof}
\begin{lemma}\label{lemma: 1thbdinter}
	For given $0<\nu<\sigma<1$, let $\beta_t=c_\beta/(1+t)^\nu$ and $\alpha_t=c_\alpha/(1+t)^\sigma$. Then 
	\begin{flalign*}
	&\sum_{i=0}^{t} e^{\lambda_w\sum_{k=i+1}^{t}\beta_k}\mE\langle C^{-1}A(\theta_{i+1}-\theta_i),z_i\rangle\\
	&\leq \frac{2c_\alpha(1+\gamma)\rho_{\max}}{c_\beta\lambda_{cm}}R_w(K_{f_1}+K_{g_1}) \frac{2e^{-\lambda_w c_\beta/2}}{-\lambda_w} \Big(e^{\frac{\lambda_w c_\beta}{2(1-\nu)}[(1+t)^{1-\nu}-1]}D_2 + \frac{1}{(1+t)^{\sigma-\nu}} \Big),
	\end{flalign*}
	where $D_2=\max_{i\in[0, i_{d_2}]}\{ e^{-(\lambda_w/2)\sum_{k=0}^{i}\beta_k} \frac{1}{(1+i)^{\sigma-\nu}}  \}$ and $i_{d_2} = ( \frac{-2(\sigma-\nu)}{\lambda_w c_\beta} )^{\frac{1}{1-\nu}}$.
\end{lemma}
\begin{proof}
	Applying Lemmas \ref{lemma: boundedg1} and \ref{lemma: boundedf1}, it follows that
	\begin{flalign}\label{eq: 1thbdinter_eq1}
	&\sum_{i=0}^{t} e^{\lambda_w\sum_{k=i+1}^{t}\beta_k}\mE\langle C^{-1}A(\theta_{i+1}-\theta_i),z_i\rangle\nonumber\\
	&\leq \sum_{i=0}^{t} e^{\lambda_w\sum_{k=i+1}^{t}\beta_k}\mE\left[ \ltwo{C^{-1}}\ltwo{A}\ltwo{\theta_{i+1} - \theta_i}\ltwo{z_i}\right]\nonumber\\
	&\leq 2\ltwo{C^{-1}}\ltwo{A} R_w (K_{f_1}+K_{g_1}) \sum_{i=0}^{t} e^{\lambda_w\sum_{k=i+1}^{t}\beta_k}\alpha_i\nonumber\\
	&\leq \frac{2(1+\gamma)\rho_{\max}}{\lambda_{cm}}R_w(K_{f_1}+K_{g_1})\sum_{i=0}^{t} e^{\lambda_w\sum_{k=i+1}^{t}\beta_k}\beta_i\frac{\alpha_i}{\beta_i}\nonumber\\
	&\leq \frac{2c_\alpha(1+\gamma)\rho_{\max}}{c_\beta\lambda_{cm}}R_w(K_{f_1}+K_{g_1}) \max_{i\in[0, t]}\{ e^{(\lambda_w/2)\sum_{k=i+1}^{t}\beta_k} \frac{1}{(1+i)^{\sigma-\nu}} \} \sum_{i=0}^{t} e^{(\lambda_w/2)\sum_{k=i+1}^{t}\beta_k}\beta_i\nonumber\\
	&\leq \frac{2c_\alpha(1+\gamma)\rho_{\max}}{c_\beta\lambda_{cm}}R_w(K_{f_1}+K_{g_1}) \frac{2e^{-\lambda_w c_\beta/2}}{-\lambda_w} \Big(e^{\frac{\lambda_w c_\beta}{2(1-\nu)}[(1+t)^{1-\nu}-1]}D_2 + \frac{1}{(1+t)^{\sigma-\nu}} \Big).
	\end{flalign}
	Based on \eqref{eq: 1thbdinter_eq1}, we follow similar steps in Theroem 4.3 \cite{dalal2018finite} and obtain the following upper bound
	\begin{flalign}\label{eq: 1thbdinter_eq2}
	\sum_{i=0}^{t} e^{(\lambda_w/2)\sum_{k=i+1}^{t}\beta_k}\beta_i\leq \frac{2e^{-\lambda_wC_\beta/2}}{-\lambda_w},
	\end{flalign}
	and
	\begin{flalign}\label{eq: 1thbdinter_eq3}
	\max_{i\in[0, t]}\{ e^{(\lambda_w/2)\sum_{k=i+1}^{t}\beta_k} \frac{1}{(1+i)^{\sigma-\nu}} \}\leq e^{\frac{\lambda_w c_\beta}{2(1-\nu)}[(1+t)^{1-\nu}-1]}D_2 + \frac{1}{(1+t)^{\sigma-\nu}},
	\end{flalign}
	where $D_2=\max_{i\in[0, i_{d_2}]}\{ e^{-(\lambda_w/2)\sum_{k=0}^{i}\beta_k} \frac{1}{(1+i)^{\sigma-\nu}}  \}$ and $i_{d_2} = ( \frac{-2(\sigma-\nu)}{\lambda_w c_\beta} )^{\frac{1}{1-\nu}}$.
\end{proof}
\begin{lemma}\label{lemma: mixorderbound2}
	Suppose \eqref{eq: firstboundz} holds. If $\sigma> \frac{3}{2}\nu$, we have 
	\begin{flalign*}
	\sum_{i=0}^{t} e^{\lambda_w\sum_{k=i+1}^{t}\beta_k}\mE\langle C^{-1}A(\theta_{i+1}-\theta_i),z_i\rangle = \mathcal{O}\left(\frac{1}{t^\nu}\right),\;
	\end{flalign*}
	amd if $\nu<\sigma\leq\frac{3}{2}\nu$, we have
	\begin{flalign*}
	\sum_{i=0}^{t} e^{\lambda_w\sum_{k=i+1}^{t}\beta_k}\mE\langle C^{-1}A(\theta_{i+1}-\theta_i),z_i\rangle = \mathcal{O}\left(\frac{1}{t^{2(\sigma-\nu)-\epsilon}}\right),
	\end{flalign*}
	where $\epsilon$ is any constant in $(0,\sigma-\nu]$.
\end{lemma}
\begin{proof}
	If $\sigma\geq2\nu$, \eqref{lemma: 1thbdinter} implies that
	\begin{flalign*}
	    \sum_{i=0}^{t} e^{\lambda_w\sum_{k=i+1}^{t}\beta_k}\mE\langle C^{-1}A(\theta_{i+1}-\theta_i),z_i\rangle=\mathcal{O}\left(\frac{1}{(1+t)^\nu}\right).
	\end{flalign*}
	If $\sigma\leq2\nu$, it follows that $\mE\ltwo{z_{t}}^2=\mathcal{O}(\frac{1}{t^{\sigma-\nu}})$. Hence there exists a constant $0<C<\infty$ and $T>0$ such that
	\begin{flalign}
	    &\mE\ltwo{z_t}^2 \leq 4R_w^2 \qquad \text{for all}\,  0\leq t\leq T,\label{eq: mo2eq1}   \\
	    &\mE\ltwo{z_t}^2 \leq \frac{C}{(1+t)^{(\sigma-\nu)}}\quad \text{for all}\, t>T. \label{eq: mo2eq2}
	\end{flalign}
	Then, substituting \eqref{eq: mo2eq1} and \eqref{eq: mo2eq2} into \eqref{interac1}, we have
	\begin{flalign}
		&\sum_{i=0}^{t} e^{\lambda_w\sum_{k=i+1}^{t}\beta_k}\mE\langle C^{-1}A(\theta_{i+1}-\theta_i),z_i\rangle\\
		&\leq \ltwo{C^{-1}}\ltwo{A}(K_{f_1}+K_{g_1}) \sum_{i=0}^{t} e^{\lambda_w\sum_{k=i+1}^{t}\beta_k}\alpha_i \sqrt{\mE \ltwo{z_i}^2}\\
		&\leq \frac{(1+\gamma)\rho_{\max}}{\lambda_{cm}}(K_{f_1}+K_{g_1}) \Big(\sum_{i=0}^{T} e^{\lambda_w\sum_{k=i+1}^{t}\beta_k}\alpha_i \sqrt{\mE\ltwo{z_i}^2} + \sum_{i=T+1}^{t} e^{\lambda_w\sum_{k=i+1}^{t}\beta_k}\alpha_i \sqrt{\mE\ltwo{z_i}^2}\Big)\\
		&\leq \frac{c_\alpha(1+\gamma)\rho_{\max}}{c_\beta\lambda_{cm}}(K_{f_1}+K_{g_1}) \Big(2R_w\sum_{i=0}^{T} e^{\lambda_w\sum_{k=i+1}^{t}\beta_k}\beta_i\frac{1}{(1+i)^{(\sigma-\nu)}} \nonumber\\
		&\qquad\qquad\qquad\qquad\qquad\qquad\qquad+ C\sum_{i=T+1}^{t} e^{\lambda_w\sum_{k=i+1}^{t}\beta_k}\beta_i \frac{1}{(1+i)^{1.5(\sigma-\nu)}}\Big). \label{eq: mo2eq3}
	\end{flalign}
	Here, we follow similar steps in \eqref{eq: bdacmbiasf2_7} and \eqref{eq: bdacmbiasf2_2}-\eqref{eq: bdacmbiasf2_5} to get
	\begin{flalign*}
	    \sum_{i=0}^{T} e^{\lambda_w\sum_{k=i+1}^{t}\beta_k}\beta_i\frac{1}{(1+i)^{(\sigma-\nu)}} \leq \sum_{i=0}^{T} e^{\lambda_w\sum_{k=i+1}^{t}\beta_k}\beta_i \leq \frac{e^{-\lambda_w c_\beta}}{-\lambda_w}e^{\frac{\lambda_w c_\beta}{1-\nu}[ (1+t)^\nu - (1+T)^\nu ]},
	\end{flalign*}
	and
	\begin{flalign*}
	    \sum_{i=T+1}^{t} e^{\lambda_w\sum_{k=i+1}^{t}\beta_k}\beta_i \frac{1}{(1+i)^{1.5(\sigma-\nu)}} \leq \frac{2e^{-\lambda_w C_\beta/2}}{-\lambda_w}\Big( e^{\frac{\lambda_w c_\beta}{2(1-\nu)}[(1+t)^{1-\nu}-1]}D + \frac{1}{(1+t)^{1.5(\sigma-\nu)}} \Big),
	\end{flalign*}
	where $D=\max_{i\in[0, i_d]}\{ e^{-(\lambda_w/2)\sum_{k=0}^{i}\beta_k} \frac{1}{(1+i)^{1.5(\sigma-\nu)}}  \}$ and $i_d = ( \frac{-3(\sigma-\nu)}{\lambda_w c_\beta} )^{\frac{1}{1-\nu}}$. 
	
	It follows that 
	\begin{flalign*}
		\sum_{i=0}^{t} e^{\lambda_w\sum_{k=i+1}^{t}\beta_k}\alpha_i\mE\langle C^{-1}A(\theta_{i+1}-\theta_i),z_i\rangle=\mathcal{O}\left(\frac{1}{t^{1.5(\sigma-\nu)}}\right).
	\end{flalign*}
	If $\frac{3}{2}\nu<\sigma\leq2\nu$, we have $\mE\ltwo{z_{t}}^2=\mathcal{O}\left(\frac{1}{t^{1.5(\sigma-\nu)}}\right)$. Then, by following the similar steps in \eqref{eq: mo2eq1}-\eqref{eq: mo2eq3}, we have
	\begin{flalign*}
	    \sum_{i=0}^{t} e^{\lambda_w\sum_{k=i+1}^{t}\beta_k}\alpha_i\mE\langle C^{-1}A(\theta_{i+1}-\theta_i),z_i\rangle=\mathcal{O}\left(\frac{1}{t^{1.75(\sigma-\nu)}}\right),
	\end{flalign*}
	and $\mE\ltwo{z_{t}}^2=\mathcal{O}\left(\frac{1}{t^{1.75(\sigma-\nu)}}\right)$. Then we repeat the steps \eqref{eq: mo2eq1}-\eqref{eq: mo2eq3} for a total number $N = \lceil-\log_2(2-\frac{\nu}{\sigma-\nu})\rceil$ of times, we have
	\begin{flalign*}
	    \sum_{i=0}^{t} e^{\lambda_w\sum_{k=i+1}^{t}\beta_k}\alpha_i\mE\langle C^{-1}A(\theta_{i+1}-\theta_i),z_i\rangle=\mathcal{O}\left(\frac{1}{t^{(2-2^{-N})(\sigma-\nu)}}\right)=\mathcal{O}\left(\frac{1}{(1+t)^\nu}\right).
	\end{flalign*}
	Since $(2-2^{-N})(\sigma-\nu)>\nu$, we have $\mE\ltwo{z_{t}}^2=\mathcal{O}(\frac{\log t}{t^\nu})+\mathcal{O}(\frac{1}{t^\nu})$.
	
	If $\nu<\sigma\leq\frac{3}{2}\nu$, then we repeat steps \eqref{eq: mo2eq1}-\eqref{eq: mo2eq3} for a total number $N = \lceil\log_2(\frac{\sigma-\nu}{\epsilon})\rceil$ of times, we have
	\begin{flalign*}
	\sum_{i=0}^{t} e^{\lambda_w\sum_{k=i+1}^{t}\beta_k}\alpha_i\mE\langle C^{-1}A(\theta_{i+1}-\theta_i),z_i\rangle=\mathcal{O}\left(\frac{1}{(1+t)^{2(\sigma-\nu)-\epsilon}}\right).
	\end{flalign*}
\end{proof}
\subsection{Technical Lemmas for Convergence Proof of Slow Time-scale Iteration}\label{proof_thm1_slow}
In this subsection, we obtain the following properties for the slow time-scale. 
\begin{lemma}\label{lemma: boundedf1}
	For any $\theta\in \mR^d$ such that $\ltwo{\theta}\leq R_\theta$, $\ltwo{f_1(\theta, O_i)}\leq K_{f_1}$ for any $i\geq 0$, where $K_{f_1}<\infty$ is a bounded constant indepedent of $\theta$ and $w$.
\end{lemma}
\begin{proof}
	By the definition of $f_1(\theta, O_i)$, and denoting $\lambda_{cm}=\min|\lambda(C)|$, we obtain
	\begin{flalign*}
	||f_1(\theta, O_i)|| &=  \ltwo{(A_i-B_iC^{-1}A)\theta + (b_i-B_iC^{-1}b) }\\
	&\leq  \ltwo{(A_i-B_iC^{-1}A)\theta} + \ltwo{(b_i-B_iC^{-1}b) }\\
	&\leq (\ltwo{A_i}+\ltwo{B_i}\ltwo{C^{-1}}\ltwo{A})\ltwo{\theta}+\ltwo{b_i}+\ltwo{B_i}\ltwo{C^{-1}}\ltwo{b}\\
	&\leq \left[(1+\gamma)\rho_{\max}+\frac{1}{\lambda_{cm}}\gamma
	(1+\gamma)\rho_{\max}^2\right]+\rho_{\max}r_{\max}+\frac{1}{\lambda_{cm}}\gamma
	\rho_{\max}^2 r_{\max}\\
	&\leq K_{f_1}.
	\end{flalign*}
\end{proof}
\begin{lemma}\label{lemma: boundedg1}
	For any $z\in \mR^d$ such that $\ltwo{z}\leq 2R_w$, $\ltwo{g_1(z, O_i)}\leq K_{g_1}$ for any $i\geq 0$.
\end{lemma}
\begin{proof}
	By the definition of $g_1(z, O_i)$, we obtain $\ltwo{g_1(z_t,O_t)}=\ltwo{B_t z_t} \leq \ltwo{B_t}\ltwo{z_t}\leq 2\gamma\rho_{\max}R_w$.
\end{proof}
\begin{lemma}\label{lemma: boundedbiasf1}
	For all $\theta\in \mR^d$ such that $\ltwo{\theta}\leq R_\theta$, we have for all $i\geq 0$, (a) $\ltwo{\zeta_{f_1}(\theta,O_i)}\leq 4R_\theta K_{f_1}$; (b) $|\zeta_{f_2}(\theta_1,O_i)-\zeta_{f_2}(\theta_2,O_i)|\leq L_{f_1,\theta}\ltwo{\theta_1-\theta_2}$.
\end{lemma}
\begin{proof}
	For (a), following steps similar in \eqref{lemma: boundedf1}, we have $\ltwo{\bar{f}_1(\theta)}\leq K_{f_1}$. Then by the defination we have 
	\begin{flalign*}
		\ltwo{\zeta_{f_1}(\theta,O_i)}\leq (\ltwo{f_1(\theta,O_i)} + \ltwo{\bar{f}_1(\theta)} ) (\ltwo{\theta} + \ltwo{\theta^*})\leq 4R_\theta K_{f_1}.
	\end{flalign*}
	For (b), we derive the bound as follows
	\begin{flalign*}
	&|\zeta_{f_1}(\theta_1,O_i)-\zeta_{f_1}(\theta_2,O_i)|\\
	&=| \langle f_1(\theta_1-\bar{f}(\theta_1), O_i), \theta_1-\theta^*\rangle - \langle f_1(\theta_2, O_i)-\bar{f}_1(\theta_2), \theta_2-\theta^*\rangle |\\
	&\leq \ltwo{\theta_1-\theta^*}\ltwo{f_1(\theta_1,O_i)-\bar{f}_1(\theta_1)-f_1(\theta_2,O_i)+\bar{f}_1(\theta_2)} + \ltwo{f_1(\theta_2,O_i)-\bar{f}_1(\theta_2)}\ltwo{\theta_1-\theta_2}\\
	&\leq \ltwo{\theta_1-\theta^*}(\ltwo{f_1(\theta_1,O_i)-f_1(\theta_2,O_i)} + \ltwo{\bar{f}_1(\theta_1)-\bar{f}_1(\theta_2)}) + \ltwo{f_1(\theta_2,O_i)-\bar{f}_1(\theta_2)}\ltwo{\theta_1-\theta_2}\\
	&\leq 2R_\theta(\ltwo{(A_t-B_tC^{-1}A)(\theta_1-\theta_2)} + \ltwo{(A-BC^{-1}A)(\theta_1-\theta_2)} )+2K_{f_1}\ltwo{\theta_1-\theta_2}\\
	&\leq 4R_\theta(1+\gamma)\rho_{\max}(1+\frac{1}{\lambda_{cm}}\gamma\rho_{\max})\ltwo{\theta_1-\theta_2} + 2K_{f_1}\ltwo{z_1-z_2}\\
	&\leq L_{f_1,\theta}\ltwo{\theta_1-\theta_2}.
	\end{flalign*}
\end{proof}
\begin{lemma}\label{lemma: biasf1final}
	For $i\leq \tau_\alpha$, $\mE[\zeta_{f_1}(\theta_i,O_i)]\leq c_\alpha L_{f_1,\theta}(K_{f_1}+K_{g_1})\tau_\alpha$; and for $i>\tau_\alpha$, $\mE[\zeta_{f_1}(\theta_i,O_i)]\leq 8R_\theta K_{f_1}\alpha_t + L_{f_1,\theta}(K_{f_1}+K_{g_1})\tau_\alpha\alpha_{i-\tau_\alpha}$.
\end{lemma}
\begin{proof}
	Applying the Lipschitz continuous property of $\zeta_{g_2}(z,O_i)$ and the inequality \eqref{eq: boundztdifference} in Lemma \ref{lemma: biasf2final}, it follows that
	\begin{flalign*}
	|\zeta_{f_1}(\theta_i,O_i)-\zeta_{f_1}(\theta_{i-\tau},O_i)|\leq L_{f_1,\theta}(K_{f_1}+K_{g_1})\ltwo{\theta_i-\theta_{i-\tau}}\leq L_{f_1,\theta}(K_{f_1}+K_{g_1})\sum_{k=i-\tau}^{i-1}\alpha_k.
	\end{flalign*}
	Then, we need to provide an upper bound for $\mE[\zeta_{f_1}(\theta_{i-\tau},O_i)]$. 
	We further define an independent $\theta_{i-\tau}^\prime$ and $O_i^\prime=(s_i^\prime, a_i^\prime, r_i^\prime, s_{i+1}^\prime)$, which have the same marginal distributions as $\theta_{i-\tau}$ and $O_i$. Using Lemma \ref{lemma: boundedbiasf1} and following the steps similar to those in Lemma \ref{lemma: biasf2final}, we have
	\begin{flalign*}
	\mE[\zeta_{f_1}(\theta_{i-\tau},O_i)]&\leq | \mE[\zeta_{f_1}(\theta_{i-\tau},O_i)] - \mE[\zeta_{f_1}(\theta_{i-\tau}^\prime,O_i^\prime)] |\leq 8R_\theta K_{f_1}m\rho^\tau.
	\end{flalign*}
	If $i\leq \tau_\alpha$, it follows that
	\begin{flalign*}
	\mE[\zeta_{f_1}(\theta_i,O_i)]&\leq \mE[\zeta_{f_1}(\theta_0,O_i)] + L_{f_1,\theta}(K_{f_1}+K_{g_1})\sum_{k=0}^{i-1}\alpha_k \leq L_{f_1,\theta}(K_{f_1}+K_{g_1})i\alpha_0 \\
	&\leq c_\alpha L_{f_1,\theta}(K_{f_1}+K_{g_1})\tau_\alpha.
	\end{flalign*}
	If $i> \tau_\alpha$, it follows that
	\begin{flalign*}
	\mE[\zeta_{f_1}(\theta_i,O_i)]&\leq \mE[\zeta_{f_1}(\theta_{i-\tau_\alpha},O_i)] + L_{f_1,\theta}(K_{f_1}+K_{g_1})\sum_{k=i-\tau_\alpha}^{i-1}\alpha_k\\
	&\leq 8R_\theta K_{f_1}m\rho^{\tau_\alpha} + L_{f_1,\theta}(K_{f_1}+K_{g_1})\tau_\alpha\alpha_{i-\tau_\alpha} \\
	&\leq 8R_\theta K_{f_1}\alpha_t + L_{f_1,\theta}(K_{f_1}+K_{g_1})\tau_\alpha\alpha_{i-\tau_\alpha}.
	\end{flalign*}
\end{proof}
\begin{lemma}\label{lemma: boundaccumulatebiasf1}
	Fix $0<\sigma<1$, and let $\sigma_t=c_\alpha/(1+t)^\sigma$. Then
	\begin{flalign*}
	&\sum_{i=0}^{t} e^{\lambda_\theta\sum_{k=i+1}^{t}\alpha_k} \alpha_i\mE[\zeta_{f_1}(\theta_i,O_i)]&\nonumber\\
	&\leq  c_\alpha L_{f_1,\theta}(K_{f_1}+K_{g_1})\tau_\alpha \frac{e^{-\lambda_\theta c_\alpha}}{-\lambda_\theta}e^{\frac{\lambda_\theta c_\alpha}{1-\sigma}[ (1+t)^\sigma - (1+\tau_\sigma)^\sigma ]}  + 8R_\theta K_{f_1}\frac{e^{-\lambda_\theta c_\alpha}}{-\lambda_\theta} \frac{c_\alpha}{(1+t)^\sigma}& \nonumber \\
	&\quad + L_{f_1,\theta}(K_{f_1}+K_{g_1})\tau_\alpha \frac{2e^{-\lambda_\theta c_\alpha/2}}{-\lambda_\theta} (e^{\frac{\lambda_\theta c_\alpha}{2(1-\sigma)}[(t+1)^{1-\sigma}-1]}D_4 \mathbbm{1}_{\{ \tau_\alpha+1< i_\alpha \}}+\alpha_{t-\tau_\alpha}),&
\end{flalign*}
	where $T_n=\sum_{k=0}^{n-1}\alpha_k$, $D_4=c_\alpha \max_{i\in[0, i_{d_4}]}\{ e^{-(\lambda_\theta/2)\sum_{k=0}^{i}\alpha_k}  \}$ and $i_{d_4} = ( \frac{-2\sigma}{\lambda_\theta c_\alpha} )^{\frac{1}{1-\sigma}}$.
\end{lemma}
\begin{proof}
	Applying Lemma \ref{lemma: biasf1final}, it follows that
	\begin{flalign}\label{eq: bdacmbiasf1_1}
	&\sum_{i=0}^{t} e^{\lambda_\theta\sum_{k=i+1}^{t}\alpha_k} \alpha_i\mE[\zeta_{f_1}(\theta_i,O_i)] \nonumber\\
	&\leq c_\alpha L_{f_1,\theta}(K_{f_1}+K_{g_1})\tau_\alpha \sum_{i=0}^{\tau_\alpha} e^{\lambda_\theta\sum_{k=i+1}^{t}\alpha_k} \alpha_i + 8R_\theta K_{f_1}\alpha_t \sum_{i=\tau_\alpha+1}^{t} e^{\lambda_\theta\sum_{k=i+1}^{t}\alpha_k} \alpha_i \nonumber\\
	& \quad + L_{f_1,\theta}(K_{f_1}+K_{g_1})\tau_\alpha \sum_{i=\tau_\alpha+1}^{t} e^{\lambda_\theta\sum_{k=i+1}^{t}\alpha_k} \alpha_{i-\tau_\beta}\alpha_i.
	\end{flalign}
	Following steps similar to those in Lemma \ref{lemma: boundaccumulatebiasf2}, we obtain:
	\begin{flalign}
	&\sum_{i=0}^{\tau_\alpha} e^{\lambda_w\sum_{k=i+1}^{t}\alpha_k} \alpha_i \leq \frac{e^{-\lambda_\theta c_\alpha}}{-\lambda_\theta}e^{\frac{\lambda_\theta c_\alpha}{1-\sigma}[ (1+t)^\sigma - (1+\tau_\sigma)^\sigma ]} \label{eq. bdacmbiasf1_1}\\
	&\alpha_t \sum_{i=\tau_\alpha+1}^{t} e^{\lambda_\theta \sum_{k=i+1}^{t}\alpha_k} \alpha_i \leq \frac{e^{-\lambda_\theta c_\alpha}}{-\lambda_\theta} \frac{c_\alpha}{(1+t)^\sigma}\label{eq. bdacmbiasf1_2}\\
	&\sum_{i=\tau_\alpha+1}^{t} e^{\lambda_\theta\sum_{k=i+1}^{t}\alpha_k} \alpha_{i-\tau_\alpha}\alpha_i \leq \frac{2e^{-\lambda_\theta c_\alpha/2}}{-\lambda_\theta} (e^{\frac{\lambda_\theta c_\alpha}{2(1-\sigma)}[(t+1)^{1-\sigma}-1]}D_4 \mathbbm{1}_{\{ \tau_\alpha+1< i_\alpha \}}+\alpha_{t-\tau_\alpha}) \label{eq. bdacmbiasf1_3},
	\end{flalign}
	which yields the desired result.
\end{proof}
\begin{lemma}\label{lemma: mixorderbound1}
	For $0<\sigma<1$, $c_\alpha>0$, $\alpha_t=\frac{c_\alpha}{(1+t)^\sigma}$, and $0<x<1$, $0<y<1$. If $\mE\ltwo{z_t}^2=\mathcal{O}(\frac{\log t}{t^\nu} + \frac{1}{t^\nu} )^x$ and $\mE\ltwo{\theta_t-\theta^*}^2=\mathcal{O}(\frac{\log t}{t^\nu} + \frac{1}{t^\nu} )^y$ for $a,b>0$, then we have
	\begin{flalign*}
	\sum_{i=0}^{t} e^{\lambda_\theta\sum_{k=i+1}^{t}\alpha_k}\alpha_i\mE[\langle B_i z_i, \theta_i-\theta^* \rangle] = \mathcal{O}\Big(\frac{\log t}{t^\nu} + \frac{1}{t^\nu} \Big)^{0.5(x+y)}.
	\end{flalign*}
	If $\mE\ltwo{z_t}^2=\mathcal{O}(\frac{\log t}{t^\nu}+\frac{1}{t^{2(\sigma-\nu)-\epsilon}})^x$, $\mE\ltwo{\theta_t-\theta^*}^2=\mathcal{O}(\frac{\log t}{t^\nu}+\frac{1}{t^{2(\sigma-\nu)-\epsilon}})^y$, then we have
	\begin{flalign*}
	\sum_{i=0}^{t} e^{\lambda_\theta\sum_{k=i+1}^{t}\alpha_k}\alpha_i\mE[\langle B_i z_i, \theta_i-\theta^* \rangle] = \mathcal{O}\Big(\frac{\log t}{t^\nu}+\frac{1}{t^{2(\sigma-\nu)-\epsilon}}\Big)^{0.5(x+y)}.
	\end{flalign*}
\end{lemma}
\begin{proof}
	Consider the first case. Without loss of generality, we assume that there exist constant $0<C_1,C_2<\infty$, $T>0$ such that
	\begin{flalign*}
	&\mE\ltwo{z_t}^2 \leq 4R_w^2 \qquad \text{for all}\,  0\leq t\leq T,\\
	&\mE\ltwo{z_t}^2 \leq C_1^2\Big(\frac{\log t+1}{t^\nu}\Big)^x\quad \text{for all}\, t>T,
	\end{flalign*}
	and
	\begin{flalign*}
	&\mE\ltwo{\theta_t-\theta^*}^2 \leq R_\theta^2 \qquad \text{for all}\,  0\leq t\leq T,\\
	&\mE\ltwo{\theta_t-\theta^*}^2 \leq C_2^2 \Big(\frac{\log t+1}{t^\nu}\Big)^y\quad \text{for all}\, t>T.
	\end{flalign*}
	Then, it follows that
	\begin{flalign*}
	&\sum_{i=0}^{t} e^{\lambda_\theta\sum_{k=i+1}^{t}\alpha_k}\alpha_i\mE[\langle B_i z_i, \theta_i-\theta^* \rangle]\\
	&\leq \sum_{i=0}^{t} e^{\lambda_\theta\sum_{k=i+1}^{t}\alpha_k}\alpha_i\sqrt{\mE[\ltwo{B_i z_i}^2]}\sqrt{\mE[\ltwo{\theta_i-\theta^*}^2]}\\
	&\leq \sum_{i=0}^{t} e^{\lambda_\theta\sum_{k=i+1}^{t}\alpha_k}\alpha_i \ltwo{B_i} \sqrt{\mE\ltwo{z_i}^2}\sqrt{\mE[\ltwo{\theta_i-\theta^*}^2]}\\
	&\leq \gamma\rho_{\max}\Big(\sum_{i=0}^{T} e^{\lambda_\theta\sum_{k=i+1}^{t}\alpha_k}\alpha_i \sqrt{\mE\ltwo{z_i}^2}\sqrt{\mE[\ltwo{\theta_i-\theta^*}^2]} \\
	&\qquad\qquad\quad + \sum_{i=T+1}^{t} e^{\lambda_\theta\sum_{k=i+1}^{t}\alpha_k}\alpha_i \sqrt{\mE\ltwo{z_i}^2}\sqrt{\mE[\ltwo{\theta_i-\theta^*}^2]}\Big)\\
	&\leq \gamma\rho_{\max} \Big(2R_w R_\theta\sum_{i=0}^{T} e^{\lambda_\theta\sum_{k=i+1}^{t}\alpha_k}\alpha_i + C_1C_2\sum_{i=T+1}^{t} e^{\lambda_\theta\sum_{k=i+1}^{t}\alpha_k}\alpha_i \Big(\frac{\log i + 1}{i^\nu}\Big)^{0.5(x+y)} \Big)\\
	&\leq 2\gamma\rho_{\max}R_w R_\theta \Big(2R_w R_\theta \frac{e^{-\lambda_\theta c_\alpha}}{-\lambda_\theta}e^{\frac{\lambda_\theta c_\alpha}{1-\sigma}[ (1+t)^\sigma - (1+T)^\sigma]} \\ &\qquad\qquad\qquad\qquad\quad+C_1C_2\frac{2e^{-\lambda_\theta C_\alpha/2}}{-\lambda_\theta}\Big( e^{\frac{\lambda_\theta c_\alpha}{2(1-\sigma)}[(1+t)^{1-\sigma}-1]}D + \Big(\frac{\log t+1}{t^\nu}\Big)^{0.5(x+y)} \Big) \Big)
	\end{flalign*}
	Here we follow similar steps in \eqref{eq: bdacmbiasf2_7} and \eqref{eq: bdacmbiasf2_2}-\eqref{eq: bdacmbiasf2_5} to obtain
	\begin{flalign*}
	\sum_{i=0}^{T} e^{\lambda_\theta\sum_{k=i+1}^{t}\alpha_k}\alpha_i \leq \frac{e^{-\lambda_\theta c_\alpha}}{-\lambda_\theta}e^{\frac{\lambda_\theta c_\alpha}{1-\sigma}[ (1+t)^\sigma - (1+T)^\sigma ]},
	\end{flalign*}
	and
	\begin{flalign*}
	\sum_{i=T+1}^{t} e^{\lambda_\theta\sum_{k=i+1}^{t}\alpha_k}\alpha_i \Big(\frac{\log i}{i^\nu}\Big)^{0.5(x+y)} \leq \frac{2e^{-\lambda_\theta C_\alpha/2}}{-\lambda_\theta}\Big( e^{\frac{\lambda_\theta c_\alpha}{2(1-\sigma)}[(1+t)^{1-\sigma}-1]}D + \Big(\frac{\log t+1}{t^\nu}\Big)^{0.5(x+y)} \Big),
	\end{flalign*}
	where $0<D<\infty$ is a constant depend on $x$ and $y$. 
	
	The proof for the second case follows similarly.
\end{proof}

\begin{lemma}\label{lemma: mixorderbound3}
	For $0<\frac{3}{2}\nu<\sigma<1$, if $\mE\ltwo{z_t}^2=\mathcal{O}(\frac{\log t}{t^\nu}) + \mathcal{O}(\frac{1}{t^\nu})$, then
	\begin{flalign*}
	\sum_{i=0}^{t} e^{\lambda_\theta\sum_{k=i+1}^{t}\alpha_k}\alpha_i\mE[\langle B_i z_i, \theta_i-\theta^* \rangle] = \mathcal{O}\Big(\frac{\log t}{t^\nu} + \frac{1}{t^\nu} \Big)^{1-\epsilon^\prime},
	\end{flalign*}
	and for $0<\nu<\sigma\leq\frac{3}{2}\nu<1$, if $\mE\ltwo{z_t}^2=\mathcal{O}(\frac{\log t}{t^\nu})+\mathcal{O}(\frac{1}{t^{2(\sigma-\nu)-\epsilon}})$, then
	\begin{flalign*}
	\sum_{i=0}^{t} e^{\lambda_\theta\sum_{k=i+1}^{t}\alpha_k}\alpha_i\mE[\langle B_i z_i, \theta_i-\theta^* \rangle] =  \mathcal{O}\Big( \frac{\log t}{t^\nu} + \frac{1}{t^{2(\sigma-\nu)-\epsilon}} \Big)^{1-\epsilon^\prime},
	\end{flalign*}
	where $\epsilon^\prime$ can be any constant in $(0, 0.5]$.
\end{lemma}
\begin{proof}
	Consider the first case. First, $\mE\ltwo{\theta_t-\theta^*}^2\leq 4R_\theta^2=\mathcal{O}(1)$, applying Lemma \ref{lemma: mixorderbound1} we immediately have
	\begin{flalign}\label{eq: motheta1}
	    \sum_{i=0}^{t} e^{\lambda_\theta\sum_{k=i+1}^{t}\alpha_k}\alpha_i\mE[\langle B_i z_i, \theta_i-\theta^* \rangle] = \mathcal{O}\Big(\frac{\log t}{t^\nu} + \frac{1}{t^\nu} \Big)^{0.5}.
	\end{flalign}
	Then it follows that $\mE\ltwo{\theta_{t+1}-\theta^*}^2=\mathcal{O}(\frac{\log t}{t^\nu})^{0.5}$. Then again applying Lemmas \ref{lemma: mixorderbound1} and \eqref{eq: motheta1}, we obtain
	\begin{flalign}\label{eq: motheta2}
		\sum_{i=0}^{t} e^{\lambda_\theta\sum_{k=i+1}^{t}\alpha_k}\alpha_i\mE[\langle B_i z_i, \theta_i-\theta^* \rangle] = \mathcal{O}\Big(\frac{\log t}{t^\nu} + \frac{1}{t^\nu} \Big)^{0.75}.
	\end{flalign}
	Hence, following the steps in \eqref{eq: motheta2} for a total number $N = \lceil \log_2(\frac{1}{1-\epsilon^\prime})\rceil$ of times, we have
	\begin{flalign*}
	    \sum_{i=0}^{t} e^{\lambda_\theta\sum_{k=i+1}^{t}\alpha_k}\alpha_i\mE[\langle B_i z_i, \theta_i-\theta^* \rangle] &=\mathcal{O}\Big(\frac{\log t}{t^\nu} + \frac{1}{t^\nu} \Big)^{1-\frac{1}{2^N}} \\
	    &= \mathcal{O}\Big(\frac{\log t}{t^\nu} + \frac{1}{t^\nu} \Big)^{1-\epsilon^\prime}.
	\end{flalign*}
	The proof for the second case follows similarly.
\end{proof}
\begin{lemma}\label{lemma: boundaccumulatebiasf1_II}
	Let $\alpha_t=\frac{1}{-\lambda_\theta(1+t)}$. Then
	\begin{flalign*}
		\sum_{i=0}^{t}\mE\zeta_{f_1}(\theta_i,O_i) \leq \frac{2L_{f_1,\theta}(K_{f_1}+K_{g_1})}{-\lambda_\theta}\tau_\alpha^2 + \frac{8R_\theta K_{f_1}}{-\lambda_\theta} + L_{f_1,\theta}(K_{f_1}+K_{g_1})\tau_\alpha \ln(1+t).
	\end{flalign*}
\end{lemma}
\begin{proof}
	Applying Lemma \eqref{lemma: biasf1final}, it follows that
	\begin{flalign*}
		\sum_{i=0}^{t}\mE\zeta_{f_1}(\theta_i,O_i) &= \sum_{i=0}^{\tau_\alpha}\mE\zeta_{f_1}(\theta_i,O_i) + \sum_{i=\tau_\alpha+1}^{t}\mE\zeta_{f_1}(\theta_i,O_i)\\
		& \leq \frac{L_{f_1,\theta}(K_{f_1}+K_{g_1})}{-\lambda_\theta}\tau_\alpha(1+\tau_\alpha) + \frac{8R_\theta K_{f_1}(t-\tau_\alpha)}{-\lambda_\theta(1+t)} \\
		&\qquad+ L_{f_1,\theta}(K_{f_1}+K_{g_1})\tau_\alpha \sum_{i=\tau_\alpha+1}^{t}\alpha_{i-\tau_\alpha}\\
		& \leq \frac{2L_{f_1,\theta}(K_{f_1}+K_{g_1})}{-\lambda_\theta}\tau_\alpha^2 + \frac{8R_\theta K_{f_1}}{-\lambda_\theta} + L_{f_1,\theta}(K_{f_1}+K_{g_1})\tau_\alpha \sum_{i=1}^{t-\tau_\alpha}\frac{1}{1+i}\\
		&\leq \frac{2L_{f_1,\theta}(K_{f_1}+K_{g_1})}{-\lambda_\theta}\tau_\alpha^2 + \frac{8R_\theta K_{f_1}}{-\lambda_\theta} + L_{f_1,\theta}(K_{f_1}+K_{g_1})\tau_\alpha \ln(1+t)
	\end{flalign*}
\end{proof}
\begin{lemma}\label{lemma: mixorderbound4}
	Suppose $0<x<1$, $0<y\leq1$. If $\mE\ltwo{z_t}^2=\mathcal{O}(\frac{\log t}{t^\nu}+\frac{1}{t^{\nu}})^x$, $\mE\ltwo{\theta_t-\theta^*}^2=\mathcal{O}(\frac{\log t}{t^\nu}+\frac{1}{t^{\nu}})^y$, then we have
	\begin{flalign*}
	\frac{1}{1+t}\sum_{i=0}^{t} \mE[\langle B_i z_i, \theta_i-\theta^* \rangle] = \mathcal{O}\Big(\frac{\log t}{t^\nu}+\frac{1}{t^{\nu}}\Big)^{0.5(x+y)}.
	\end{flalign*}
	If $\mE\ltwo{z_t}^2=\mathcal{O}(\frac{\log t}{t^\nu}+\frac{1}{t^{2(\sigma-\nu)-\epsilon}})^x$ and $\mE\ltwo{\theta_t-\theta^*}^2=\mathcal{O}(\frac{\log t}{t^\nu}+\frac{1}{t^{2(\sigma-\nu)-\epsilon}})^y$, then we have
	\begin{flalign*}
	\frac{1}{1+t}\sum_{i=0}^{t} \mE[\langle B_i z_i, \theta_i-\theta^* \rangle] = \mathcal{O}\Big(\frac{\log t}{t^\nu}+\frac{1}{t^{2(\sigma-\nu)-\epsilon}}\Big)^{0.5(x+y)}.
	\end{flalign*}
\end{lemma}
\begin{proof}
	Consider the first case. Similarly to the proof in \eqref{lemma: mixorderbound1}, without loss of generality, we can assume there exist constants $0<C_1,C_2<\infty$ and $T>0$ such that
	\begin{flalign*}
	&\mE\ltwo{z_t}^2 \leq 4R_w^2 \qquad \text{for all}\,  0\leq t\leq T,\\
	&\mE\ltwo{z_t}^2 \leq C_1^2\Big(\frac{\log t}{t^\nu}+\frac{1}{t^{\nu}}\Big)^x\quad \text{for all}\, t>T,
	\end{flalign*}
	and
	\begin{flalign*}
	&\mE\ltwo{\theta_t-\theta^*}^2 \leq R_\theta^2 \qquad \text{for all}\,  0\leq t\leq T,\\
	&\mE\ltwo{\theta_t-\theta^*}^2 \leq C2^2\Big(\frac{\log t}{t^\nu}+\frac{1}{t^{\nu}}\Big)^y\quad \text{for all}\, t>T.
	\end{flalign*}
	Then, it follows that
	\begin{flalign*}
	&\frac{1}{1+t}\sum_{i=0}^{t} \mE[\langle B_i z_i, \theta_i-\theta^* \rangle]\\
	&\leq \frac{1}{1+t} \sum_{i=0}^{t} \sqrt{\mE[\ltwo{B_i z_i}^2]}\sqrt{\mE[\ltwo{\theta_i-\theta^*}^2]}\\
	&\leq \frac{1}{1+t} \sum_{i=0}^{t} \ltwo{B_i} \sqrt{\mE\ltwo{z_i}^2}\sqrt{\mE[\ltwo{\theta_i-\theta^*}^2]}\\
	&\leq \gamma\rho_{\max} \frac{1}{1+t} \Big(\sum_{i=0}^{T} \sqrt{\mE\ltwo{z_i}^2}\sqrt{\mE[\ltwo{\theta_i-\theta^*}^2]} + \sum_{i=T+1}^{t} \sqrt{\mE\ltwo{z_i}^2}\sqrt{\mE[\ltwo{\theta_i-\theta^*}^2]}\Big)\\
	&\leq \gamma\rho_{\max} \frac{1}{1+t} \Big(2R_w R_\theta(1+T) + C_1C_2\sum_{i=T+1}^{t} \Big(\frac{\log i+1}{i^\nu}\Big)^{0.5(x+y)} \Big)\\
	&\leq \gamma\rho_{\max} \frac{1}{1+t} \Big(2R_w R_\theta(1+T) + C_1C_2(\log t+1)^{0.5(x+y)}\sum_{i=T+1}^{t} \Big(\frac{1}{i^\nu}\Big)^{0.5(x+y)} \Big)\\
	&\leq \gamma\rho_{\max} \left( 2R_w R_\theta \frac{1+T}{1+t} + D  \Big(\frac{\log t+1}{t^\nu}\Big)^{0.5(x+y)} \right),
	\end{flalign*}
	where $0<D<\infty$ is a constant dependent on $x$ and $y$. The proof for the second case follows similarly.
\end{proof}
\begin{lemma}\label{lemma: mixorderbound5}
	Suppose $0<\nu<\frac{2}{3}$, if $\mE\ltwo{z_t}^2=\mathcal{O}(\frac{\log t}{t^\nu}+\frac{1}{t^{\nu}})$, then
	\begin{flalign*}
	\frac{1}{1+t}\sum_{i=0}^{t} \mE[\langle B_i z_i, \theta_i-\theta^* \rangle] =  \mathcal{O}\Big(\frac{\log t}{t^\nu}+\frac{1}{t^{\nu}}\Big)^{1-\epsilon^\prime},
	\end{flalign*}
	and suppose $\frac{2}{3}\leq\nu<1$, if $\mE\ltwo{z_t}^2=\mathcal{O}(\frac{\log t}{t^\nu})+\mathcal{O}(\frac{1}{t^{2(1-\nu)-\epsilon}})$, then
	\begin{flalign*}
	\frac{1}{1+t}\sum_{i=0}^{t} \mE[\langle B_i z_i, \theta_i-\theta^* \rangle] = \mathcal{O}\Big(\frac{\log t}{t^\nu} + \frac{1}{t^{2(1-\nu)-\epsilon}}  \Big)^{1-\epsilon^\prime},
	\end{flalign*}
	where $\epsilon^\prime$ can be any constant in $(0, 0.5]$.
\end{lemma}
\begin{proof}
	We proof this lemma by following similar steps in the proof of Lemma \ref{lemma: mixorderbound3}.
\end{proof}
\section{Proof of Theorem \ref{thm2}}
From \eqref{eq: expzinter} and use the fact that $\beta_t=\beta$ for all $t>0$, we have
\begin{flalign}\label{eq: expzinter_css}
\mE\ltwo{z_{t+1}}^2&\leq (1-|\lambda_w|\beta)^{1+t}\ltwo{z_0}^2 \nonumber\\
&\quad + 2\beta\sum_{i=0}^{t} (1-|\lambda_w|\beta)^{t-i} [\zeta_{f_2}(\theta_i,z_i,O_i)] \nonumber\\
&\quad + 2\beta\sum_{i=0}^{t} (1-|\lambda_w|\beta)^{t-i} [\zeta_{g_2}(z_i,O_i)]\nonumber\\
&\quad + 2\sum_{i=0}^{t} (1-|\lambda_w|\beta)^{t-i}\mE\langle C^{-1}A(\theta_{i+1}-\theta_i),z_i\rangle\nonumber\\
&\quad + 3(K^2_{f_2} + K^2_{g_2})\beta^2\sum_{i=0}^{t} (1-|\lambda_w|\beta)^{t-i} + 3\eta^2\beta^2K^2_{r_1}\sum_{i=0}^{t} (1-|\lambda_w|\beta)^{t-i}.
\end{flalign}
By slightly modifying the proof of Lemma \ref{lemma: biasf2final}, Lemma \ref{lemma: biasg2final} and Lemma \ref{lemma: 1thbdinter}, we have
\begin{flalign}\label{eq: f2bias_css}
	\mE[\zeta_{f_2}(\theta_i,z_i,O_i)]\leq \beta(8R_wK_{f_2} + K_{r_3}\tau_\beta),
\end{flalign}
and
\begin{flalign}\label{eq: g2bias_css}
	\mE[\zeta_{g_2}(z_i,O_i)]\leq \beta(8R_wK_{g_2}+ L_{g_2,z}K_{r_2}\tau_\beta),
\end{flalign}
and
\begin{flalign}\label{eq: zinter_css}
	\sum_{i=0}^{t} (1-|\lambda_w|\beta)^{t-i}\mE\langle C^{-1}A(\theta_{i+1}-\theta_i),z_i\rangle \leq \frac{2(1+\gamma)\rho_{\max}R_w(K_{g_1}+K_{f_1})}{c_{\beta}|\lambda_w|\lambda_{cm}}\eta
\end{flalign}
Substituting \eqref{eq: f2bias_css}, \eqref{eq: g2bias_css} and \eqref{eq: zinter_css} into \eqref{eq: expzinter_css}, and use the fact that $\tau_\beta<\log_{\frac{1}{\rho}}\frac{m}{\rho}+\ln^{-1}(\frac{1}{\rho})\ln(\frac{1}{\beta})$, we have
\begin{flalign}\label{eq: firstratez}
	\mE\ltwo{z_{t+1}}^2&\leq (1-|\lambda_w|\beta)^{1+t}\ltwo{z_0}^2 \nonumber\\
	&\quad + \frac{2( K_{r_3} + L_{g_2,z}K_{r_2})}{|\lambda_w|}\Big(\log_{\frac{1}{\rho}}\frac{m}{\rho}+\ln^{-1}(\frac{1}{\rho})\ln(\frac{1}{\beta})\Big)\beta \nonumber\\
	&\quad + \frac{[16R_w(K_{f_2} + K_{g_2}) + 3(K^2_{f_2} + K^2_{g_2}) + 3\eta^2K^2_{r_1} ]}{|\lambda_w|}\beta\nonumber\\
	&\quad +\frac{2(1+\gamma)\rho_{\max}R_w(K_{g_1}+K_{f_1})}{|\lambda_w|\lambda_{cm}}\eta.
\end{flalign}
Let
\begin{flalign}\label{thm2: c5}
	C_5 = &\frac{2( K_{r_3} + L_{g_2,z}K_{r_2})}{|\lambda_w|}\Big(\log_{\frac{1}{\rho}}\frac{m}{\rho}+\ln^{-1}(\frac{1}{\rho})\Big)\nonumber\\ 
	&+\frac{[16R_w(K_{f_2} + K_{g_2}) + 3(K^2_{f_2} + K^2_{g_2}) + 3K^2_{r_1} ]}{|\lambda_w|}
\end{flalign}
and
\begin{flalign}\label{thm2: c6}
	C_6 = \frac{2(1+\gamma)\rho_{\max}R_w(K_{g_1}+K_{f_1})}{|\lambda_w|\lambda_{cm}}
\end{flalign}
then we have
\begin{flalign*}
	\mE\ltwo{z_t}^2 \leq (1-|\lambda_w|\beta)^t\ltwo{z_0}^2 + C_5\max\{\beta, \ln(\frac{1}{\beta})\beta  \} + C_6\eta.
\end{flalign*}
Let $T = \lceil\frac{\ln[C_5\max\{\beta, \ln(\frac{1}{\beta})\beta  \}/\ltwo{z_0}^2]}{-\ln(1-|\lambda_w|\beta)}\rceil$. Then
\begin{flalign}
    &\mE\ltwo{z_t}^2 \leq 4R_w^2 \qquad\qquad\qquad\qquad\qquad\qquad\, \text{for all}\,  0\leq t\leq T,\label{eq: zupper1_css}\\
    &\mE\ltwo{z_t}^2 \leq 2C_5\max\{\beta, \ln(\frac{1}{\beta})\beta  \} + C_6\eta,\qquad\text{for all}\, t>0.\label{eq: zupper2_css}
\end{flalign}
Consider the recursion of $\theta_t$. From \eqref{eq: expthetainter} and use the fact that $\alpha_t=c_\alpha\alpha$ for all $t>0$, we have
\begin{flalign}
\mE\ltwo{\theta_{t+1}-\theta^*}^2&\leq  (1-|\lambda_\theta|\alpha)^{1+t} \ltwo{\theta_0-\theta^*}^2 \nonumber \\
&\quad + 2\alpha\sum_{i=0}^{t} (1-|\lambda_\theta|\alpha)^{t-i} \mE[\zeta_{f_1}(\theta_i,O_i)] \label{eq: interz_css}\\
&\quad + 2\alpha\sum_{i=0}^{t} (1-|\lambda_\theta|\alpha)^{t-i} \mE\langle B_iz_i,\theta_i-\theta^*\rangle \label{eq: expthetainter_css}  \\ 
&\quad + 2(K^2_{f_1} + K^2_{g_1})\alpha^2\sum_{i=0}^{t} (1-|\lambda_\theta|\alpha)^{t-i} .\nonumber
\end{flalign}
By slightly modifying the proof of Lemma \ref{lemma: biasf1final}, we have
\begin{flalign}\label{eq: f1bias_css}
\mE[\zeta_{f_1}(\theta_i,O_i)]\leq \alpha(8R_\theta K_{f_1} + L_{f_1,\theta}(K_{f_1}+K_{g_1})\tau_\alpha).
\end{flalign}
Substitute \eqref{eq: zupper1_css} and \eqref{eq: zupper2_css} into \eqref{eq: expthetainter_css}, we have
\begin{flalign}
&2\alpha\sum_{i=0}^{t} (1-|\lambda_\theta|\alpha)^{t-i} \mE\langle B_iz_i,\theta_i-\theta^*\rangle\nonumber\\
&\leq 4\alpha\gamma\rho_{\max}R_\theta \Big[2R_w\sum_{i=0}^{T}(1-|\lambda_\theta|\alpha)^{t-i} + (2C_5\max\{\beta, \ln(\frac{1}{\beta})\beta  \} + C_6\eta)^{0.5}  \sum_{i=T+1}^{t}(1-|\lambda_\theta|\alpha)^{t-i}\Big]\nonumber\\
&\leq 8\gamma\rho_{\max}R_\theta R_w \frac{1-(1- |\lambda_\theta|\alpha)^{T+1}}{|\lambda_\theta|(1-|\lambda_\theta|\alpha)^T} (1-|\lambda_\theta|\alpha)^t + \frac{4\gamma \rho_{\max}R_\theta}{|\lambda_\theta|}(2C_5\max\{\beta, \ln(\frac{1}{\beta})\beta  \} + C_6\eta)^{0.5} \label{eq: interfistbd}
\end{flalign}
Substitute \eqref{eq: f1bias_css} and \eqref{eq: interfistbd} into \eqref{eq: interz_css} and \eqref{eq: expthetainter_css} and using the fact that $\tau_\alpha<\log_{\frac{1}{\rho}}\frac{m}{\rho}+\ln^{-1}(\frac{1}{\rho})\ln(\frac{1}{\alpha})$ we have
\begin{flalign}
\mE\ltwo{\theta_{t+1}-\theta^*}^2&\leq  (1-|\lambda_\theta|\alpha)^{1+t} \ltwo{\theta_0-\theta^*}^2\nonumber\\
&\quad + \frac{2 L_{f_1,\theta}(K_{f_1}+K_{g_1})}{|\lambda_\theta|}(\log_{\frac{1}{\rho}}\frac{m}{\rho}+\ln^{-1}(\frac{1}{\rho})\ln(\frac{1}{\alpha}))\alpha\nonumber\\
&\quad + \frac{2c_\alpha(8R_\theta K_{f_1}+K^2_{f_1} + K^2_{g_1})}{|\lambda_\theta|}\alpha\nonumber\\
&\quad + \frac{4\gamma \rho_{\max}R_\theta}{|\lambda_\theta|}(2C_5\max\{\beta, \ln(\frac{1}{\beta})\beta  \}+ C_6\eta)^{0.5}\nonumber\\
&\quad + 8\gamma\rho_{\max}R_\theta R_w \frac{1-(1- |\lambda_\theta|\alpha)^{T+1}}{|\lambda_\theta|(1-|\lambda_\theta|\alpha)^T} (1-|\lambda_\theta|\alpha)^t.\nonumber
\end{flalign}
Let
\begin{flalign}\label{thm2: c2}
	C_2 = \frac{2 L_{f_1,\theta}(K_{f_1}+K_{g_1})}{|\lambda_\theta|}(\log_{\frac{1}{\rho}}\frac{m}{\rho}+\ln^{-1}(\frac{1}{\rho})) + \frac{2(8R_\theta K_{f_1}+K^2_{f_1} + K^2_{g_1})}{|\lambda_\theta|},
\end{flalign}
and
\begin{flalign}\label{thm2: c3}
	C_3 = 32\Big(\frac{\gamma \rho_{\max}R_\theta}{|\lambda_\theta|}\Big)^2C_5,
\end{flalign}
and
\begin{flalign}\label{thm2: c4}
	C_4 = 16\Big(\frac{\gamma \rho_{\max}R_\theta}{|\lambda_\theta|}\Big)^2C_6,
\end{flalign}
then we have
\begin{flalign}
	\mE\ltwo{\theta_{t+1}-\theta^*}^2 \leq (1-|\lambda_\theta|\alpha)^{1+t} (\ltwo{\theta_0-\theta^*}^2 + C_1) &+ C_2\max\{ \alpha, \ln(\frac{1}{\alpha})\alpha  \} \nonumber \\
	&+ (C_3\max\{\beta, \ln(\frac{1}{\beta})\beta  \} + C_4\eta)^{0.5}
\end{flalign}
where $C_1 = 8\gamma\rho_{\max}R_\theta R_w \frac{1-(1- |\lambda_\theta|\alpha)^{T+1}}{|\lambda_\theta|(1-|\lambda_\theta|\alpha)^{T+1}}$.
\section{Proof of Theorem \ref{thm3}}
We define vector $x_t = [\theta^\top_t, w^\top_t]^\top$ and $x^*=[\theta^{*\top},0^\top]^\top$, convex set $X=\{x| \sum_{i=1}^{d} x_i^2 \leq R_\theta^2\, \text{and}\,\sum_{i=d+1}^{2d} x_i^2 \leq R_w^2 \}$ and the projection operator $\mcpi_X(x) = \argmin_{x^\prime:x^\prime\in X}||x-x^\prime||_2$. We also define
\begin{flalign*}
	G_t = \left[ \begin{array}{cc}
	A_t & B_t \\
	\eta A_t & \eta B_t \\
	\end{array} \right], \quad g_t = \left[ \begin{array}{cc}
	b_t\\
	\eta b_t\\
	\end{array}\right],
\end{flalign*}
and
\begin{flalign*}
G = \left[ \begin{array}{cc}
A & B \\
\eta A & \eta B \\
\end{array} \right], \quad g = \left[ \begin{array}{cc}
b\\
\eta b\\
\end{array}\right].
\end{flalign*}
Then, we can rewrite the update of \eqref{algorithm1_1}-\eqref{algorithm1_2}
\begin{flalign}\label{eq: update_x}
	x_{t+1} = \mcpi_X (x_t + \alpha_t (G_t x_t + g_t)).
\end{flalign}
We define $h(x_t,O_t)=G_t x_t + g_t$ and $\bar{h}(x_t)=G x_t + g$. Then, for the recursion of $x_t$ in \eqref{eq: update_x}, for any $t>0$, we have
\begin{flalign}\label{eq: update_x_II}
	\ltwo{x_{t+1} - x^*}^2 &= \ltwo{\mcpi_X (x_t + \alpha_t h(x_t,O_t)) - x^*}^2\nonumber\\
	&=\ltwo{\mcpi_X (x_t + \alpha_t h(x_t,O_t)) - \mcpi_X(x^*)}^2\nonumber\\
	&\leq \ltwo{x_t - x^* + \alpha_t h(x_t,O_t)}^2\nonumber\\
	&=\ltwo{x_t-x^*}^2 + 2\alpha_t\langle h(x_t,O_t), x_t - x^* \rangle + \alpha_t^2 \ltwo{h(x_t,O_t)}^2\nonumber\\
	&=\ltwo{x_t-x^*}^2 + 2\alpha_t\langle \bar{h}(x_t), x_t - x^* \rangle + 2\alpha_t\langle h(x_t,O_t) - \bar{h}(x_t), x_t - x^* \rangle + \alpha_t^2 \ltwo{h(x_t,O_t)}^2\nonumber\\
	&=(1-\alpha_t|\lambda_x|)\ltwo{x_t-x^*}^2 + 2\alpha_t\zeta_{h}(x_t,O_t) + \alpha_t^2 \ltwo{h(x_t,O_t)}^2,
\end{flalign}
where $\lambda_x=\lambda_{\max}(G+G^\top)$, and $\lambda_x<0$ as shown in \cite{maei2011gradient}. Then, consider the update in any block $s>0$. Taking expectation on both sides conditional on the filtration $\mcf_{s-1}$ up to block $s-1$ and telescoping \eqref{eq: update_x} yield that
\begin{flalign}\label{eq: tele_x}
	\mE[\ltwo{x_s-x^*}^2|\mcf_{s-1}]&\leq (1-|\alpha_s|\lambda_x)^{T_s} \ltwo{x_{s-1}-x^*}^2 \nonumber \\
	&\quad + 2\alpha_s\sum_{i=1}^{T_s}(1-\alpha_s|\lambda_x|)^{T_s-i} \mE[\zeta_h(x_{t_{s-1}+i},O_{t_{s-1}+i})]\nonumber\\
	&\quad + \alpha_s^2\sum_{i=1}^{T_s} (1-\alpha_s|\lambda_x|)^{T_s-i}\ltwo{h(x_{t_{s-1}+i},O_{t_{s-1}+i})}^2.
\end{flalign}
Following similar steps in the proof for Theorem \ref{thm1}, we have the following results:

(a) There exist constant $C_G$ and $C_g$ such that $\ltwo{G_t}, \ltwo{G}\leq C_G$ and $\ltwo{g_t}, \ltwo{g}\leq C_g$.

(b) For all $i>0$,  $\ltwo{h(x_i, O_i)}\leq K_h$, where $K_h=C_G\sqrt{R_\theta^2+R_w^2}+C_g$.

(c) For all $i>0$, $\ltwo{\zeta_h(x_i,O_i)}\leq 4K_h\sqrt{R_\theta^2+R_w^2}$.

(d) For all $i>0$ and $x,x^\prime\in X$, $\ltwo{\zeta_h(x,O_i)-\zeta_h(x^\prime,O_i)}\leq L_h\ltwo{x-x^\prime}$, where $L_h=4C_G\sqrt{R_\theta^2+R_w^2}+2K_h$.

(e) For all $i>0$, $\mE[\zeta_h(x_i,O_i)]\leq \alpha_s(8K_h\sqrt{R_\theta^2+R_w^2} + L_hK_h\tau_{\alpha_s})$.

Then, substituting (e) into \eqref{eq: tele_x}, we obtain
\begin{flalign*}
	\mE[\ltwo{x_s-x^*}^2|\mcf_{s-1}]&\leq (1+\alpha_s\lambda_x)^{T_s} \ltwo{x_{s-1}-x^*}^2 + \frac{2}{|\lambda_x|}\alpha_s(8K_h\sqrt{R_\theta^2+R_w^2} + L_hK_h\tau_{\alpha_s}) + \frac{1}{|\lambda_x|}\alpha_s K_h^2.
\end{flalign*}
Recall that $\tau_{\alpha_s}\leq\log_{\frac{1}{\rho}}\frac{m}{\rho}+\ln^{-1}(\frac{1}{\rho})\ln(\frac{1}{\alpha_s})$. Then, we have
\begin{flalign}\label{eq: finalform_x}
	\mE[\ltwo{x_s-x^*}^2|\mcf_{s-1}]&\leq (1+\alpha_s\lambda_x)^{T_s} \ltwo{x_{s-1}-x^*}^2 + C_9\max\{\alpha_s, \ln(\frac{1}{\alpha_s})\alpha_s \},
\end{flalign}
where
\begin{flalign}\label{thm3: C_7}
	C_7 = \frac{2}{|\lambda_x|}(8K_h\sqrt{R_\theta^2+R_w^2} + L_hK_h\log_{\frac{1}{\rho}}\frac{m}{\rho} + L_hK_h\ln^{-1}(\frac{1}{\rho})+\frac{1}{2}K_h^2).
\end{flalign}
Since $\max\{\alpha_s, \ln(\frac{1}{\alpha_s})\alpha_s \} \leq \epsilon_{s-1}/(4C_7)$ and $(1+\alpha_s\lambda_x)^{T_s}\leq 1/4$, we have
\begin{flalign*}
	\mE[\ltwo{x_s-x^*}^2|\mcf_{s-1}] \leq \frac{1}{2}\epsilon_{s-1}.
\end{flalign*}
After $S=\lceil \log_2(\epsilon_0/\epsilon) \rceil$ blocks we have
\begin{flalign*}
\mE\ltwo{\theta_{S}-\theta^*}^2\leq \mE\ltwo{x_{S}-x^*}^2 \leq \epsilon.
\end{flalign*}
The total iteration complexity is $\sum_{s=1}^{S}T_s=\mathcal{O}(\frac{1}{\epsilon^{1+\xi}})$, where $\xi>0$ can be arbitrarily small.

\end{document}